\documentclass[twoside,11pt]{article}

\usepackage{jmlr2e}

\usepackage{graphicx}
\usepackage{subfigure}
\usepackage{booktabs} 

\usepackage{hyperref} 
\usepackage{fancybox}
\usepackage[square,numbers]{natbib}

\def \v {\mathbf{v}}
\def \x {\mathbf{x}}
\def \w {\mathbf{w}}
\def \z {\mathbf{z}}

\def \E {\mathbb{E}}
\def \I {\mathbb{I}}

\def \teta {\tilde{\eta}}  

\usepackage{soul} 
\usepackage{graphicx} 
\usepackage{amsmath}
\usepackage{booktabs}
\usepackage{algorithm}
\usepackage{algorithmic}

\usepackage{amsfonts} 
\usepackage{verbatim} 
\usepackage{amssymb}
\usepackage{verbatim}
\usepackage{enumitem}
\usepackage{sidecap,color}

\newtheorem{thm}{Theorem}
\newtheorem{lem}{Lemma}
\newtheorem{ass}{Assumption}

\begin{document}

\title{Federated  Deep AUC Maximization for Heterogeneous Data\\ with a Constant Communication Complexity }

\author{\name Zhuoning Yuan\thanks{Equal Contribution} 
        \email zhuoning-yuan@uiowa.edu \\
        \addr The University of Iowa \\
        \name Zhishuai Guo\footnotemark[1]
        \email zhishuai-guo@uiowa.edu\\
        \addr The University of Iowa \\
        \name Yi Xu 
        \email yixu@alibaba-inc.com\\
        \addr DAMO Academy, Alibaba Group \\
        \name Yiming Ying 
        \email yying@albany.edu\\ 
        \addr  State University of New York at Albany \\
        \name Tianbao Yang \email tianbao-yang@uiowa.edu\\ 
        \addr  The University of Iowa \\
        }

\maketitle 







\vskip 0.3in
  



\begin{abstract}
\underline{D}eep \underline{A}UC (area under the ROC curve) \underline{M}aximization (DAM) has attracted much attention recently due to its great potential for imbalanced data classification. However, the research on \underline{F}ederated \underline{D}eep \underline{A}UC \underline{M}aximization (FDAM) is still limited. Compared with standard federated learning (FL) approaches that focus on decomposable minimization objectives, FDAM is  more complicated due to its minimization objective is non-decomposable over individual examples. In this paper, we propose improved FDAM algorithms for heterogeneous data by solving the popular non-convex strongly-concave min-max formulation of DAM in a distributed fashion, which can also be applied to a class of non-convex strongly-concave min-max problems. A striking result of this paper is that the communication complexity of the proposed algorithm  is a constant independent of the number of machines and also independent of the accuracy level, which improves an existing result by orders of magnitude. The experiments have demonstrated the effectiveness of our FDAM algorithm on benchmark datasets, and on medical chest X-ray images from different organizations. Our  experiment shows that the performance of FDAM using data from multiple hospitals can improve the AUC score on testing data from a single hospital for detecting life-threatening diseases based on chest radiographs. The proposed method is implemented in our open-sourced library LibAUC (\url{www.libauc.org}) whose github address is~\url{https://github.com/Optimization-AI/ICML2021_FedDeepAUC_CODASCA}.

\end{abstract}

\section{Introduction}
Federated learning (FL) is an emerging paradigm for large-scale learning to deal with data that are (geographically) distributed over multiple clients, e.g., mobile phones, organizations. An important feature of FL is that the data remains at its own clients, allowing the preservation of data privacy. This feature makes FL  attractive not only to internet companies such as Google and Apple but also to conventional industries such as those that provide services to hospitals and banks in the big data era~\cite{rieke2020future,long2020federated}.  Data in these industries is usually collected from people who are concerned about data leakage. But in order to provide better services, large-scale machine learning from diverse data sources is important for addressing model bias. For example, most patients in hospitals located in urban areas could have dramatic differences in demographic data, lifestyles, and diseases from patients who are from rural areas. Machine learning models (in particular, deep neural networks)  trained based on patients' data from one hospital could dramatically bias towards its major population, which could bring serious ethical concerns~\cite{pooch2020can}.  

One of the fundamental issues that could cause model bias is data imbalance, where the number of samples from different classes are skewed. Although FL provides an effective framework for leveraging multiple data sources, {\it most existing FL methods still lack the capability to tackle the  model bias caused by data imbalance}. The reason is that most existing FL methods are developed for minimizing the conventional objective function, e.g., the average of a standard  loss function on all data, which are not amenable to optimizing more suitable measures such as area under the ROC curve (AUC) for imbalanced data. It has been recently shown that directly maximizing AUC for deep learning can lead to great improvements on real-world difficult classification tasks~\cite{yuan2021robust}. For example,  \citet{yuan2021robust} reported the best performance by DAM on the Stanford CheXpert Competition for interpreting chest X-ray images like radiologists~\cite{chexpert19}. 

\begin{table*}[t] 
	\caption{The summary of sample and communication complexities of different algorithms for FDAM under a $\mu$-PL condition in both heterogeneous and homogeneous settings, where $K$ is the number of machines and $\mu \leq 1$. NPA denotes the naive parallel (large mini-batch)  version of PPD-SG \cite{liu2019stochastic} for DAM, where $M$ denotes the batch size in the NPA. The $*$ indicate the results that are derived by us.  $\widetilde O(\cdot)$ suppresses a logarithmic factor.
}\label{tab:1} 
	\centering
	\label{tab:2}
	\begin{small}
	{\begin{tabular}{l|l|l|l}
			\toprule
		~& Heterogeneous Data & Homogeneous Data &Sample Complexity\\
		\hline
		NPA ($M< \frac{1}{K\mu\epsilon})$& $\widetilde{O}\left(\frac{1}{KM\mu^2\epsilon} + \frac{1}{\mu\epsilon}  \right)$& $\widetilde{O}\left(\frac{1}{KM\mu^2\epsilon} + \frac{1}{\mu\epsilon}  \right)$&  $\widetilde{O}\left( \frac{M}{\mu\epsilon} + \frac{1}{\mu^2 K \epsilon} \right)$\\
		\hline 
		NPA ($M\geq  \frac{1}{K\mu\epsilon})$ &  $\widetilde{O}\left( \frac{1}{\mu} \right)^*$
		&$\widetilde{O} \left( \frac{1}{\mu} \right)^*$&  $\widetilde{O}\left( \frac{M}{\mu} \right)^*$  \\ 
		\hline
        \hline
        CODA+ (CODA)& $\widetilde{O}\left(\frac{K}{\mu} + \frac{1}{\mu\epsilon^{1/2}} + \frac{1}{\mu^{3/2}\epsilon^{1/2}} \right)$ &
        $\widetilde{O}\left(\frac{K}{\mu}\right)^*$ & $\widetilde{O} \left( \frac{1}{\mu\epsilon} + \frac{1}{\mu^2K \epsilon} \right)$ \\
        \hline 
        CODASCA & $\widetilde{O}\left(\frac{1}{\mu}\right)$ &
        $\widetilde{O}\left(\frac{1}{\mu}\right)$ & $\widetilde{O} \left( \frac{1}{\mu\epsilon} +  \frac{1}{\mu^2K \epsilon} \right)$\\
 		\bottomrule 
	\end{tabular}}
	\end{small}
	\vspace*{-0.15in} 
\end{table*}
However, the research on FDAM is still limited. To the best of our knowledge, \citet{dist_auc_guo} is the only work that was dedicated to FDAM by solving {\bf the non-convex strongly-concave min-max} problem in a distributed manner. Their algorithm (CODA) is similar to the standard FedAvg method~\cite{DBLP:conf/aistats/McMahanMRHA17} except that the periodic averaging is applied both to the primal and the dual variables. Nevertheless, their results on FDAM are not comprehensive. By a deep investigation of their algorithms and analysis, we found that (i) although their FL algorithm CODA was shown to be better than the naive parallel algorithm (NPA) with a small mini-batch for DAM, the NPA using a larger mini-batch at local machines  can enjoy a smaller communication complexity than CODA; (ii) the communication complexity of CODA for homogeneous data becomes better than that was established for the heterogeneous data, but is still worse than that of NPA with a large mini-batch at local clients. These shortcomings of CODA for FDAM motivate us to develop better federated averaging algorithms  and analysis with a better communication complexity without sacrificing  the sample complexity.


This paper aims to provide more comprehensive results for FDAM, with a focus on improving the communication complexity of CODA for heterogeneous data.  In particular, our contributions are summarized below: 
\begin{itemize}[leftmargin=*]
\vspace*{-0.1in}
    \item First, we provide a stronger baseline with a simpler algorithm than CODA named CODA+, and establish its complexity in both homogeneous and heterogeneous data settings. Although CODA+ has a slight change from CODA, its analysis is much more involved than that of CODA, which is based on the duality gap analysis instead of the primal objective gap analysis. 
    \vspace{-0.1in} 
    \item Second, we propose a new variant of CODA+ named CODASCA with a much improved communication complexity than CODA+. The key thrust is to incorporate the idea of stochastic controlled averaging of SCAFFOLD~\cite{karimireddy2019scaffold} into the framework of CODA+ to correct the client-drift for both local primal updates and local dual updates. A striking result of CODASCA under a PL condition for deep learning is that its communication complexity is independent of the number of machines and the targeted accuracy level, which is even better than CODA+ in the homogeneous data setting. The analysis of CODASCA is also non-trivial that combines the duality gap analysis of CODA+ for a non-convex strongly-concave min-max problem and the variance reduction analysis of SCAFFOLD. The comparison between CODASCA and CODA+ and the NPA for FDAM is shown in Table~\ref{tab:2}. 
    \vspace{-0.1in} 
    \item Third, we conduct experiments on benchmark datasets to verify our theory by showing CODASCA can enjoy a larger communication window size than CODA+ without sacrificing the performance. Moreover, we conduct empirical studies on medical chest X-ray images from different hospitals by showing that the performance of CODASCA using data from multiple organizations can improve the performance on testing data from a single hospital. 
\end{itemize}


\section{Related Work}
\paragraph{Federated Learning (FL).} Many empirical studies~\cite{povey2014parallel,su2015experiments,mcmahan2016communication,chen2016scalable,lin2018don,kamp2018efficient,DBLP:conf/eccv/YuanGYWY20} have shown that FL exhibits good empirical performance for distributed deep learning. For a more thorough survey of FL, we refer the readers to~\cite{mcmahan14advances}. This paper is closely related to recent studies on the design of distributed stochastic algorithms for FL with provable convergence guarantee. 

The most popular FL algorithm is Federated Averaging (FedAvg)~\cite{DBLP:conf/aistats/McMahanMRHA17}, also referred to as local SGD~\cite{stich2018local}.
\citet{stich2018local} is the first that establishes the convergence of local SGD for strongly convex functions.  \citet{yu2019parallel,yu_linear} establishes  the convergence of local SGD and their momentum variants for non-convex functions. The analysis in \cite{yu2019parallel} has exhibited the difference of communication complexities of local SGD in homogeneous and heterogeneous data settings, which is also discovered in recent  works~\cite{khaled2020tighter,woodworth2020local,DBLP:conf/nips/WoodworthPS20}. These latter studies  provide a tight analysis of local SGD in homogeneous and/or heterogeneous data settings, improving its upper bounds for convex functions and strongly convex functions than some earlier works, which sometimes improve over large mini-batch SGD, e.g., when the level of heterogeneity is sufficiently small. 

\citet{DBLP:conf/nips/HaddadpourKMC19} improve the complexities of local SGD for non-convex optimization by leveraging the Polyak-\L ojasiewicz (PL) condition. \cite{karimireddy2019scaffold} propose a new FedAvg algorithm SCAFFOLD by introducing control variates (variance reduction) to correct for the `client-drift' in the local updates for heterogeneous data. The communication complexities of SCAFFOLD are no worse than that of large mini-batch SGD for both strongly convex and non-convex functions. The proposed algorithm CODASCA is inspired by  the idea of stochastic controlled averaging of SCAFFOLD. However, the analysis of CODASCA for non-convex min-max optimization under a PL condition of the primal objective function is non-trivial compared to that of SCAFFOLD.

{\bf AUC Maximization.} This work builds on the foundations of stochastic  AUC maximization developed in many previous works.  \citet{ying2016stochastic} address the scalability issue of optimizing AUC by introducing a min-max reformulation of the AUC square surrogate loss and solving it by a convex-concave stochastic gradient method~\cite{nemirovski2009robust}. ~\citet{natole2018stochastic} improve the convergence rate by adding a strongly convex regularizer into  the original formulation. Based on the same min-max formulation as in~\cite{ying2016stochastic}, ~\citet{liu2018fast} achieve an improved convergence rate by developing a multi-stage algorithm by leveraging the quadratic growth condition of the problem. However, all of these studies focus on learning a linear model, whose corresponding problem is convex and strongly concave. \citet{yuan2021robust} propose a more robust margin-based surrogate loss for the AUC score, which can be formulated as a similar min-max problem to the AUC square surrogate loss. 

{\bf Deep AUC Maximization (DAM).} \cite{rafique2018non} is the first work that develops algorithms and convergence theories for weakly convex and strongly concave min-max problems, which is applicable to DAM. However, their convergence rate is slow for a practical purpose.  \citet{liu2019stochastic} consider improving the convergence rate for DAM  under a practical PL condition of the primal objective function.   \citet{guo2020fast} further develop more generic algorithms for non-convex strongly-concave min-max problems, which can also be applied to DAM. There are also several studies~\cite{yan2020optimal,lin2019gradient,arXiv:2001.03724,yang2020global} focusing on non-convex strongly concave min-max problems without considering the application to DAM.   Based on~\citet{liu2019stochastic}'s algorithm, \citet{dist_auc_guo} propose a communication-efficient FL algorithm (CODA) for DAM. However, its communication cost is still high for heterogeneous data. 


{\bf DL for Medical Image Analysis.}
In past decades, machine learning, especially deep learning methods have revolutionized many domains such as machine vision, natural language processing. For medical image analysis, deep learning methods are also showing great potential such as in classification of skin lesions \cite{esteva2017dermatologist,li2018skin}, interpretation of chest radiographs \cite{ardila2019end,chexpert19}, and breast cancer screening \cite{bejnordi2017diagnostic,mckinney2020international,wang2016deep}. Some works have already achieved expert-level performance in different tasks \cite{ardila2019end,mckinney2020international,DBLP:journals/mia/LitjensKBSCGLGS17}. Recently, \citet{yuan2021robust} employ DAM for medical image classification and achieve great success on two challenging tasks, namely CheXpert competition for chest X-ray image classification and Kaggle competition for melanoma classification based on skin lesion images.  However, to the best of our knowledge,  the application of FDAM methods on medical datasets from different hospitals have not be thoroughly investigated.

\section{Preliminaries and Notations}
We consider federated learning of deep neural networks by maximizing the AUC score. The setting is the same to that was considered as in~\cite{dist_auc_guo}. Below, we present some preliminaries and notations, which are mostly the same as in~\cite{dist_auc_guo}. 
In this paper, we consider the following min-max formulation for distributed problem: 
\vspace{-0.2in}  
\begin{equation}
\label{opt:problem}
\min\limits_{\w\in \mathbb{R}^d \atop (a, b)\in \mathbb{R}^2} \max\limits_{\alpha\in \mathbb{R}} f(\w, a, b, \alpha)=\frac{1}{K}\sum\limits_{k=1}^{K} f_k(\mathbf{w}, a, b, \alpha),
\end{equation}
where $K$ is the total number of machines. This formulation covers a class of non-convex strongly concave min-max problems and specifically for the AUC maximization, $f_k(\w, a, b, \alpha)$ is defined below.
\vspace{-0.05in} 
\begin{equation}
\begin{split}
&f_k(\w, a, b, \alpha)=\E_{\z^k}[F_k(\textbf{w}, a, b, \alpha; \z^k)]\\
&= \E_{\z^k}\left[(1-p) (h(\w; \x^k)- a)^2 \mathbb{I}_{[y^k=1]} 
+ p(h(\w; \x^k) - b)^2\mathbb{I}_{[y^k=-1]} \right. \\
&\hspace*{0.2in}+ 2(1+\alpha)(p h(\w; \x^k)\mathbb{I}_{[y^k=-1]} - \left.(1-p) h(\w, \x^k)\mathbb{I}_{[y^k=1]}) - p(1-p)\alpha^2\right].
\end{split}
\end{equation}
where $\z^k = (\x^k, y^k)\sim \mathbb{P}_k$, $\mathbb{P}_k$ is the data distribution on machine $k$, $p$ is the ratio of positive data. When $\mathbb \phi_k = \mathbb \phi_l,\forall k\neq l$, this is referred to as the homogeneous data setting; otherwise heterogeneous data setting.  

{\bf Notations.} We define the following notations:  
%
%
\begin{align*}
&\v = (\w^T, a, b)^T, \quad \phi(\v) = \max_{\alpha} f(\v, \alpha),\\
&  \phi_s (\v) = \phi(\v) + \frac{1}{2\gamma}\|\v-\v_{s-1}\|^2,\\
&f^s(\v, \alpha) = f(\v, \alpha) + \frac{1}{2\gamma}\|\v -\v_{s-1}\|^2\\
&F^s_k(\v, \alpha; \z_k) = F_k(\v, \alpha; \z_k) + \frac{1}{2\gamma}\|\v - \v_{s-1}\|^2\\
&\v^*_{\phi} = \arg\min\limits_{\v}\phi(\v),\quad \v^*_{\phi_s} = \arg\min\limits_{\v} \phi_s(\v).
\end{align*}

{\bf Assumptions.} Similar to~\cite{dist_auc_guo}, 
we make the following assumptions throughout this paper.
\begin{ass}~\\
\label{ass:1}
(i) There exist $\v_0, \Delta_0>0$ such that $\phi(\v_0) - \phi(\v^*_\phi)\leq \Delta_0$. \\
(ii) PL condition: $\phi(\v)$ satisfies the $\mu$-PL condition, i.e., $\mu(\phi(\v)-\phi(\v_*))\leq \frac{1}{2}\|\nabla\phi(\v)\|^2$; 
(iii) Smoothness: For any $\z$, $f(\v, \alpha;\z)$ is $\ell$-smooth in $\v$ and $\alpha$. $\phi(\v)$ is $L$-smooth, i.e., $\|\nabla\phi(\v_1) - \nabla\phi(\v_2)\| \leq L \|\v_1 - \v_2\|$.\\
(iv) Bounded variance: 
\begin{equation}
\begin{split}
&\E[\|\nabla_{\v} f_k(\v, \alpha) - \nabla_{\v} F_k(\v, \alpha; \z)\|^2] \leq \sigma^2, \\ &\E[|\nabla_{\alpha}f_k(\v, \alpha) - \nabla_{\alpha} F_k(\v, \alpha; \z)|^2] 
\leq \sigma^2.
\end{split}
\end{equation}
\label{assumption_1}
\end{ass}
\vspace{-0.1in}
To quantify the drifts between different clients, we introduce the following assumption.
\begin{ass}
\label{ass2}
Bounded client drift: 
\begin{equation}
\begin{split}
\frac{1}{K}\sum\limits_{k=1}^{K} \|\nabla_{\v} f_k(\v, \alpha) - \nabla_{\v} f(\v, \alpha)\|^2 &\leq D^2, \\
\frac{1}{K}\sum\limits_{k=1}^{K}\|\nabla_{\alpha} f_k(\v, \alpha) - \nabla_{\alpha} f(\v, \alpha)\|^2  &\leq D^2. 
\end{split}
\end{equation}
\end{ass} 
\textbf{Remark.} $D$ quantifies the drift between the local objectives and the global objective. $D=0$ denotes the homogeneous data setting that all the local objectives are identical. $D>0$ corresponds to the heterogeneous data setting.  

\section{CODA+: A stronger baseline}
In this section, we present a stronger baseline than CODA~\cite{dist_auc_guo}. The motivation is that (i) the CODA algorithm uses a step to compute the dual variable from the primal variable by using sampled data from all clients; but we find this step is unnecessary by an improved analysis; (ii) the complexity of CODA for homogeneous data is not given in its original paper. Hence, CODA+ is a simplified version of CODA but with much refined analysis.  

We present the steps of CODA+ in Algorithm~\ref{alg:codaplus_outer}. It is similar to CODA that uses stagewise updates.  In $s$-th stage, a strongly convex strongly concave subproblem is constructed: 
\begin{equation}
\begin{split}
\min\limits_{\v} \max_{\alpha} f(\v, \alpha) + \frac{\gamma}{2}\|\v - \v_0^s\|^2,
\end{split}
\end{equation}
where $\v_0^s$ is the output of the previous stage.

CODA+ improves upon CODA in two folds. First, CODA+ algorithm is more concise since the output primal and dual variables of each stage can be directly used as input for the next stage, while CODA needs an extra large batch of data after each stage to compute the dual variable. This modification not only reduces the sample complexity, but also makes the algorithm applicable to a boarder family of nonconvex  min-max problems. Second, CODA+ has a smaller communication complexity for homogeneous data than that for heterogeneous data while the previous  analysis of CODA yields the same communication complexity for homogeneous data and heterogeneous data. 

\begin{algorithm}[t]
\caption{CODA+}
\begin{algorithmic}[1] 
\STATE{Initialization: $(\v_0, \alpha_0, \gamma)$.} 
\FOR{$s=1, ..., S$}
\STATE{$\v_s, \alpha_s = \text{DSG+} (\v_{s-1}, \alpha_{s-1}, \eta_s, I_s, \gamma)$;}
\ENDFOR 
\STATE{Return $\v_{S}, \alpha_S$.}   
\end{algorithmic}
\label{alg:codaplus_outer}
\end{algorithm}

\begin{algorithm}[t]
\caption {DSG+($\v_0, \alpha_0, \eta, T, I, \gamma$)}
\begin{algorithmic}
\STATE{Each machine does initialization: $\v_0^k = \v_0, \alpha_0^k = \alpha_0$,}
\FOR{$t=0, 1, ..., T-1$}
\STATE{Each machine $k$ updates its local solution in parallel:}
\STATE{~~~~$\v_{t+1}^k = \v_t^k - \eta (\nabla_{\v} F_k(\v_t^k, \alpha_t^k; \z_t^k) + \gamma(\v_t^k-\v_0))$,}
\STATE{~~~~$\alpha_{t+1}^k = \alpha_t^k + \eta \nabla_{\alpha} F_k(\v_t^k, \alpha_t^k; \z_t^k)$,}
\IF{$t+1$ mod $I = 0$ } 
\STATE{$\v^k_{t+1} = \frac{1}{K} \sum\limits_{k=1}^{K} \v_{t+1}^k$, \hfill $\diamond$ communicate }
\STATE{$\alpha^k_{t+1} = \frac{1}{K}\sum\limits_{k=1}^{K} \alpha_{t+1}^k$, \hfill $\diamond$ communicate}
\ENDIF 
\ENDFOR 
\STATE{Return $\left(\bar{\v} =  \frac{1}{K}\sum\limits_{k=1}^{K}\frac{1}{T} \sum\limits_{t=1}^{T} \v_t^k, \bar{\alpha} = \frac{1}{K}\sum\limits_{k=1}^{K}\frac{1}{T} \sum\limits_{t=1}^{T} \alpha_t^k\right)$.} 
\end{algorithmic} 
\label{alg:codaplus_inner}
\end{algorithm}

We have the following lemma to bound the convergence for the subproblem in each $s$-th stage.
\begin{lem}\label{lem:1}(One call of Algorithm \ref{alg:codaplus_inner})
Let $(\Bar{\v}, \Bar{\alpha})$ be the output of Algorithm \ref{alg:codaplus_inner}.
Suppose Assumption \ref{ass:1} and \ref{ass2} hold.
By running Algorithm \ref{alg:codaplus_inner} with given input $\v_0, \alpha_0$ for $T$ iterations, $\gamma = 2\ell$, and $\eta\leq \min(\frac{1}{3\ell + 3\ell^2/\mu_2}, \frac{1}{4\ell})$, we have for any $\v$ and $\alpha$\\  
\begin{align*}
&\E[f^s(\bar{\v}, \alpha)  - f^s(\v, \bar{\alpha})] 
\leq \frac{1}{\eta T} \|\v_0 - \v\|^2 + \frac{1}{\eta T} (\alpha_0 - \alpha)^2 \\
&
+\underbrace{\left(\frac{3\ell^2}{2\mu_2} + \frac{3\ell}{2}\right)(12\eta^2 I \sigma^2 + 36 \eta^2 I^2 D^2 )\I_{I>1} }_{A_1}  + \frac{3\eta\sigma^2}{K},
\end{align*} 
where $\mu_2=2p(1-p)$ is the strong concavity coefficient of $f(\v, \alpha)$ in $\alpha$. 
\label{lem:one_stage_codaplus}  
\end{lem} 

\textbf{Remark.}  Note that the term $A_1$ on the RHS is the drift of clients caused by skipping communication. When $D=0$, i.e., the machines have homogeneous data distribution, we need $\eta I = O\left(\frac{1}{K}\right)$, then $A_1$ can be merged with the last term. When  $D>0$, we need $\eta I^2 =   O\left(\frac{1}{K} \right)$, which means that $I$ has to be smaller in heterogeneous data setting and thus the communication complexity is higher.  

\textbf{Remark.}  The key difference between the analysis of CODA+ and that of CODA lies at how to handle the term $(\alpha_0 - \alpha)^2$ in Lemma~\ref{lem:1}. In CODA, the initial dual variable $\alpha_0$ is computed from the initial primal variable $\v_0$, which reduces the error term $(\alpha_0 - \alpha)^2$ to  one similar to $\|\v_0 - \v\|^2$, which is then bounded by the primal objective gap due to the PL condition. However, since we do not conduct the extra computation of $\alpha_0$ from $\v_0$, our analysis directly deals with such error term by using the duality gap of $f^s$. This technique is originally developed by~\cite{yan2020optimal}. 

\begin{thm}
\label{thm:coda_plus} 
Define $\hat{L} \hspace{-0.1in} =  \hspace{-0.1in} L\!+\!2\ell, c=\frac{\mu/\hat{L}}{5+\mu/\hat{L}}$. Set $\gamma \!=\! 2\ell$, $\eta_s = \eta_0 \exp(-(s-1)c)$, 
$T_s = \frac{212}{\eta_0 \min(\ell, \mu_2)} \exp( (s-1)c)$. 
 To return $\v_S$ such that $\E[\phi(\v_S) - \phi(\v^*_{\phi})] \leq \epsilon$, it suffices to choose $S\geq O\left(\frac{5\hat{L}+\mu}{\mu} \max\bigg\{\log \left(\frac{2\Delta_0}{\epsilon}\right), \log S + \log\bigg[ \frac{2\eta_0}{\epsilon} \frac{ 12(\sigma^2)}{5K}\bigg]\bigg\}\right)$.
 The iteration complexity is  $\widetilde{O}\bigg(\max\left(\frac{\Delta_0}{\mu \epsilon \eta_0 K}, \frac{\hat{L}}{\mu^2 K\epsilon}\right)\bigg)$ and the communication complexity is $\widetilde{O}\left(\frac{K}{\mu} \right)$ by setting $I_s = \Theta(\frac{1}{K\eta_s})$ if $D=0$, and is $\widetilde{O}\bigg(\max\left(\frac{K}{\mu} +  \frac{\Delta_0^{1/2}}{\mu(\eta_0 \epsilon)^{1/2}}, \frac{K}{\mu} +   \frac{\hat{L}^{1/2}}{\mu^{3/2}\epsilon^{1/2}}\right)\bigg)$ by setting $I_s = \Theta(\frac{1}{\sqrt{K\eta_s}})$ if $D>0$,  where $\widetilde{O}$ suppresses logarithmic factors.
\end{thm}

\textbf{Remark.} Due to the PL condition, the step size $\eta$ decreases geometrically. Accordingly, $I$ increases geometrically due to Lemma~\ref{lem:1}, and $I$ increases with a faster rate when the data are homogeneous than that when data are heterogeneous. In result, the total number of communications in homogeneous setting is much less than that in heterogeneous setting.

\section{CODASCA}
Although CODA+ has a highly reduced communication complexity for homogeneous data, it is still suffering from a high communication complexity for heterogeneous data. Even for the homogeneous data,  CODA+ has a worse communication complexity with a dependence on the number of clients $K$ than the NPA algorithm with a large batch size. 

\Ovalbox{\begin{minipage}[t]{0.95\columnwidth}%
\it Can we further reduce the communication complexity for FDAM for both homogeneous and heterogeneous data without using a large batch size?%
\end{minipage}}

The main reason for the degeneration in the heterogeneous data setting is the data difference. Even at global optimum $(\v_*, \alpha_*)$, the gradient of local functions in different clients could be different and non-zero. In the homogeneous data setting, different clients still produce different solutions due to stochastic error (cf. the $\eta^2 \sigma^2 I$ term of $A_1$ in Lemma~\ref{lem:1}). These together contribute to the client drift.  

To correct the client drift, we propose to leverage the idea of stochastic controlled averaging due to~\cite{karimireddy2019scaffold}. The key idea is to maintain and update a control variate  to accommodate the client drift, which is taken into account when updating the local solutions. 
In the proposed algorithm CODASCA, we apply control variates to both primal and dual variables. 
CODASCA shares the same stagewise framework as CODA+, where a strongly convex strongly concave subproblem is constructed and optimized in a distributed fashion approximately in each stage. The steps of CODASCA are presented in Algorithm \ref{alg:codasca_outer} and Algorithm \ref{alg:codasca_inner}. Below,  we describe the algorithm in each stage.

Each stage has $R$ communication rounds. Between two rounds, there are $I$ local updates, and each machine $k$ does the local updates as
\begin{equation*}
\begin{split}
& \v_{r,t+1}^k=\v^k_{r,t} - \eta_l (\nabla_\v F^s_k(\v^k_{r,t}, \alpha^k_{r,t}; \z^t_{r,t}) - c_\v^k + c_\v), \\
& \alpha_{r,t+1}^k=\alpha^k_{r,t} + \eta_l (\nabla_\alpha F^s_k(\v^k_{r,t}, \alpha^k_{r,t}; \z^k_{r,t})  - c_\alpha^k + c_\alpha),
\end{split}
\end{equation*} 
where $c_\v^k, c_\v$ are local and global control variates for the primal variable, and $c^k_\alpha, c_\alpha$ are local and global control variates for the dual variable.  Note that $\nabla_\v F^s_k(\v^k_{r,t}, \alpha^k_{r,t}; \z^t_{r,t})$ and $\nabla_\alpha F^s_k(\v^k_{r,t}, \alpha^k_{r,t}; \z^k_{r,t})$ are unbiased stochastic gradient on local data. However, they are biased estimate of global gradient when data on different clients are heterogeneous.
Intuitively, the term $-c_\v^k + c_\v$ and $-c_\alpha^k + c_\alpha$ work to correct the local gradients to get closer to the global gradient. They also play a role of reducing variance of stochastic gradients, which is helpful as well to reduce the communication complexity in the homogeneous data setting. 

At each communication round, the primal and dual variables on all clients get aggregated, averaged and broadcast to all clients. The control variates $c$ at $r$-th round get updated as
\begin{equation}
\begin{split}
& c_\v^k = c_\v^k - c_\v + \frac{1}{I \eta_l}(\v_{r-1} - \v^k_{r,I}), \\  
& c_\alpha^k = c_\alpha^k - c_\alpha + \frac{1}{I  \eta_l}(\alpha^k_{r,I}-\alpha_{r-1}),
\end{split}
\end{equation}
which is equivalent to 
\begin{equation}
\begin{split}
& c_\v^k = \frac{1}{I} \sum\limits_{t=1}^{I} \nabla_\v f^s_k(\v_{r,t}^k, \alpha^k_{r,t}; \z^k_{r,t}), \\  
& c_\alpha^k = \frac{1}{I}   \sum\limits_{t=1}^{I}\nabla_\alpha f^s_k(\v_{r,t}^k, \alpha^k_{r,t}; \z^k_{r,t}). 
\end{split}   
\end{equation} 
Notice that they are simply the average of stochastic gradients used in this round. An alternative way to compute the control variates is by computing the stochastic gradient with a large batch of extra samples at each client, but this would bring extra cost and is unnecessary. $c_\v$ and $c_\alpha$ are averages of $c_\v^k$ and $c_\alpha^k$ over all clients. After the local primal and dual variables are averaged, an extrapolation step with $\eta_g>1$ is performed, which will boost the convergence. 
 
In order to establish the convergence of CODASCA, we first present a key lemma below. 
\begin{lem} (One call of Algorithm \ref{alg:codasca_inner})
\label{lem:codasca:one_stage}
Under the same setting as in Theorem \ref{thm:main}, 
with $\teta =  \eta_l \eta_g I \leq  \frac{\mu_2}{40 \ell^2} $, 
for $\v' = \arg\min\limits_{\v} f^s(\v, \alpha_{\Tilde{r}}),\alpha'=\arg\max\limits_{\alpha} f^s(\v_{\Tilde{r}}, \alpha)$ we have
\begin{equation*}
\begin{split}
&\E[f^s(\v_{\Tilde{r}}, \alpha') - f^s(\v', \alpha_{\Tilde{r}})] \leq  \frac{2}{\eta_l\eta_g T} \|\v_0 - \v'\|^2
+ \frac{2}{\eta_l\eta_g T} (\alpha_0 -\alpha')^2 
+  \underbrace{\frac{10\eta_l\sigma^2}{\eta_g}}\limits_{A_2}+ \frac{10\eta_l\eta_g\sigma^2}{K}  \\ 
\end{split} 
\end{equation*}
where $T=I\cdot R$ is the number of iterations for each stage. 
\end{lem}
\textbf{Remark.} Compared the above bound with that in Lemma~\ref{lem:1}, in particular the term $A_2$ vs the term $A_1$, we can see that CODASCA will not be affected by the data heterogeneity $D>0$, and the stochastic variance is also much reduced.  As will seen in the next theorem, the value of $\teta$ and $R$ will keep the same in all stages. Therefore, by decreasing local step size $\eta_l$ geometrically, the communication window size $I_s$ will increase geometrically to ensure $\teta \leq O(1)$.
\setlength{\textfloatsep}{0.2cm}
\setlength{\floatsep}{0.2cm}
\begin{algorithm}[t]
\caption{CODASCA}
\begin{algorithmic}[1] 
\STATE{Initialization: $(\v_0, \alpha_0, \gamma)$.} 
\FOR{$s=1, ..., S$}
\STATE{$\v_s, \alpha_s = \text{DSGSCA+} (\v_{s-1}, \alpha_{s-1}, \eta_l, \eta_g, I_s, R_s, \gamma)$;}
\ENDFOR 
\STATE{Return $\v_{S}, \alpha_S$.}   
\end{algorithmic}
\label{alg:codasca_outer}
\end{algorithm}

\begin{algorithm}[t]
\caption {DSGSCA+($\v_0, \alpha_0, \eta_l, \eta_g, I, R, \gamma$)} 
\begin{algorithmic}
\STATE{Each machine does initialization: $\v_{0,0}^k = \v_0, \alpha_{0,0}^k = \alpha_0$, $c_\v^k = \mathbf{0}$, $c_\alpha^k = 0$}
\FOR{$r=1,...,R$}
\FOR{$t=0, 1, ..., I-1$} 
\STATE{Each machine $k$ updates its local solution in parallel:}
\STATE{~~~~$\v_{r,t+1}^k\!=\!\v^k_{r,t}\! -\! \eta_l (\nabla_\v F^s_k(\v^k_{r,t}, \alpha^k_{r,t}; \z^k_{r,t}) - c_\v^k + c_\v) $,}
\STATE{~~~~$\alpha_{r,t+1}^k\!=\! \alpha^k_{r,t} \!+\! \eta_l (\nabla_\alpha F^s_k(\v^k_{r,t}, \alpha^k_{r,t}; \z^k_{r,t})  \!-\! c_\alpha^k \!+\! c_\alpha)$,}  
\ENDFOR
\STATE{$c_\v^k = c_\v^k - c_\v + \frac{1}{I \eta_l}(\v_{r-1} - \v^k_{r,I}) $} 
\STATE{$c_\alpha^k = c_\alpha^k - c_\alpha + \frac{1}{I  \eta_l}(\alpha^k_{r,I}-\alpha_{r-1}) $}  
\STATE{$c_\v = \frac{1}{K} \sum\limits_{k=1}^K c_\v^k$, $c_\alpha = \frac{1}{K} \sum\limits_{k=1}^K{} c_\alpha^k$  \hfill $\diamond$ communicate  }
\STATE{$\v_r = \frac{1}{K} \sum\limits_{k=1}^{K} \v^k_{r, I}, \alpha_r = \frac{1}{K} \sum\limits_{k=1}^{K} \alpha^k_{r,t}$ \hfill $\diamond$ communicate } 
\STATE{$\v_r = \v_{r-1} + \eta_g (\v_r-\v_{r-1})$,}
\STATE{$\alpha_r = \alpha_{r-1} + \eta_g (\alpha_r-\alpha_{r-1})$}
\STATE{Broadcast $\v_r, \alpha_r, c_\v, c_\alpha$ \hfill $\diamond$ communicate}
\ENDFOR 
\STATE{Return $\v_{\tilde{r}}, \alpha_{\tilde{r}}$ where $\tilde{r}$ is randomly sampled from $1, ..., R$} 
\end{algorithmic}
\label{alg:codasca_inner}
\end{algorithm}

The convergence result of CODASCA is presented below. 
\begin{thm}
\label{thm:main}
Define $\hat{L} \!=\! L + \!2\ell$, $c \!=\! 4\ell \!+\! \frac{248}{53}\hat{L}$. 
Set 
$\eta_g = \sqrt{K}$, $I_s = I_0 \exp\left(\frac{2\mu_1}{c+2\mu_1}(s-1) \right)$,
$R= \frac{1000}{\teta \mu_2}$, 
$\eta_l^s = \frac{\teta}{\eta_g I_s} = \frac{\teta}{\sqrt{K} I_0 } \exp\left(-\frac{2\mu}{c+2\mu}  (s-1)\right)$, 
$\teta \leq \min\{\frac{1}{3\ell + 3\ell^2/\mu_2}, \frac{\mu_2}{40 \ell^2}\}$.
After $S=O( \max\bigg\{\frac{c+2\mu}{2\mu}\log \frac{4\epsilon_0}{\epsilon}, 
\frac{c+2\mu}{2\mu}\log \frac{160 \hat{L} S } {(c+2\mu)\epsilon} \frac{\teta\sigma^2}{K I_0} \bigg\})$ stages, the output $\v_S$ satisfies $\E[\phi(\v_S) - \phi(\v^*_\phi)]  \leq \epsilon$.
The communication complexity is $\widetilde{O}\left( \frac{1}{\mu} \right)$. 
The iteration complexity  is $\widetilde{O}\left( \max\{\frac{1}{\mu \epsilon}, \frac{1}{\mu^2 K  \epsilon}\}\right)$. 
\end{thm}
\vspace{-0.1in} 
\textbf{Remark.} 
(i) The number of communications is $\widetilde{O}\left( \frac{1}{\mu} \right)$, independent of number of clients $K$ and the accuracy level $\epsilon$. 
This is a significant improvement over CODA+, which has a communication complexity of $\widetilde{O}\left( K/\mu + 1/({\mu^{3/2} \epsilon^{1/2}}) \right)$
in heterogeneous setting. Moreover, $\widetilde{O}\left(1/({\mu}) \right)$ is a nearly optimal rate up to a logarithmic factor, since $O(1/\mu)$ is the lower bound communication complexity of distributed strongly convex optimization \cite{karimireddy2019scaffold,ArjevaniS15} and strongly convexity is a stronger condition than the PL condition. 

(ii) Each stage has the same number of communication rounds. However, $I_s$ increases geometrically. Therefore, the number of iterations and samples in a stage increase geometrically. Theoretically, we can also set $\eta^s_l$ to the same value as the one in the last stage, correspondingly $I_s$ can be set as a fixed large value. But this increases the number of required samples  without further speeding up the convergence. Our setting of $I_s$ is a balance between skipping communications and reducing sample complexity.  For simplicity,  we use the fixed setting of $I_s$ to compare  CODASCA and the baseline CODA+ in our experiment to corroborate the theory.

(iii) The local step size $\eta_l$ of CODASCA decreases similarly as the step size $\eta$ in CODA+. But $I_s =O(1/(\sqrt{K}\eta_l^s))$ in CODASCA increases faster than that $I_s= O(1/(\sqrt{K\eta_s}))$ in CODA+ on heterogeneous data. It is noticeable that different from CODA+, we do not need Assumption \ref{ass2} which bounds the client drift, meaning that CODASCA can be applied to optimize the global objective even if local objectives arbitrarily deviate from the global function.

\begin{figure}[t!]
    \centering
    {\includegraphics[scale=0.11]{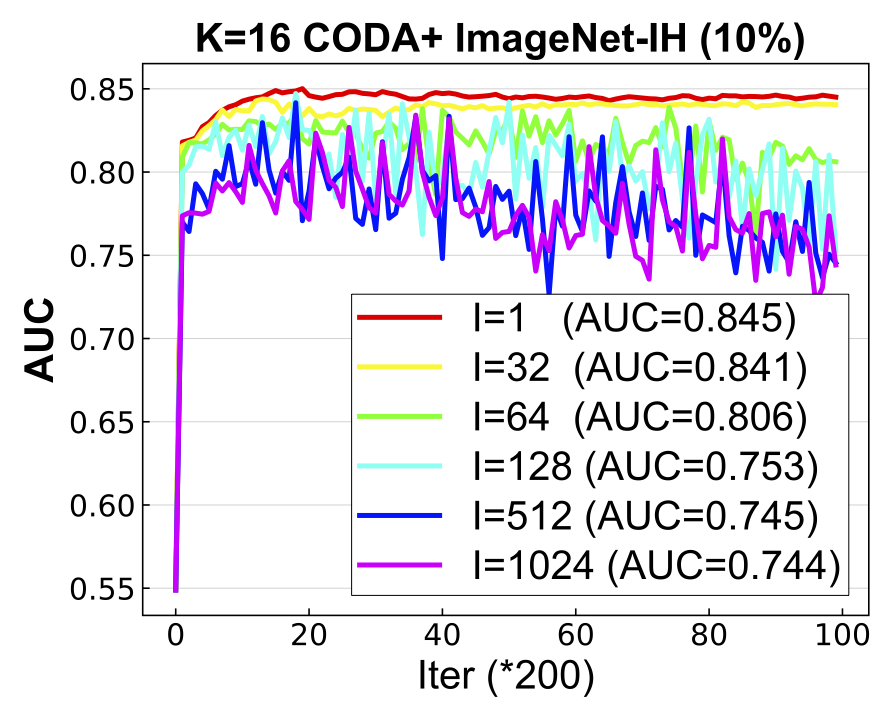}
    \includegraphics[scale=0.11]{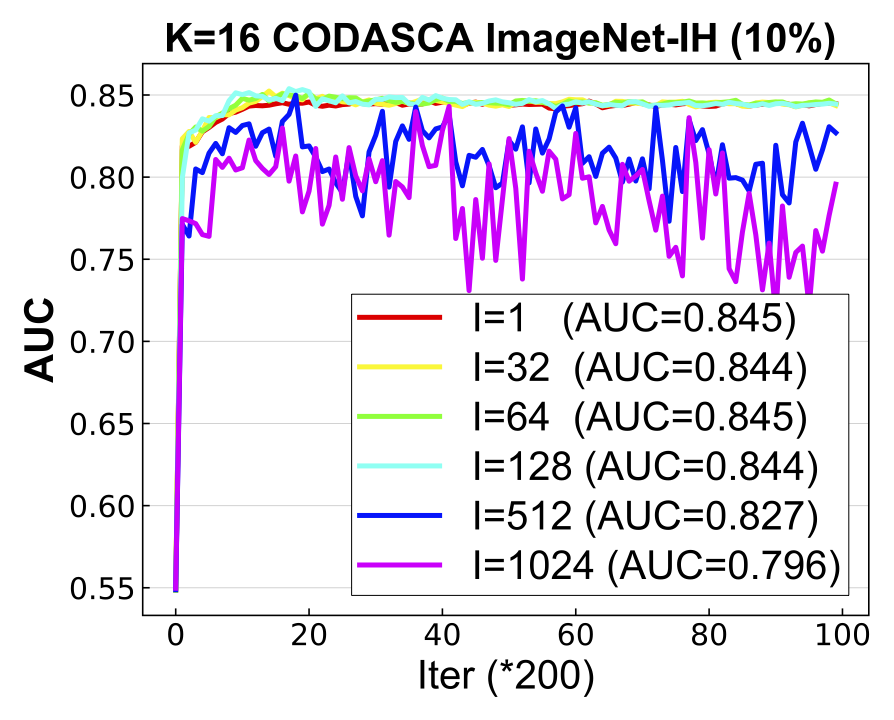}
    }
    {
    \includegraphics[scale=0.11]{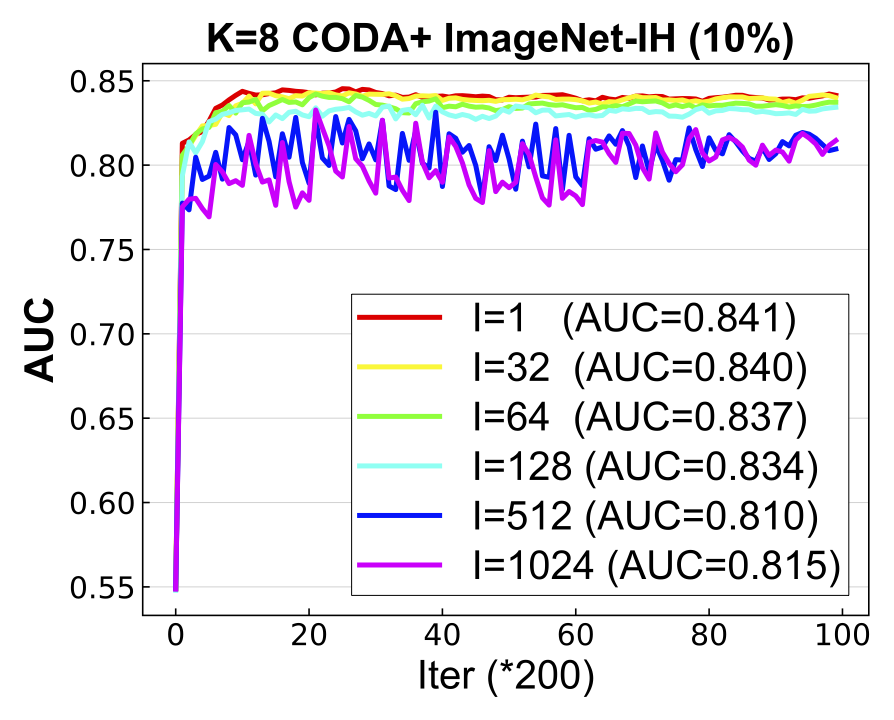}
    \includegraphics[scale=0.11]{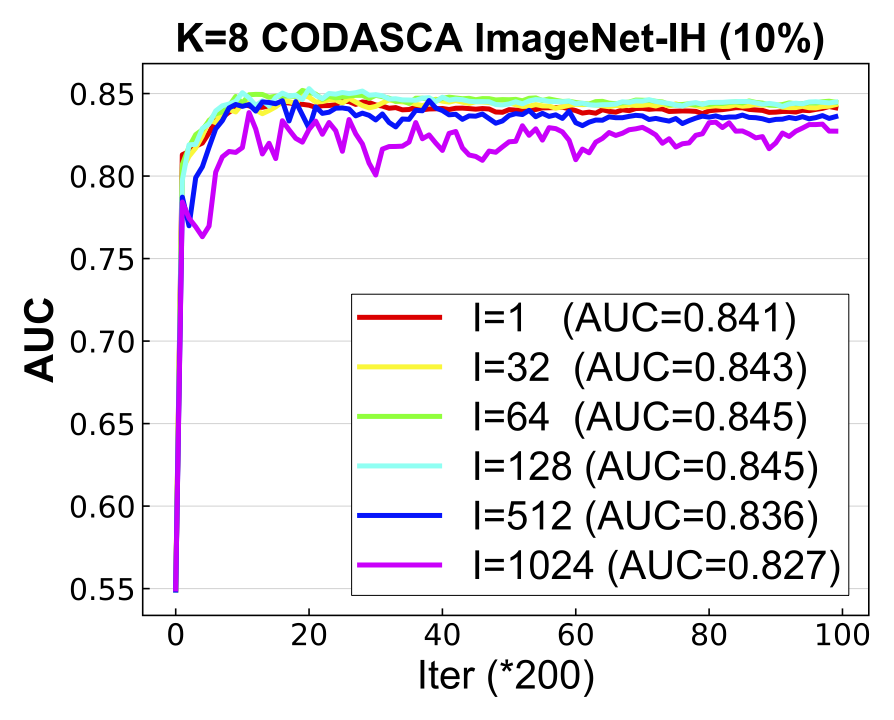}
    }
    \label{fig:imagenet_0.1}

    \centering
    {
    \includegraphics[scale=0.11]{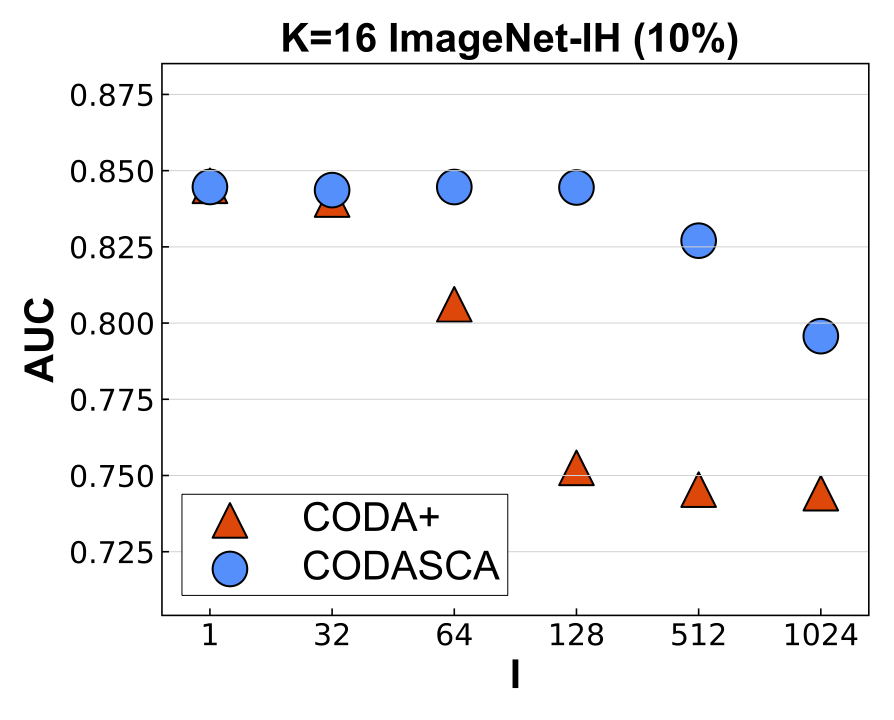}
    \includegraphics[scale=0.11]{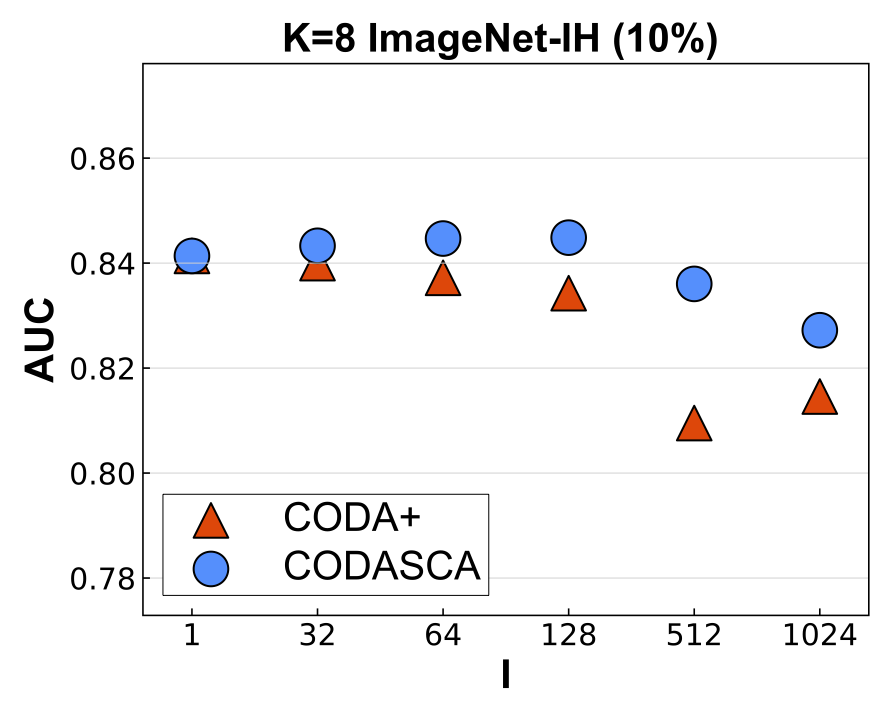}
    } 
    {
     \includegraphics[scale=0.11]{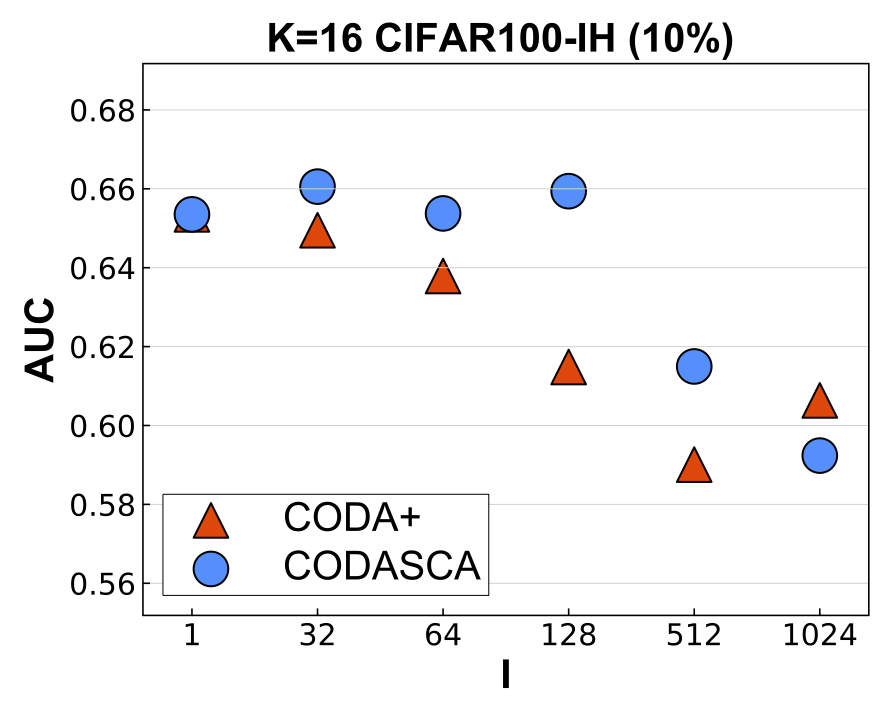}
     \includegraphics[scale=0.11]{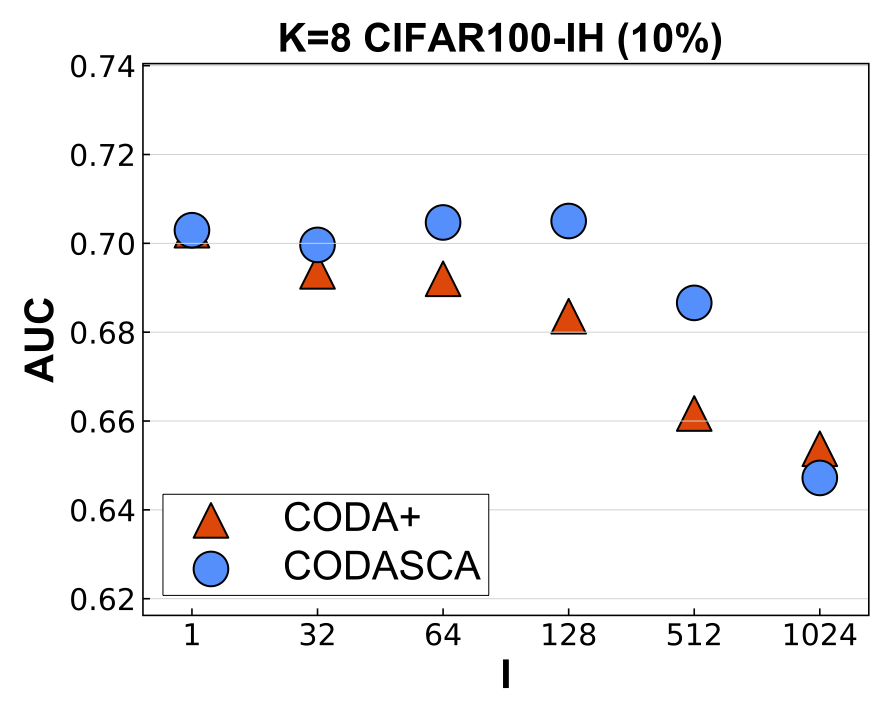}
    } 
    \caption{Top row: the testing AUC score of CODASCA vs \# of iterations for different values of $I$ on ImageNet-IH and CIFAR100-IH with imratio = 10\% and $K$=16, 8 on Densenet121. Bottom row: the achieved testing AUC vs different values of $I$ for CODASCA and CODA+. The AUC score in the legend in top row figures represent the AUC score at the last iteration.  }
    \label{fig:codasca}
    \vspace{-0.2in}
\end{figure}

\section{Experiments}
In this section, we first verify the effectiveness of CODASCA compared to CODA+ on various datasets, including two benchmark datasets, i.e., ImageNet, CIFAR100 \cite{imagenet_cvpr09, krizhevsky2009cifar} and a constructed large-scale chest X-ray dataset. Then, we demonstrate the effectiveness of FDAM on improving the performance on a single domain (CheXpert) by using data from multiple sources. For notations, $K$ denotes the number of ``clients" (\# of machines, \# of data sources)  and $I$ denotes the communication window size. The code used for
the experiments are available at \url{https://github.com/Optimization-AI/ICML2021_FedDeepAUC_CODASCA/}. 

\textbf{Chest X-ray datasets}. Five medical chest X-ray datasets, i.e., CheXpert, ChestXray14, MIMIC-CXR, PadChest, ChestXray-AD \cite{chexpert19, wang2017chestx, johnson2019mimic, bustos2020padchest, nguyen2020vindr} are collected from different organizations. The statistics of these medical datasets are summarized in Table \ref{table:xray_stat}. We construct five binary classification tasks for predicting  five popular diseases, Cardiomegaly (C0), Edema (C1), Consolidation (C2), Atelectasis (C3), P. Effusion (C4), as in CheXpert competition~\cite{chexpert19}. These datasets are naturally imbalanced and heterogeneous due to different patients' populations, different data collection protocols and etc. We refer to the whole medical dataset as ChestXray-IH. 
\begin{table}[t]
\centering
\caption{Statistics of Medical Chest X-ray Datasets.}
\scalebox{0.9}{
\begin{tabular}{ccc}
\hline
\textbf{Dataset} & \textbf{Source} & \textbf{Samples}  \\ \hline
CheXpert      & Stanford Hospital (US)   & 224,316                  \\ 
ChestXray8  & NIH Clinical Center (US)    & 112,120              \\
PadChest  & Hospital San Juan (Spain)       & 110,641             \\ 
MIMIC-CXR   & BIDMC (US)     & 377,110             \\
ChestXrayAD   & H108 and HMUH (Vietnam)   & 15,000               \\ \hline
\end{tabular}}
\label{table:xray_stat}%
\end{table}

\textbf{Imbalanced and Heterogeneous (IH) Benchmark Datasets.} For benchmark datasets, we manually construct the imbalanced heterogeneous dataset. For ImageNet, we first randomly select 500 classes as positive class and 500 classes as negative class. To increase data heterogeneity, we further split all positive/negative classes into $K$ groups so that each split only owns samples from unique classes without overlapping with that of other groups. To increase data imbalance level, we randomly remove some samples from positive classes for each machine. Please note that due to this operation, the whole sample set for different $K$ is different.  We refer to the proportion of positive samples in all samples as imbalance ratio ($imratio$). For CIFAR100, we follow similar steps to construct imbalanced heterogeneous data. We keep the testing/validation set untouched and keep them balanced. For imbalance ratio (imratio), we explore two ratios: 10\% and 30\%. We refer to the constructed datasets as ImageNet-IH (10\%), ImageNet-IH (30\%), CIFAR100-IH (10\%), CIFAR100-IH (30\%).   Due to the limited space, we only report imratio=10\% with DenseNet121 and defer the other results to supplement. 

\textbf{Parameters and Settings.} 
We train Desenet121 on all datasets.  For the parameters in CODASCA/CODA+, we tune $1/\gamma$ in [500, 700, 1000] and $\eta$ in [0.1, 0.01, 0.001]. For learning rate schedule, we decay the step size by 3 times every $T_0$ iterations, where $T_0$ is tuned in [2000, 3000, 4000]. We experiment with a fixed value of $I$ selected from [1, 32, 64, 128, 512, 1024] and we include experiments with increasing $I_s$ in the supplement. We tune $\eta_g$ in [1.1, 1, 0.99, 0.999]. The local batch size is set to 32 for each machine. We run a total of 20000 iterations for all experiments.  

\subsection{Comparison with CODA+}
We plot the testing AUC on ImageNet (10\%) vs \# of iterations for CODASCA and CODA+ in Figure~\ref{fig:codasca} (top row) by varying the value of $I$ for different values of $K$. Results on CIFAR100 are shown in the Supplement.  In the bottom row of Figure~\ref{fig:codasca}, we plot the achieved testing AUC score vs different values of $I$ for CODASCA and CODA+. 
We have the following observations: 
\vskip -0.05in
\noindent $\bullet~$\textbf{CODASCA enjoys a larger communication window size}. Comparing CODASCA and CODA+ in the bottom panel of Figure~\ref{fig:codasca}, we can see that CODASCA enjoys a larger communication window size without hurting the performance than CODA+, which is consistent with our theory. 

\noindent $\bullet~$  \textbf{CODASCA is consistently better for different values of $K$}. We compare the largest value of $I$ such that the performance does not degenerate too much compared with $I=1$, which is denoted by $I_{\max}$. From the bottom figures of Figure~\ref{fig:codasca}, we can see that the $I_{\max}$ value of CODASCA on ImageNet is 128 ($K$=16) and 512 ($K$=8), respectively, and that of CODA+ on ImageNet is 32 ($K$=16) and 128 ($K$=8). This demonstrates that CODASCA enjoys consistent advantage over CODA+, i.e., when $K=16$, $I^{\text{CODASCA}}_{\max}/I^{\text{CODA+}}_{\max}=4$, and when $K=8$,  $I^{\text{CODASCA}}_{\max}/I^{\text{CODA+}}_{\max}=4$. The same phenomena occur on CIFAR100 data. 

Next, we compare CODASCA with CODA+ on the ChestXray-IH medical dataset, which is also highly heterogeneous.  We  split the ChestXray-IH data into $K=16$ groups according to the patient ID and each machine only owns samples from one organization without overlapping patients. The testing set is the collection of 5\% data sampled from each organization. In addition, we use train/val split = 7:3 for the parameter tuning. We run CODASCA and CODA+ with the same number of iterations. The performance on testing set are reported in Table \ref{tab:chestxray_auc}. From the results, we can observe that CODASCA performs consistently better than CODA+ on C0, C2, C3, C4. 

\begin{table}[t]
\centering
\caption{Performance on ChestXray-IH testing set when $K$=16.}
\scalebox{0.8}{
\begin{tabular}{ccccccc}
\hline
\textbf{Method} &     $I$     & \textbf{C0}     & \textbf{C1}     & \textbf{C2}    & \textbf{C3}     & \textbf{C4}     \\ \hline
  & 1      & 0.8472 & 0.8499 & 0.7406 & 0.7475 & 0.8688  \\  \hline
CODA+ & 512      & 0.8361 & \textbf{0.8464} & 0.7356 & 0.7449 & 0.8680  \\ 
CODASCA & 512 & \textbf{0.8427} & 0.8457 & \textbf{0.7401} & \textbf{0.7468} & \textbf{0.8680} \\ 
\hline
CODA+  & 1024     & 0.8280  & \textbf{0.8451} & 0.7322 & 0.7431 & 0.8660  \\ 
CODASCA & 1024 & \textbf{0.8363} & 0.8444 & \textbf{0.7346} & \textbf{0.7481} & \textbf{0.8674} \\ \hline
\end{tabular}}
\label{tab:chestxray_auc}
\end{table}

\begin{table}[t]
\caption{Performance of FDAM on Chexpert validation set for DenseNet121. } 
\centering
\scalebox{0.8}{
\begin{tabular}{ccccccc}
\hline
\textbf{\#of sources} & \textbf{C0} & \textbf{C1}     & \textbf{C2}    & \textbf{C3}     & \textbf{C4} & \textbf{AVG}     \\ \hline
$K$=1    & 0.9007      & 0.9536          & 0.9542         & 0.9090 & 0.9571  &  0.9353         \\ 
$K$=2             & 0.9027      & 0.9586          & 0.9542         & 0.9065          & 0.9583 &  0.9361        \\
$K$=3             & 0.9021      & 0.9558          & \textbf{0.9550} & 0.9068          & 0.9583  &   0.9356       \\
$K$=4             & 0.9055      & \textbf{0.9603} & 0.9542         & 0.9072          & \textbf{0.9588} & 0.9372  \\
$K$=5             &  \textbf{0.9066}     & 0.9583          & 0.9544         & \textbf{0.9101}          & 0.9584  &  \textbf{0.9376} \\\hline
\end{tabular}}
\label{table:chexpert_auc}
\end{table}
\begin{table}[t]
\caption{Performance of FDAM on Chexpert validation set for DenSenet161.} 
\centering
\scalebox{0.8}{
\begin{tabular}{ccccccc}
\hline
\textbf{\#of sources} & \textbf{C0} & \textbf{C1} & \textbf{C2} & \textbf{C3} & \textbf{C4} & \textbf{AVG} \\ \hline
K=1 & 0.8946 & 0.9527 & 0.9544 & 0.9008 & 0.9556 & 0.9316 \\
K=2 & 0.8938 & 0.9615 & 0.9568 & 0.9109 & 0.9517 & 0.9333 \\
K=3 & \textbf{0.9008} & 0.9603 & 0.9568 & 0.9127 & 0.9505 & 0.9356 \\ 
K=4 & 0.8986 & \textbf{0.9615} & 0.9561 & 0.9128 & \textbf{0.9564} & 0.9367 \\ 
K=5 & 0.8986 & 0.9612 & \textbf{0.9568} & \textbf{0.9130} & 0.9552 & \textbf{0.9370} \\ \hline
\end{tabular}}
\label{tab:densenet161}
\end{table}

\vspace{-0.1in}
\subsection{FDAM for improving performance on CheXpert}
Finally, we show that FDAM can be used to leverage data from multiple hospitals  to improve the performance at a single target hospital. For this experiment, we choose CheXpert data from Stanford Hospital as the target data. Its validation data will be used for evaluating the performance of our FDAM method. Note that improving the AUC score on CheXpert is a very challenging task. The top 7 teams on CheXpert leaderboard differ by only 0.1\%~\footnote{\url{https://stanfordmlgroup.github.io/competitions/chexpert/}}. Hence, we consider any improvement over $0.1\%$ significant. Our procedure is following: we gradually increase the number of data resources, e.g., $K=1$ only includes the CheXpert training data, $K=2$ includes the CheXpert training data and ChestXray8, $K=3$ includes  the CheXpert training data and ChestXray8 and PadChest, and so on. 


\textbf{Parameters and Settings}. Due to the limited computing resources, we resize all images to 320x320. We follow the two stage method proposed in \cite{yuan2021robust} and compare with the baseline on a single machine with a single data source (CheXpert training data) ($K$=1) for learning DenseNet121, DenseNet161. More specifically, we first train a base model by minimizing the Cross-Entropy loss on CheXpert training dataset using Adam with a initial learning rate of 1e-5 and batch size of 32 for 2 epochs. Then, we discard the trained classifier, use the same pretrained model for initializing the local models at all machines and continue training using CODASCA. For the parameter tuning, we try $I$=[16, 32, 64, 128], learning rate=[0.1, 0.01] and we fix $\gamma$=1e-3, $T_0$=1000 and batch size=32. 

\textbf{Results.} We report all results in term of AUC score on the CheXpert validation data in Table \ref{table:chexpert_auc} and Table~\ref{tab:densenet161}. 
We can see that using more data sources from different organizations can efficiently improve the performance on CheXpert. For DenseNet121, the average improvement across all 5 classification tasks from $K=1$ to $K=5$ is over $0.2\%$ which is significant in light of the top CheXpert leaderboard results.  Specifically, we can see that CODASCA with $K$=5 achieves the highest validation AUC score on C0 and C3, and with $K$=4 achieves the highest on C1 and C4. For DenseNet161, the improvement of average AUC is over 0.5\%, which doubles the 0.2\% improvement for DenseNet121.

\vskip -0.4in
\section{Conclusion}
In this work, we have conducted comprehensive studies of federated learning for deep AUC maximization. We analyzed a stronger baseline for deep AUC maximization by establishing its convergence for both homogeneous data and heterogeneous data. 
We also developed an improved variant by adding control variates to the local stochastic gradients for both primal and dual variables, which dramatically reduces the communication complexity. Besides a strong theory guarantee, we exhibit the power of FDAM on real world medical imaging problems. We have shown that our FDAM method can improve the performance on medical imaging classification tasks by leveraging  data from different organizations that are kept locally. 

\section*{Acknowledgements}
We are grateful to the anonymous reviewers for their constructive comments and suggestions. This work is partially supported by NSF \#1933212 and NSF CAREER Award \#1844403.


\bibliography{reference,reference2} 

\onecolumn
\appendix
\section{Auxiliary Lemmas}

Noting all algorithms discussed in thpaper including the baselines implement a stagewise framework, we define the duality gap of $s$-th stage at a point $(\v, \alpha)$ as
\begin{equation}
\begin{split}
Gap_s(\v, \alpha) = \max\limits_{\alpha'} f^s(\v, \alpha') 
- \min\limits_{\v'} f^s(\v', \alpha).  
\end{split}
\end{equation}
Before we show the proofs, we first present the lemmas from \cite{yan2020optimal}.

\begin{lem} [Lemma 1 of \cite{yan2020optimal}]
Suppose a function $h(\v, \alpha)$ is $\lambda_1$-strongly convex in $\v$ and $\lambda_2$-strongly concave in $\alpha$. 
Consider the following problem 
\begin{align*}
\min\limits_{\v\in X} \max\limits_{\alpha\in Y} h(\v, \alpha), 
\end{align*}
\label{lem:Yan1}
where $X$ and $Y$ are convex compact sets.
Denote $\hat{\v}_h (y) = \arg\min\limits_{\v'\in X} h(\v', \alpha)$ 
and $\hat{\alpha}_h(\v) = \arg\max\limits_{\alpha' \in Y} h(\v, \alpha')$.
Suppose we have two solutions $(\v_0, \alpha_0)$ and $(\v_1, \alpha_1)$.
Then the following relation between variable distance and duality gap holds
\begin{align}
\begin{split}
\frac{\lambda_1}{4} \|\hat{\v}_h(\alpha_1)-\v_0\|^2 + \frac{\lambda_2}{4}\|\hat{\alpha}_h(\v_1) - \alpha_0\|^2 \leq& 
\max\limits_{\alpha' \in Y} h(\v_0, \alpha') - \min\limits_{\v' \in X}h(\v', \alpha_0) \\ 
&+ \max\limits_{\alpha' \in Y} h(\v_1, \alpha') - \min\limits_{\v' \in X} h(\v', \alpha_1).
\end{split}
\end{align}
\end{lem} 
$\hfill \Box$

\begin{lem}[Lemma 5 of \cite{yan2020optimal}]
We have the following lower bound for $\text{Gap}_s(\v_s, \alpha_s)$ 
\begin{align*} 
\text{Gap}_s(\v_s, \alpha_s) \geq  \frac{3}{50}\text{Gap}_{s+1}(\v_0^{s+1}, \alpha_0^{s+1}) + \frac{4}{5}(\phi(\v_0^{s+1}) - \phi(\v_0^s)),
\end{align*} 
\label{lem:Yan5} 
\end{lem}
where $\v_0^{s+1} = \v_s$ and $\alpha_0^{s+1} = \alpha_s$, i.e., the initialization of $(s+1)$-th stage is the output of the $s$-th stage.

$\hfill \Box$

\section{Analysis of CODA+}
The proof sketch is similar to the proof of CODA in \cite{dist_auc_guo}. 
However, there are two noticeable difference from \cite{dist_auc_guo}. First, in Lemma \ref{lem:one_stage_codaplus}, we bound the duality gap instead of the objective gap in \cite{dist_auc_guo}. This is because  the analysis later in this proof requires the bound of the duality gap. 

Second, in Lemma \ref{lem:one_stage_codaplus}, where the bound for homogeneous data is better than that of heterogeneous data. The better analysis for homogeneous data is inspired by the analysis in \cite{yu_linear}, which tackles a minimization problem. Note that $f^s$ denotes the subproblem for stage $s$, we omit the index $s$ in variables when the context is clear.

\subsection{Lemmas}
We need following lemmas for the proof. The Lemma \ref{lem:codaplus_one_state_split}, Lemma \ref{lem:A1_codaplus} and Lemma \ref{lem:A2_codaplus} are similar to Lemma 3, Lemma 4 and Lemma 5 of \cite{dist_auc_guo}, respectively.
For the sake of completeness, we will include the proof of Lemma \ref{lem:codaplus_one_state_split} and Lemma \ref{lem:A1_codaplus} since a change in the update of the primal variable.  

\begin{lem}
\label{lem:codaplus_one_state_split}
Define $\bar{\v}_t = \frac{1}{K}\sum_{k=1}^N \v^k_t, \bar{\alpha}_t =  \frac{1}{K}\sum_{k=1}^N y^k_t$.  Suppose Assumption \ref{assumption_1} holds and by running Algorithm \ref{alg:codaplus_inner}, we have for any $\v, \alpha$, 
\begin{align*}
f^s(\bar{\v}, \alpha) - f^s(\v, \bar{\alpha}) 
&\leq \frac{1}{T}\sum\limits_{t=1}^{T}\bigg[ \underbrace{\langle \nabla_{\v}f(\bar{\v}_{t-1}, \bar{\alpha}_{t-1}),\! \bar{\v}_t - x\rangle }_{B_1}  + \underbrace{ \langle \nabla_{\alpha} f(\bar{\v}_{t-1}, \bar{\alpha}_{t-1}), y - \bar{\alpha}_t \rangle }_{B_2} \\ 
&
~~~ + \underbrace{\frac{3\ell + 3\ell^2/\mu_2}{2}\|\bar{\v}_t - \bar{\v}_{t-1}\|^2  \!+ 2\ell (\bar{\alpha}_t - \bar{\alpha}_{t-1})^2}_{B_3}  
- \frac{\ell}{3} \| \bar{\v}_t - \v\|^2   - \frac{\mu_2}{3} (\bar{\alpha}_{t-1} -\alpha)^2\bigg],
\end{align*}
where $\mu_2 = 2p(1-p)$ is the strong concavity coefficient of $f(\v, \alpha)$ in $\alpha$. 
\end{lem}

\begin{proof} 
For any $\v$ and $\alpha$,
using Jensen's inequality and the fact that $f^s(\v, \alpha)$ is convex in $\v$ and concave in $\alpha$, 
\begin{small}
\begin{equation}
\begin{split}
& f^s (\bar{\v}, \alpha) - f^s(\v, \bar{\alpha}) \leq \frac{1}{T}\sum\limits_{t=1}^{T} \left(f^s(\bar{\v}_t, \alpha) - f^s(\v, \bar{\alpha}_t)\right)\\ 
\end{split} 
\label{gap_relax}  
\end{equation} 
\end{small}

By $\ell$-strongly convexity of $f^s(\v, \alpha)$ in $\v$, we have
\begin{equation} 
f^s(\bar{\v}_{t-1}, \bar{\alpha}_{t-1}) + \langle \partial_\v f^s(\bar{\v}_{t-1}, \bar{\alpha}_{t-1}), \v - \bar{\v}_{t-1} \rangle + \frac{\ell}{2} \|\bar{\v}_{t-1}-\v\|^2 \leq f(\v, \bar{\alpha}_{t-1}).
\label{strong_v} 
\end{equation}

By $3\ell$-smoothness of $f^s(\v, \alpha)$ in $\v$, we have
\begin{equation}
\begin{split}
& f^s(\bar{\v}_t, \alpha) \leq f^s(\bar{\v}_{t-1}, \alpha) + \langle \partial_\v f^s(\bar{\v}_{t-1}, \alpha), \bar{\v}_t-\bar{\v}_{t-1}\rangle + \frac{3\ell}{2}\|\bar{\v}_t-\bar{\v}_{t-1}\|^2 \\ 
&= f^s(\bar{\v}_{t-1}, \alpha) + \langle \partial_\v f^s(\bar{\v}_{t-1}, \bar{\alpha}_{t-1}), \bar{\v}_t-\bar{\v}_{t-1}\rangle + \frac{3\ell}{2}\|\bar{\v}_t-\bar{\v}_{t-1}\|^2 \\
&~~~+ \langle \partial_\v f^s(\bar{\v}_{t-1}, \alpha)-\partial_\v f^s(\bar{\v}_{t-1}, \bar{\alpha}_{t-1}), \bar{\v}_t-\bar{\v}_{t-1}\rangle \\
&\overset{(a)}\leq f^s(\bar{\v}_{t-1}, \alpha) + \langle \partial_{\v}f^s(\bar{\v}_{t-1}, \bar{\alpha}_{t-1}), \bar{\v}_t - \bar{\v}_{t-1}\rangle + \frac{3\ell}{2}\|\bar{\v}_t - \bar{\v}_{t-1}\|^2 \\
&~~~~ + \ell |\bar{\alpha}_{t-1} - \alpha| \|\bar{\v}_t - \bar{\v}_{t-1}\| \\
&\overset{(b)}\leq f^s(\bar{\v}_{t-1}, \alpha) + \langle \partial_{\v}f^s(\bar{\v}_{t-1}, \bar{\alpha}_{t-1}), \bar{\v}_t - \bar{\v}_{t-1}\rangle + \frac{3\ell}{2}\|\bar{\v}_t - \bar{\v}_{t-1}\|^2 \\
&~~~~ + \frac{\mu_2}{6} (\bar{\alpha}_{t-1} - \alpha)^2 + \frac{3\ell^2}{2\mu_2} \|\bar{\v}_t - \bar{\v}_{t-1}\|^2, \\ 
\end{split}  
\label{smooth_v}
\end{equation}
where $(a)$ holds because that we know $\partial_{\v} f(\v, \alpha)$ is $\ell$-Lipschitz in $\alpha$ since $f(\v, \alpha)$ is $\ell$-smooth, $(b)$ holds by Young's inequality, and $\mu_2=2p(1-p)$ is the strong concavity coefficient of $f^s$ in $\alpha$. 


Adding (\ref{strong_v})  and (\ref{smooth_v}),  rearranging terms, we have
\begin{small} 
\begin{equation}  
\begin{split}
& f^s(\bar{\v}_{t-1}, \bar{\alpha}_{t-1}) + f^s(\bar{\v}_t, \alpha)  \\
& \leq 
f(\v, \bar{\alpha}_{t-1})
+
f(\bar{\v}_{t-1}, \alpha) + \langle \partial_{\v}f(\bar{\v}_{t-1}, \bar{\alpha}_{t-1}), \bar{\v}_t - \v\rangle  + \frac{3\ell + 3\ell^2/\mu_{2}}{2} \|\bar{\v}_t-\bar{\v}_{t-1}\|^2  
 \\
&~~~ - \frac{\ell}{2}\|\bar{\v}_{t-1}-\v\|^2   
 + \frac{\mu_{2}}{6} (\bar{\alpha}_{t-1} - \alpha)^2 .
\end{split} 
\label{sum_v}  
\end{equation}
\end{small}

We know $f^s(\v, \alpha)$ is $\mu_2$-strong concavity in $\alpha$ ($-f(\v, \alpha)$ is $\mu_{2}$-strong convexity of  in $\alpha$).
Thus, we have 
\begin{small}
\begin{equation}
\begin{split} 
-f^s(\bar{\v}_{t-1}, \bar{\alpha}_{t-1}) 
- \partial_{\alpha}f^s(\bar{\v}_{t-1}, \bar{\alpha}_{t-1})^\top(\alpha - \bar{\alpha}_{t-1}) + \frac{\mu_{2}}{2}(\alpha - \bar{\alpha}_{t-1})^2 \leq -f^s(\bar{\v}_{t-1}, \alpha).
\end{split} 
\label{strong_alpha} 
\end{equation}
\end{small}

Since $f(\v, \alpha)$ is $\ell$-smooth in $\alpha$, we get
\begin{small} 
\begin{equation} 
\begin{split}
& - f^s(\v, \bar{\alpha}_t) \leq -f^s(\v, \bar{\alpha}_{t-1}) - \langle \partial_{\alpha}f^s(\v, \bar{\alpha}_{t-1}), \bar{\alpha}_t - \bar{\alpha}_{t-1} \rangle + \frac{\ell}{2}(\bar{\alpha}_t - \bar{\alpha}_{t-1})^2\\
&= -f^s(\v, \bar{\alpha}_{t-1}) 
- \langle \partial_{\alpha} f^s(\bar{\v}_{t-1}, \bar{\alpha}_{t-1}), \bar{\alpha}_t - \bar{\alpha}_{t-1}\rangle 
+ \frac{\ell}{2}(\bar{\alpha}_t - \bar{\alpha}_{t-1})^2 \\
&~~~~~ - \langle \partial_{\alpha} (f^s(\v, \bar{\alpha}_{t-1}) - f^s(\bar{\v}_{t-1}, \bar{\alpha}_{t-1})), \bar{\alpha}_t - \bar{\alpha}_{t-1}\rangle  \\
&\overset{(a)}\leq -f^s(\v, \bar{\alpha}_{t-1}) - \langle \partial_{\alpha} f^s(\bar{\v}_{t-1}, \bar{\alpha}_{t-1}), \bar{\alpha}_t - \bar{\alpha}_{t-1}\rangle + \frac{\ell}{2} (\bar{\alpha}_t - \bar{\alpha}_{t-1})^2 + \ell\| \v - \bar{\v}_{t-1}\| (\bar{\alpha}_t - \bar{\alpha}_{t-1})\\
&\leq -f^s(\v, \bar{\alpha}_{t-1}) - \langle \partial_{\alpha} f^s(\bar{\v}_{t-1}, \bar{\alpha}_{t-1}), \bar{\alpha}_t - \bar{\alpha}_{t-1}\rangle 
+ \frac{\ell}{2} (\bar{\alpha}_t - \bar{\alpha}_{t-1})^2 + \frac{\ell}{6} \| \bar{\v}_{t-1} - \v\|^2  +  \frac{3\ell}{2}(\bar{\alpha}_t - \bar{\alpha}_{t-1})^2\\   
\end{split} 
\label{smooth_alpha}
\end{equation}
\end{small}
where (a) holds because that $\partial_{\alpha} f^s(\v, \alpha)$ is $\ell$-Lipschitz in $\v$. 

Adding (\ref{strong_alpha}), (\ref{smooth_alpha}) and arranging terms, we have
\begin{small}
\begin{equation}
\begin{split}
& -f^s(\bar{\v}_{t-1}, \bar{\alpha}_{t-1}) 
- f^s(\v, \bar{\alpha}_t) 
\leq-f^s(\bar{\v}_{t-1}, \alpha) 
-f^s(\v, \bar{\alpha}_{t-1}) - \langle \partial_{\alpha} f^s(\bar{\v}_{t-1}, \bar{\alpha}_{t-1}), \bar{\alpha}_t - \alpha\rangle \\
&~~~~~~ + 2\ell (\bar{\alpha}_t - \bar{\alpha}_{t-1})^2 
+ \frac{\ell}{6} \| \bar{\v}_{t-1} - \v\|^2 
- \frac{\mu_{2}}{2}(\alpha - \bar{\alpha}_{t-1})^2.
\end{split}   
\label{sum_alpha} 
\end{equation}
\end{small}

Adding (\ref{sum_v}) and (\ref{sum_alpha}), we get
\begin{equation}
\begin{split}
 &f^s(\bar{\v}_t, \alpha) - f^s(\v, \bar{\alpha}_t) \leq   
\langle \partial_{\v}f(\bar{\v}_{t-1}, \bar{\alpha}_{t-1}), \bar{\v}_t - \v\rangle - \langle \partial_{\alpha} f(\bar{\v}_{t-1}, \bar{\alpha}_{t-1}), \bar{\alpha}_t - \alpha\rangle \\  
&~~~ +\frac{3\ell + 3\ell^2/\mu_{2}}{2}   \|\bar{\v}_t-\bar{\v}_{t-1}\|^2 
 + 2\ell (\bar{\alpha}_t - \bar{\alpha}_{t-1})^2 -\frac{\ell}{3}\|\bar{\v}_{t-1}-\v\|^2   - \frac{\mu_2}{3} (\bar{\alpha}_{t-1}-\alpha)^2. 
\end{split}
\end{equation}

Taking average over $t = 1, ..., T$, we get
\begin{small} 
\begin{equation*} 
\begin{split} 
&f^s(\bar{\v}, \alpha) - f^s(\v, \bar{\alpha}) \leq \frac{1}{T} \sum\limits_{t=1}^{T} 
[f^s(\bar{\v}_t, \alpha) - f^s(\v, \bar{\alpha}_t)] \\
& \leq \frac{1}{T} \sum\limits_{t=1}^{T}\bigg[ \underbrace{\langle 
\partial_{\v}f^s(\bar{\v}_{t-1}, \bar{\alpha}_{t-1}), \bar{\v}_t - \v\rangle }_{B_1} 
+ \underbrace{\langle \partial_{\alpha}  f^s(\bar{\v}_{t-1}, \bar{\alpha}_{t-1}),  \alpha - \bar{\alpha}_t \rangle}_{B_2}  \\
&~~~~~~~~~~ 
+\underbrace{\frac{3\ell + 3\ell^2/\mu_{2}}{2}  \|\bar{\v}_t-\bar{\v}_{t-1}\|^2 
 + 2\ell (\bar{\alpha}_t - \bar{\alpha}_{t-1})^2 }_{B_3} 
- \frac{\ell}{3} \|\v - \bar{\v}_t\|^2 
- \frac{\mu_2}{3}(\bar{\alpha}_{t-1}-\alpha)^2  \bigg].
\end{split} 
\end{equation*}      
\end{small}
\end{proof}

In the following, we will bound the term $B_1$ by Lemma \ref{lem:A1_codaplus}, $B_2$ by Lemma \ref{lem:A2_codaplus} and $B_3$ by Lemma \ref{lem:diverge_codaplus}.  

\begin{lem}
\label{lem:A1_codaplus}
Define $\hat{\v}_t = \Bar{\v}_{t-1} - \frac{\eta}{K}\sum\limits_{k=1}^{K}\! \nabla_{\v} f^s(\v^k_{t-1}, \alpha^k_{t-1})$ and 
\begin{equation}
\begin{split} 
\Tilde{\v}_t = \Tilde{\v}_{t-1} - \frac{\eta}{K} \sum\limits_{k=1}^{K} \left(\nabla_\v F_k^s(\v^k_{t-1}, y^k_{t-1}; z_{t-1}^k) - \nabla_\v f_k^s(\v_{t-1}^k, \alpha_{t-1}^k) \right), \text{ for $t>0$; } \Tilde{\v}_0 = \v_0.  
\end{split}
\end{equation}.
 We have
\begin{equation*} 
\begin{split}
&B_1 \leq \frac{3\ell}{2}\frac{1}{K} \sum\limits_{k=1}^{K} (\bar{\alpha}_{t-1}-\alpha_{t-1}^k)^2 + \frac{3 \ell}{2}\frac{1}{K}\sum\limits_{k=1}^{K}\|\bar{\v}_{t-1}-\v_{t-1}^k\|^2 \\
&~~~~~+ \frac{3\eta}{2} \left\|\frac{1}{K}\sum\limits_{k=1}^{K}[\nabla_{\v}f_k(\v_{t-1}^k, \alpha_{t-1}^k) - \nabla_{\v}F_k(\v_{t-1}^k, \alpha_{t-1}^k; z_{t-1}^k)]\right\|^2 \\
&~~~~~+\left\langle \frac{1}{K}\sum\limits_{k=1}^{K}[\nabla_{\v}f_k(\v_{t-1}^k, \alpha_{t-1}^k)-\nabla_{\v}F_k(\v_{t-1}^k, \alpha_{t-1}^k;z_{t-1}^k)], \hat{\v}_t - \Tilde{\v}_{t-1}\right\rangle\\  
&~~~~~+\frac{1}{2\eta} (\|\bar{\v}_{t-1}-\v\|^2 - \|\bar{\v}_{t-1} - \bar{\v}_t\|^2 - \|\bar{\v}_t - \v\|^2) \\
&~~~~~+ \frac{\ell}{3}\|\bar{\v}_t - \v\|^2
+ \frac{1}{2\eta} ( \|\v - \tilde{\v}_{t-1}\|^2 - \|\v-\Tilde{\v}_t\|^2 )  
\end{split} 
\end{equation*}
\end{lem}

\begin{proof} 
We have 
\begin{equation}
\begin{split}
&\langle \nabla_{\v}f^s(\bar{\v}_{t-1}, \bar{\alpha}_{t-1}), \bar{\v}_t - \v\rangle= \bigg\langle \frac{1}{K}\sum\limits_{k=1}^{K} \nabla_{\v}f^s_k(\bar{\v}_{t-1}, \bar{\alpha}_{t-1}), \bar{\v}_t - \v\bigg\rangle \\ 
&\leq \bigg\langle \frac{1}{K}\sum\limits_{k=1}^{K} [\nabla_{\v}f^s_k(\bar{\v}_{t-1}, \bar{\alpha}_{t-1}) - \nabla_{\v} f^s_k(\bar{\v}_{t-1}, \alpha_{t-1}^k)], \bar{\v}_t - \v\bigg\rangle ~~~~~~~~~~~\textcircled{\small{1}}\\
&~~~ +\bigg\langle \frac{1}{K}\sum\limits_{k=1}^{K} [\nabla_{\v}f^s_k(\bar{\v}_{t-1}, \alpha_{t-1}^k) - \nabla_{\v}f^s_k(\v_{t-1}^k, \alpha_{t-1}^k)], \bar{\v}_t - \v\bigg\rangle ~~~~~~~~~~~\textcircled{\small{2}}\\
&~~~ +\bigg\langle \frac{1}{K}\sum\limits_{k=1}^{K}[\nabla_{\v}f^s_k(\v_{t-1}^k,\alpha_{t-1}^k) - \nabla_{\v}F^s_k(\bar{\v}_{t-1}, \alpha_{t-1}^k; z_{t-1}^k)], \bar{\v}_t - \v\bigg\rangle~~~~~~~~~~~\textcircled{\small{3}}  \\
&~~~ +\bigg\langle \frac{1}{K}\sum\limits_{k=1}^{K}\nabla_{\v}F^s_k(\bar{\v}_{t-1}, \alpha_{t-1}^k; z_{t-1}^k), \bar{\v}_t - \v\bigg\rangle~~~~~~~~~~~\textcircled{\small{4}} 
\end{split}
\label{circledv}
\end{equation}

Then we will bound \textcircled{\small{1}}, \textcircled{\small{2}}, \textcircled{\small{3}} and \textcircled{\small{4}}, respectively,   
\begin{equation}
\begin{split}
\textcircled{\small{1}} &\overset{(a)}\leq \frac{3}{2\ell} \left\| \frac{1}{K} \sum\limits_{k=1}^{K}[\nabla_{\v}f^s_k(\bar{\v}_{t-1}, \bar{\alpha}_{t-1}) - \nabla_{\v}f^s_k(\bar{\v}_{t-1}, \alpha_{t-1}^{k})] \right\|^2 + \frac{\ell}{6}\|\bar{\v}_t - \v\|^2\\
&\overset{(b)}\leq \frac{3}{2 \ell} \frac{1}{K}\sum\limits_{k=1}^{K} \|\nabla_{\v}f^s_k(\bar{\v}_{t-1}, \bar{\alpha}_{t-1}) - \nabla_{\v} f^s_k(\bar{\v}_{t-1}, \alpha_{t-1}^k)\|^2 + \frac{\ell}{6}\|\bar{\v}_t - \v\|^2\\
&\overset{(c)}\leq \frac{3\ell}{2}\frac{1}{K}\sum\limits_{k=1}^{K}(\bar{\alpha}_{t-1} - \alpha_{t-1}^k)^2 + \frac{\ell}{6}\|\bar{\v}_t - \v\|^2,\\ 
\end{split}
\label{circled1} 
\end{equation}
where (a) follows from Young's inequality, (b) follows from Jensen's inequality.
and (c) holds because $\nabla_{\v} f^s_k(\v, \alpha)$ is $\ell$-Lipschitz in $\alpha$.
Using similar techniques, we have 
\begin{equation} 
\begin{split}   
\textcircled{\small{2}} &\leq \frac{3}{2\ell} \frac{1}{K}\sum\limits_{k=1}^{K}\| \nabla_{\v} f^s_k(\bar{\v}_{t-1}, \alpha_{t-1}^k) - \nabla_{\v} f^s_k(\v_{t-1}^k, \alpha_{t-1}^k)\|^2 + \frac{\ell}{6}\|\bar{\v}_t -\v\|^2 \\
& \leq \frac{3 \ell}{2}\frac{1}{K}\sum\limits_{k=1}^{K} \|\bar{\v}_{t-1} - \v_{t-1}^{k}\|^2 + \frac{\ell}{6} \|\bar{\v}_t - \v\|^2.
\end{split}
\label{circled2}
\end{equation}

Let $\hat{\v}_t = \arg\min\limits_{\v} \left(\frac{1}{K}\sum\limits_{k=1}^{K} \nabla_{\v} f^s(\v^k_{t-1}, \alpha^k_{t-1})\right)^\top x + \frac{1}{2\eta} \|\v - \bar{\v}_{t-1}\|^2$, 
then we have 
\begin{equation}
\begin{split}
\bar{\v}_t - \hat{\v}_t =\eta \bigg(\nabla_{\v}f^s(\v^{k}_{t-1}, y^k_{t-1}) - \frac{1}{K}\sum\limits_{k=1}^{K}\nabla_{\v}f^s_k(\v^{k}_{t-1}, y^k_{t-1}; z_{t-1}^k)\bigg) 
\end{split}
\end{equation}

Hence we get 
\begin{equation}
\begin{split}
&\textcircled{\small{3}} = \left\langle \frac{1}{K}\sum\limits_{k=1}^{K}[\nabla_{\v}f^s_k(\v_{t-1}^k, \alpha_{t-1}^k) - \nabla_{\v} F^s_k(\v_{t-1}^k, \alpha_{t-1}^k; z_{t-1}^k)], \bar{\v}_t - \hat{\v}_t \right\rangle \\
&~~~~+ \left\langle \frac{1}{K}\sum\limits_{k=1}^{K}[\nabla_{\v}f^s_k(\v_{t-1}^k, \alpha_{t-1}^k) - \nabla_{\v} F^s_k(\v_{t-1}^k, \alpha_{t-1}^k; z_{t-1}^k)], \hat{\v}_t - \v\right\rangle\\
& = \eta \left\|\frac{1}{K} \sum\limits_{k=1}^{K}[\nabla_{\v}f^s_k(\v_{t-1}^k, \alpha_{t-1}^k) - \nabla_{\v} F^s_k(\v_{t-1}^k, \alpha_{t-1}^k; z_{t-1}^k)]  \right\|^2\\
&~~~~+ \left\langle \frac{1}{K}\sum\limits_{k=1}^{K}[\nabla_{\v}f^s_k(\v_{t-1}^k, \alpha_{t-1}^k) - \nabla_{\v} F^s_k(\v_{t-1}^k, \alpha_{t-1}^k; z_{t-1}^k)], \hat{\v}_t - \v\right\rangle\\
\end{split}
\label{pre_circled3}
\end{equation} 

Define another auxiliary sequence as
\begin{equation}
\begin{split}
\Tilde{\v}_t = \Tilde{\v}_{t-1} - \frac{\eta}{K} \sum\limits_{k=1}^{K} \left(\nabla_\v F^s_k(\v^k_{t-1}, y^k_{t-1}; z_{t-1}^k) - \nabla_\v f^s_k(\v_{t-1}^k, \alpha_{t-1}^k) \right), \text{ for $t>0$; } \Tilde{\v}_0 = \v_0. 
\end{split}
\end{equation}

Denote 
\begin{equation}
\Theta_{t-1} (\v) = \left(-\frac{1}{K}\sum\limits_{k=1}^{K} (\nabla_\v F^s_k(\v^k_{t-1}, y^k_{t-1}; z_{t-1}^k) - \nabla_\v f^s_k(\v_{t-1}^k, \alpha_{t-1}^k) )\right)^\top x + \frac{1}{2\eta} \|\v - \Tilde{\v}_{t-1}\|^2.
\end{equation} 
Hence, for the auxiliary sequence $\Tilde{\alpha}_t$, we can verify that
\begin{equation}
\Tilde{\v}_t = \arg\min\limits_\v \Theta_{t-1}(\v).
\end{equation}
Since $\Theta_{t-1}(\v)$ is $\frac{1}{\eta}$-strongly convex, we have
\begin{small}
\begin{align} 
\begin{split}
& \frac{1}{2}\|\v - \tilde{\v}_t\|^2 \leq \Theta_{t-1}(\v) - \Theta_{t-1}(\tilde{\v}_{t}) \\
& = \bigg(-\frac{1}{K}\sum\limits_{k=1}^{K}( 
 \nabla_{\v} F^s_k(\v_{t-1}^{k},\alpha_{t-1}^k; z_{t-1}^k) 
-\nabla_{\v} f^s_k(\v_{t-1}^{k},\alpha_{t-1}^k))\bigg)^\top x + \frac{1}{2\eta}\|\v - \tilde{\v}_{t-1}\|^2 \\  
&~~~ -\bigg(-\frac{1}{K}\sum\limits_{k=1}^{K}( 
 \nabla_{\v} F^s_k(\v_{t-1}^{k},\alpha_{t-1}^k; z_{t-1}^k) 
-\nabla_{\v} f^s_k(\v_{t-1}^{k},\alpha_{t-1}^k))\bigg)^\top\tilde{\v}_{t} 
- \frac{1}{2\eta}\|\tilde{\v}_{t}  - \tilde{\v}_{t-1}\|^2 \\ 
& = \bigg(-\frac{1}{K}\sum\limits_{k=1}^{K}(  
 \nabla_{\alpha} F^s_k(\v_{t-1}^{k},\alpha_{t-1}^k; z_{t-1}^k) 
-\nabla_{\alpha} f_k(\v_{t-1}^{k},\alpha_{t-1}^k))\bigg)^\top (\v - \tilde{\v}_{t-1}) 
+ \frac{1}{2\eta}\|\v - \tilde{\v}_{t-1}\|^2 \\ 
&~~~ -\bigg(-\frac{1}{K} \sum\limits_{k=1}^{K}(  
 \nabla_{\alpha} F^s_k(\v_{t-1}^{k},\alpha_{t-1}^k; z_{t-1}^k) 
-\nabla_{\alpha} f^s_k(\v_{t-1}^{k},\alpha_{t-1}^k))\bigg)^\top 
(\tilde{\v}_{t} - \tilde{\v}_{t-1})  
- \frac{1}{2\eta}\|\tilde{\v}_{t}  - \tilde{\v}_{t-1}\|^2 \\
&\leq \bigg(-\frac{1}{K}\sum\limits_{k=1}^{K}(
 \nabla_{\v} F^s_k(\v_{t-1}^{k},\alpha_{t-1}^k; z_{t-1}^k)  
-\nabla_{\v} f^s_k(\v_{t-1}^{k},\alpha_{t-1}^k))\bigg)^\top (\v - \tilde{\v}_{t-1})+ \frac{1}{2\eta}\|\v - \tilde{\v}_{t-1}\|^2 \\ 
&~~~+ \frac{\eta}{2}\bigg\|\frac{1}{K}\sum\limits_{k=1}^{K}( 
 \nabla_{\v} F^s_k(\v_{t-1}^{k},\alpha_{t-1}^k; z_{t-1}^k) 
-\nabla_{\v} f^s_k(\v_{t-1}^{k},\alpha_{t-1}^k))\bigg\|^2
\end{split} 
\end{align} 
\end{small} 

Adding this with (\ref{pre_circled3}), we get
\begin{equation} 
\begin{split}
\textcircled{3} \leq &\frac{3\eta}{2}\bigg\|\frac{1}{K}\sum\limits_{k=1}^{K}( 
 \nabla_{\v} F_k(\v_{t-1}^{k},\alpha_{t-1}^k; z_{t-1}^k) 
-\nabla_{\v} f_k(\v_{t-1}^{k},\alpha_{t-1}^k))\bigg\|^2 + \frac{1}{2\eta}\|\v - \tilde{\v}_{t-1}\|^2
- \frac{1}{2}\|\v - \tilde{\v}_t\|^2 \\
&+ \left\langle \frac{1}{K}\sum\limits_{k=1}^{K}[\nabla_{\v}f_k(\v_{t-1}^k, \alpha_{t-1}^k) - \nabla_{\v} F_k(\v_{t-1}^{k}, \alpha_{t-1}^k; z_{t-1}^k)], \hat{\v}_t - \Tilde{\v}_{t-1}  \right\rangle
\end{split}
\label{circled3} 
\end{equation}

\textcircled{4} can be bounded as 
\begin{equation}
\textcircled{4} = -\frac{1}{\eta}\langle \Bar{\v}_t - \Bar{\v}_{t-1}, \Bar{\v}_t - \v \rangle 
= \frac{1}{2\eta} (\|\Bar{\v}_{t-1} - \v\|^2 - \|\Bar{\v}_{t-1} - \Bar{\v}_t\|^2 - \|\Bar{\v}_t - \Bar{\v}\|^2)  
\label{circled4_1} 
\end{equation}

Plug (\ref{circled1}), (\ref{circled2}), (\ref{circled3}) and (\ref{circled4_1}) into (\ref{circledv}), we get
\begin{equation*} 
\begin{split}
& \left\langle \nabla_{\v} f(\bar{\v}_{t-1}, \bar{\alpha}_{t-1}), \bar{\v}_t - x\right\rangle \\
& \leq \frac{3\ell}{2}\frac{1}{K} \sum\limits_{k=1}^{K} (\bar{\alpha}_{t-1}-\alpha_{t-1}^k)^2 + \frac{3 \ell}{2}\frac{1}{K}\sum\limits_{k=1}^{K}\|\bar{\v}_{t-1}-\v_{t-1}^k\|^2 \\
&~~~~~+ \frac{3\eta}{2} \left\|\frac{1}{K}\sum\limits_{k=1}^{K}[\nabla_{\v}f_k(\v_{t-1}^k, \alpha_{t-1}^k) - \nabla_{\v}F_k(\v_{t-1}^k, \alpha_{t-1}^k; z_{t-1}^k)]\right\|^2 \\
&~~~~~+\left\langle \frac{1}{K}\sum\limits_{k=1}^{K}[\nabla_{\v}f_k(\v_{t-1}^k, \alpha_{t-1}^k)-\nabla_{\v}F_k(\v_{t-1}^k, \alpha_{t-1}^k;z_{t-1}^k)], \hat{\v}_t - \Tilde{\v}_{t-1}\right\rangle\\  
&~~~~~+\frac{1}{2\eta} (\|\bar{\v}_{t-1}-\v\|^2 - \|\bar{\v}_{t-1} - \bar{\v}_t\|^2 - \|\bar{\v}_t - \v\|^2) \\
&~~~~~+ \frac{\ell}{3}\|\bar{\v}_t - \v\|^2
+ \frac{1}{2\eta} ( \|\v - \tilde{\v}_{t-1}\|^2 - \|\v-\Tilde{\v}_t\|^2 )  
\end{split} 
\label{grad_v}
\end{equation*}
\end{proof}

$B_2$ can be bounded by the following lemma, whose proof is identical to that of Lemma 5 in \cite{dist_auc_guo}. 
\begin{lem} 
\label{lem:A2_codaplus}
Define $\hat{\alpha}_t\! = \!\bar{\alpha}_{t-1} + \frac{\eta}{K}\sum\limits_{k=1}^{K} \nabla_{\alpha} f_k(\v_{t-1}^k, \alpha_{t-1}^k)$, 
and  
\begin{small}
\begin{equation*}
\tilde{\alpha}_{t}\! =\! \tilde{\alpha}_{t-1}\! +\! \frac{\eta}{K}\sum\limits_{k=1}^{K}( 
 \nabla_{\alpha} F_k(\v_{t-1}^k,\alpha_{t-1}^k;z_{t-1}^k)  
\!-\!\nabla_{\alpha} f_k(\v_{t-1}^k,\!\alpha_{t-1}^k) ). 
\end{equation*} 
\end{small}
\vspace{-0.1in} 
We have, 
\begin{equation*}  
\begin{split}
&B_2\leq \frac{3\ell^2}{2\mu_2} \frac{1}{K}\sum\limits_{k=1}^{K}\|\bar{\v}_{t-1}-\v_{t-1}^k\|^2  + \frac{3\ell^2}{2\mu_2}\frac{1}{K}\sum\limits_{k=1}^{K} (\bar{\alpha}_{t-1} - \alpha_{t-1}^k)^2\\
&~~~ +\frac{3\eta}{2} \left(\frac{1}{K} \sum\limits_{k=1}^{K}[ \nabla_{\alpha} f_k(\v_{t-1}^k, \alpha_{t-1}^k) -  \nabla_{\alpha} F_k (\v_{t-1}^k, \alpha_{t-1}^k; z_{t-1})] \right)^2\\
&~~~ + \frac{1}{K}\sum\limits_{k=1}^{K} \langle \nabla_{\alpha} f_k(\v_{t-1}^k, \alpha_{t-1}^k)-  \nabla_{\alpha} F_i(\v_{t-1}^k, \alpha_{t-1}^k; z_{t-1}^k), \tilde{\alpha}_{t-1} - \hat{\alpha}_t \rangle\\ 
&~~~ +\frac{1}{2\eta} ((\bar{\alpha}_{t-1} - \alpha)^2 - (\bar{\alpha}_{t-1} - \bar{\alpha}_t)^2 - (\bar{\alpha}_t - \alpha)^2)\\
&~~~ + \frac{\mu_2}{3} (\bar{\alpha}_t - \alpha)^2
+ \frac{1}{2\eta}(\alpha -  \tilde{\alpha}_{t-1})^2 - \frac{1}{2\eta}(\alpha - \tilde{\alpha}_{t})^2. 
\end{split} 
\end{equation*}
\end{lem}

$\hfill \Box$

$B_3$ can be bounded by the following lemma.
\begin{lem}
\label{lem:diverge_codaplus}
If $K$ machines communicate every $I$ iterations, where $I \leq \frac{1}{18\sqrt{2} \eta \ell}$, then\\
\vspace{-0.2in} 
\begin{small} 
\begin{align*} 
&\sum\limits_{t=0}^{T-1} \frac{1}{K}\sum\limits_{k=1}^{K} \E\left[\|\bar{\v}_t - \v_t^k\|^2 + \|\bar{\alpha}_t-\alpha_t^k\|^2 \right] \leq \left(12\eta^2 I \sigma^2 T + 36\eta^2 
I^2 D^2 T\right)\mathbb I_{I>1}\\ 
\end{align*}
\end{small} 
\end{lem}

\begin{proof}
In this proof, we introduce a couple of new notations to make the proof brief:
$F^s_{k, t} = F^s_{k, t}(\v^k_t, \alpha^k_t; z_{t}^k)$ and $f^s_{k, t} = f^s_{k, t}(\v^k_t, \alpha^k_t)$.
Similar bounds for minimization problems have been analyzed in \citep{yu_linear,stich2018local}. 

Denote $t_0$ as the nearest communication round before $t$, i.e., $t-t_0\leq I$. 
By the update rule of $\v$, we have that on each machine $k$, 
\begin{equation}
\begin{split}
\v_{t}^k = \Bar{\v}_{t_0} - \eta \sum\limits_{\tau=t_0}^{t-1} \nabla_{\v}F^s_{k, \tau}.
\end{split} 
\end{equation} 
Taking average over all $K$ machines,
\begin{equation}
\begin{split}
\Bar{\v}_t = \Bar{\v}_{t_0}-\eta \sum\limits_{\tau=t_0}^{t-1} \frac{1}{K}\sum\limits_{k=1}^{K} \nabla_{\v} F^s_{k, \tau}. 
\end{split} 
\end{equation}

Therefore, 
\begin{small}
\begin{equation}
\begin{split} 
&\frac{1}{K}\sum\limits_{k=1}^{K} \|\Bar{\v}_t - \v_t^k\|^2 = 
\frac{\eta^2}{K}\sum\limits_{k=1}^{K}  
\E\left[\left\| \sum\limits_{\tau=t_0}^{t-1}  \left[ \nabla_{\v} F^s_{k, \tau} -
\frac{1}{K}\sum\limits_{j=1}^{K} \nabla_{\v} F^s_{j, \tau}
\right] \right\|^2\right]  \\ 
& \leq  
\frac{2\eta^2 }{K} \sum\limits_{k=1}^{K} \left[ \left\| \sum\limits_{\tau=t_0}^{t-1}  \left[ [\nabla_{\v} F^s_{k, \tau} - \nabla_{\v} f^s_{k, \tau} ] - 
\frac{1}{K}\sum\limits_{j=1}^{K} \left[\nabla_{\v} F^s_{j, \tau} - \nabla_{\v} f^s_{j,\tau}  \right] 
\right] \right\|^2 \right] \\
&~~~ +  \frac{2\eta^2}{K} \sum\limits_{k=1}^{K} \E\left[\left\| \sum\limits_{\tau=t_0}^{t-1} \left[\nabla_{\v} f^s_{k,\tau} -  \frac{1}{K}\sum\limits_{j=1}^{K} \nabla_{\v} f^s_{j, \tau} \right] \right\|^2\right]
\end{split} 
\label{equ:local_proof_lem7_open}
\end{equation} 
\end{small}

In the following, we will address these two terms on the right hand side separately. 
First, we have 
\begin{equation}
\begin{split}
& \frac{2\eta^2}{K} \sum\limits_{k=1}^{K} \left[ \left\| \sum\limits_{\tau=t_0}^{t-1}  \left[ [\nabla_{\v} F^s_{k, \tau} - \nabla_{\v} f^s_{k, \tau} ] - 
\frac{1}{K}\sum\limits_{j=1}^{K} \left[\nabla_{\v} F^s_{j, \tau} - \nabla_{\v} f^s_{j,\tau}  \right] 
\right] \right\|^2 \right] \\
&\overset{(a)}{\leq} \frac{ 2\eta^2 }{K} \sum\limits_{k=1}^{K} \left[ \left\| \sum\limits_{\tau=t_0}^{t-1}  \left[\nabla_{\v} F^s_{k, \tau} - \nabla_{\v} f^s_{k, \tau} 
\right] \right\|^2 \right] \\
&\overset{(b)}{=} \frac{2\eta^2 }{K} \sum\limits_{k=1}^{K} \sum\limits_{\tau=t_0}^{t-1}  \left[ \left\| \left[\nabla_{\v} F^s_{k, \tau} - \nabla_{\v} f^s_{k, \tau} 
\right] \right\|^2 \right] \leq  2\eta^2 I \sigma^2, 
\end{split}  
\label{equ:local_proof_lem7_fitst}
\end{equation} 
where $(a)$ holds by $\frac{1}{K}\sum\limits_{k=1}^{K} \|a_k-\left[\frac{1}{K}\sum\limits_{j=1}^{K}a_j\right]\|^2 = \frac{1}{K}\sum\limits_{k=1}^{K}\|a_k\|^2 - \|\frac{1}{K}\sum\limits_{k=1}^{K} a_k\|^2 \leq \frac{1}{K}\sum\limits_{k=1}^{K} \|a_k\|^2$, where $a_k = \sum\limits_{\tau=t_0}^{t-1}[\nabla F^s_{k, \tau} - \nabla_{\v} f_{k,\tau}]$; $(b)$ follows because $\E_{k, \tau-1}[\nabla_\v F^s_{k,\tau}-\nabla_\v f^s_{k,\tau}]=0$.

Second, we have
\begin{equation}
\begin{split} 
& \frac{1}{K} \sum\limits_{k=1}^{K}  \E\left[\left\| \sum\limits_{\tau=t_0}^{t-1}  \left[\nabla_{\v} f^s_{i,\tau} -  \frac{1}{K}\sum\limits_{j=1}^{K} \nabla_{\v} f^s_{j, \tau} \right] \right\|^2\right] \\ 
&\leq \frac{1}{K} \sum\limits_{k=1}^{K} (t-t_0) \sum\limits_{\tau=t_0}^{t-1} \E\left[ \left\| \nabla_{\v} f^s_{i,\tau} -  \frac{1}{K}\sum\limits_{j=1}^{K} \nabla_{\v} f^s_{j, \tau} \right\|^2 \right] \\
&\leq  I \sum\limits_{\tau=t_0}^{t-1} \frac{1}{K} \sum\limits_{k=1}^{K} \E\left[ \left\| \nabla_{\v} f^s_{k,\tau} -  \frac{1}{K}\sum\limits_{j=1}^{K} \nabla_{\v} f^s_{j, \tau} \right\|^2 \right],  
\end{split}  
\label{equ:local_proof_lem7_second}
\end{equation}  
where 
\begin{small}
\begin{equation}
\begin{split}
&\frac{1}{K} \sum\limits_{k=1}^{K} \E \left\| \nabla_{\v} f^s_{k,\tau} -  \frac{1}{K}\sum\limits_{j=1}^{K} \nabla_{\v} f^s_{j, \tau} \right\|^2 \\
&= \frac{1}{K} \sum\limits_{k=1}^{K} \E \bigg\| \nabla_{\v} f^s_{k,\tau} - \nabla_{\v} f^s_k(\Bar{\v}_{\tau}, \bar{\alpha}_{\tau}) + \nabla_{\v} f^s_k(\Bar{\v}_{\tau}, \bar{\alpha}_{\tau}) 
- \nabla_{\v} f^s(\Bar{\v}_{\tau}, \Bar{\alpha}_{\tau}) +  \nabla_{\v} f^s(\Bar{\v}_{\tau}, \Bar{\alpha}_{\tau}) 
-  \frac{1}{K}\sum\limits_{j=1}^{K} \nabla_{\v} f^s_{j, \tau} \bigg\|^2  \\ 
& \leq \frac{1}{K} \sum\limits_{k=1}^{K} \bigg[ 
  3\E\|\nabla_{\v} f^s_{k, \tau} - \nabla_{\v} f_k(\Bar{\v}_{\tau}, \Bar{\alpha}_{\tau}) \|^2  
+ 3\E\|\nabla_{\v} f^s_k(\Bar{\v}_{\tau}, \Bar{\alpha}_{\tau}) - \nabla_{\v} f^s(\Bar{\v}_{\tau}, \Bar{\alpha}_{\tau}) \|^2\bigg] \\
&~~~
+ 3\E\bigg\|\nabla_{\v} f^s(\Bar{\v}_{\tau}, \Bar{\alpha}_{\tau}) 
- \frac{1}{K}\sum\limits_{j=1}^{K} \nabla_{\v} f^s_{j,\tau}\bigg\|^2 \\
& = \frac{1}{K} \sum\limits_{k=1}^{K} \bigg[ 
  3\E\|\nabla_{\v} f^s_{k, \tau} - \nabla_{\v} f^s_k(\Bar{\v}_{\tau}, \Bar{\alpha}_{\tau}) \|^2  
+ 3\E\|\nabla_{\v} f^s_k(\Bar{\v}_{\tau}, \Bar{\alpha}_{\tau}) - \nabla_{\v} f^s(\Bar{\v}_{\tau}, \Bar{\alpha}_{\tau}) \|^2 \bigg]\\
&~~~  
+ 3\E\bigg\|\frac{1}{K}\sum\limits_{j=1}^{K} [\nabla_{\v} f^s_j (\Bar{\v}_{\tau}, \Bar{\alpha}_{\tau}) 
- \nabla_{\v} f^s_{j,\tau}] \bigg\|^2 \bigg] \\ 
& \leq \frac{1}{K} \sum\limits_{k=1}^{K} \bigg[ 
  3\E\|\nabla_{\v} f^s_{k, \tau} - \nabla_{\v} f_k(\Bar{\v}_{\tau}, \Bar{\alpha}_{\tau}) \|^2  
+ 3\E\|\nabla_{\v} f^s_k(\Bar{\v}_{\tau}, \Bar{\alpha}_{\tau}) - \nabla_{\v} f^s(\Bar{\v}_{\tau}, \Bar{\alpha}_{\tau}) \|^2 \bigg]\\
&~~~   
+ 3\frac{1}{K}\sum\limits_{j=1}^{K} \E\bigg\| [\nabla_{\v} f^s_j (\Bar{\v}_{\tau}, \Bar{\alpha}_{\tau}) 
- \nabla_{\v} f^s_{j,\tau}] \bigg\|^2 \bigg] \\  
&\overset{(a)}{\leq} \frac{54\ell^2}{K} \sum\limits_{k=1}^{K} \left[ \|\v_{k, \tau} - \Bar{\v}_{\tau}\|^2 + |\alpha_{k, \tau}-\Bar{\alpha}_{\tau} |^2  \right] 
+ \frac{3}{K}\sum\limits_{k=1}^{K} \|\nabla_{\v} f^s_k(\Bar{\v}_{\tau}, \Bar{\alpha}_{\tau}) - \nabla_{\v} f^s(\Bar{\v}_{\tau}, \Bar{\alpha}_{\tau})\|^2 \\
&\leq \frac{54\ell^2}{K} \sum\limits_{k=1}^{K} \left[ \|\v_{k, \tau} - \Bar{\v}_{\tau}\|^2 + |\alpha_{k, \tau}-\Bar{\alpha}_{\tau} |^2  \right] 
+ 3D^2,
\end{split} 
\label{equ:local_proof_lem7_second_continue}
\end{equation} 
\end{small}
where $(a)$ holds because $f$ is $\ell$-smooth, i.e., $f^s$ is $3\ell$-smooth. 
 
Combining (\ref{equ:local_proof_lem7_open}), (\ref{equ:local_proof_lem7_fitst}), (\ref{equ:local_proof_lem7_second}) and (\ref{equ:local_proof_lem7_second_continue}), 
\begin{small} 
\begin{equation}
\begin{split} 
\frac{1}{K} \sum\limits_{k=1}^{K} \|\Bar{\v}_t - \v_t^k\|^2 \leq 2\eta^2 I \sigma^2 + 2\eta^2 
\left(I \sum\limits_{\tau=t_0}^{t-1} \left[ \frac{54\ell^2}{K} \sum\limits_{k=1}^{K} \left[ \|\v^k_{\tau} - \Bar{\v}_{\tau}\|^2 + \|\alpha_{k, \tau}-\Bar{\alpha}_{\tau} \|^2  \right] + 3D^2 \right] \right)   
\end{split}   
\end{equation}
\end{small}

Summing over $t=\{0, ..., T-1\}$,
\begin{small}
\begin{equation}
\begin{split} 
\sum\limits_{t=0}^{T-1} \frac{1}{K} \sum\limits_{k=1}^{K} \|\Bar{\v}_t - \v_t^k\|^2 \! \leq \!  2\eta^2 I\sigma^2 T \!+ \! 108\eta^2 I^2 \ell^2 \sum\limits_{t=0}^{T-1} \frac{1}{K} 
\left(  \|\v^k_{t} - \Bar{\v}_{t}\|^2 + \|\alpha^k_t -\Bar{\alpha}_{\tau} \|^2 \right) 
+ 6\eta^2 I^2 D^2 T.
\end{split}   
\end{equation}
\end{small}

Similarly for $\alpha$ side, we have
\begin{small}
\begin{equation}
\sum\limits_{t=0}^{T-1} \frac{1}{K} \sum\limits_{k=1}^{K} \|\bar{\alpha}_t-\alpha_t^k\|^2 
\leq 2\eta^2 I \sigma^2 T + 108\eta^2 I^2 \ell^2 \sum\limits_{t=0}^{T-1}  \frac{1}{K}\left(\|\v_t^k-\bar{\v}_t\|^2 + \|\alpha_t^k - \bar{\alpha}_t\|^2\right) + 6\eta^2 I^2 D^2 T.
\end{equation} 
\end{small} 

Summing up the above two inequalities,
\begin{equation}
\begin{split}
\sum\limits_{t=0}^{T-1}\frac{1}{K} \sum\limits_{k=1}^{K} [\|\Bar{\v}_t - \v_t^k\|^2 + \E[\|\Bar{\alpha}_t - \alpha_t^k\|^2] &  \leq 
\frac{4\eta^2 I \sigma^2}{1 - 216 \eta^2 I^2 \ell^2} T 
+ \frac{12\eta^2 I^2 D^2}{1 - 216 \eta^2 I^2 \ell^2} T  
\\ 
&\leq  12\eta^2 I \sigma^2 T + 36 \eta^2 I^2 D^2 T,
\end{split}   
\end{equation}  
where the second inequality is due to $I \leq \frac{1}{18\sqrt{2} \eta \ell}$, i.e., $1-216\eta^2 I^2 \ell^2 \geq \frac{2}{3}$.
\end{proof}

With the above lemmas, we are ready to give the convergence of duality gap in one stage of CODA+. 

\subsection{Proof of Lemma \ref{lem:one_stage_codaplus}}
\begin{proof}
Note $\E\langle \frac{1}{K}\sum\limits_{k=1}^{K}[\nabla_{\v}f_k(\v_{t-1}^k, \alpha_{t-1}^k) - \nabla_{\v}F_k(\v_{t-1}^k, \alpha_{t-1}^k; z_{t-1}^k)], \hat{\v}_t - \Tilde{\v}_{t-1} \rangle = 0 $ and \\
$\E\left\langle -\frac{1}{K}\sum\limits_{k=1}^{K}[\nabla_{\alpha} f_k(\v_{t-1}^k, \alpha_{t-1}^k) - F_k(\v_{t-1}^k, \alpha_{t-1}^k; z_{t-1}^k)], \tilde{\alpha}_{t-1} - \hat{\alpha}_t \right\rangle = 0$. 
And then
plugging Lemma \ref{lem:A1_codaplus}  and Lemma \ref{lem:A2_codaplus}  into Lemma \ref{lem:codaplus_one_state_split}, and taking expectation, we get 
\begin{small}
\begin{equation}
\begin{split}  
&\E[f^s(\bar{\v}, \alpha) - f^s(\v, \bar{\alpha})] \\
& \leq \frac{1}{T}\sum\limits_{t=1}^{T} \E\Bigg[ \underbrace{ \left(\frac{3\ell+3\ell^2/\mu_2}{2} - \frac{1}{2\eta}\right) \|\bar{\v}_{t-1} - \bar{\v}_t\|^2 +  \left(2\ell - \frac{1}{2\eta}\right)\|\bar{\alpha}_t -  \bar{\alpha}_{t-1}\|^2}_{C_1}\\ 
&~~~+\underbrace{\left(\frac{1}{2\eta} - \frac{\mu_2}{3} \right) \|\bar{\alpha}_{t-1} - \alpha\|^2 - \left(\frac{1}{2\eta} - \frac{\mu_2}{3}\right)(\bar{\alpha}_t-\alpha)^2}_{C_2} \\  
&~~~ + \underbrace{ \left( \frac{1}{2\eta}-\frac{\ell}{3} \right) \|\bar{\v}_{t-1} - \v\|^2 -  \left(\frac{1}{2\eta} - \frac{\ell}{3} \right)\|\bar{\v}_t - \v\|^2}_{C_3}\\ 
&~~~ +\underbrace{\frac{1}{2\eta}((\alpha - \Tilde{\alpha}_{t-1})^2 -  (\alpha-\Tilde{\alpha}_t)^2)}_{C_4}
 +\underbrace{\frac{1}{2\eta}(\|\v - \Tilde{\v}_{t-1}\|^2 -  \|\v-\Tilde{\v}_t\|^2)}_{C_5} \\
&~~~ +\underbrace{\left(\frac{3\ell^2}{2\mu_2} + \frac{3\ell}{2}\right)\frac{1}{K}\sum\limits_{k=1}^{K}\|\bar{\v}_{t-1}-\v_{t-1}^k\|^2 
+ \left(\frac{3 \ell}{2} + \frac{3 \ell^2}{2\mu_2}\right)\frac{1}{K}\sum\limits_{k=1}^{K} (\bar{\alpha}_{t-1} - \alpha_{t-1}^k)^2}_{C_6}\\
&~~~ +\underbrace{\frac{3\eta}{2} \left\|\frac{1}{K}\sum\limits_{k=1}^{K} [\nabla_{\v} f_k(\v_{t-1}^k, \alpha_{t-1}^k) - \nabla_{\v}F^s_k(\v_{t-1}^k, \alpha_{t-1}^k; z_{t-1}^k)]\right\|^2}_{C_7} \\
&~~~ +  \underbrace{\frac{3\eta}{2} \left\|  \frac{1}{K}\sum\limits_{k=1}^{K} \nabla_{\alpha}f^s_k(\v_{t-1}^k, \alpha_{t-1}^k) - \nabla_{\alpha} F^s_k(\v_{t-1}^k, \alpha_{t-1}^k; z_{t-1}^k) \right\|^2}_{C_8} \Bigg]. 
\end{split}
\label{equ:codaplus:before_summation}
\end{equation}
\end{small}

Since $\eta\leq \min(\frac{1}{3\ell + 3\ell^2/\mu_2}, \frac{1}{4\ell})$,
thus in the RHS of (\ref{equ:codaplus:before_summation}), $C_1$ can be cancelled.
$C_2$, $C_3$, $C_4$ and $C_5$ will be handled by telescoping sum. 
$C_6$ can be bounded by Lemma \ref{lem:diverge_codaplus}. 



Taking expectation over $C_7$,
\begin{small}
\begin{equation}
\begin{split}
&\E\left[\frac{3\eta}{2} \left\|\frac{1}{K}\sum\limits_{k=1}^{K}[\nabla_{\v} f^s_k(\v_{t-1}^k, \alpha_{t-1}^k) - \nabla_{\v}F^s_k(\v_{t-1}^k, \alpha_{t-1}^k; z_{t-1}^k)]\right\|^2\right]\\
&=\E\left[\frac{3\eta}{2K^2} \left\|\sum\limits_{k=1}^{K}[\nabla_{\v} f^s_k(\v_{t-1}^k, \alpha_{t-1}^k) - \nabla_{\v}F_k(\v_{t-1}^k, \alpha_{t-1}^k; z_{t-1}^k)]\right\|^2\right]\\
&=\E\left[\frac{3\eta}{2K^2}\left(\sum\limits_{k=1}^K \|\nabla_{\v} f^s_k(\v_{t-1}^k, \alpha_{t-1}^k) - \nabla_{\v}F^s_k(\v_{t-1}^k, \alpha_{t-1}^k; z_{t-1}^k)\|^2\right.\right.\\
&~~~~~~\left.\left.+  2\sum\limits_{k=1}^{K}\sum\limits_{j=i+1}^{K} \left\langle \nabla_{\v} f^s_k(\v_{t-1}^k, \alpha_{t-1}^k) 
- \nabla_{\v} F^s_k(\v_{t-1}^{k}, \alpha_{t-1}^{k}; z_{t-1}^k),  
\nabla_{\v} f_j(\v_{t-1}^j, \alpha_{t-1}^j) 
- \nabla_{\v} F^s_j(\v_{t-1}^{j}, \alpha_{t-1}^{j}; z_{t-1}^j)
\right\rangle \right)  \right]\\ 
&\leq \frac{3\eta \sigma^2}{2K}. 
\end{split}  
\label{local_variance_v}
\end{equation}
\end{small}
The last inequality holds because $\|\nabla_{\v}f_k(\v_{t-1}^k,\alpha_{t-1}^k) - \nabla_{\v}F_k (\v_{t-1}^k, \alpha_{t-1}^k;z_{t-1}^k)\|^2 \leq \sigma^2$
 and $\E\langle \nabla_{\v} f_k(\v_{t-1}^k, \alpha_{t-1}^k) 
\!-\! \nabla_{\v} F_k(\v_{t-1}^{k}, \alpha_{t-1}^{k}; z_{t-1}^k), 
\nabla_{\v} f_j(\v_{t-1}^j, \alpha_{t-1}^j) 
\!-\! \nabla_{\v} F_j(\v_{t-1}^{j}, \alpha_{t-1}^{j}; z_{t-1}^j)
\rangle = 0$ for any $k \neq j$ as each machine draws data independently.
Similarly, we take expectation over $C_8$ and have
\begin{small} 
\begin{equation}
\begin{split} 
&\E\left[\frac{3\eta}{2}  \left\|\frac{1}{K}\sum\limits_{k=1}^{K}[\nabla_{\alpha} f_k(\v_{t-1}^k, \alpha_{t-1}^k) - \nabla_{\alpha}F_k(\v_{t-1}^k, \alpha_{t-1}^k; \z_{t-1}^k)]\right\|^2\right]
\leq \frac{3\eta \sigma^2}{2K}. 
\end{split}  
\label{local_variance_alpha}
\end{equation}
\end{small}


Plugging (\ref{local_variance_v}) and  (\ref{local_variance_alpha})  into (\ref{before_summation}), and taking expectation, it yields
\begin{equation*}
\begin{split}
&\E[f^s(\bar{\v}, \alpha)  - f^s(\v, \bar{\alpha}) \\
&\leq \E\bigg\{\frac{1}{T}\left( \frac{1}{2\eta}-\frac{\ell}{3}\right) \|\bar{\v}_0-\v\|^2 + \frac{1}{2\eta T}\|\tilde{\v}_0-\v\|^2 +  \frac{1}{T}\left(\frac{1}{2\eta} - \frac{\mu_2}{3} \right)\|\bar{\alpha}_0 - \alpha\|^2 
+ \frac{1}{2\eta T} \|\tilde{\alpha}_0 - \alpha\|^2 \\ 
&~~~~~+ \frac{1}{T}\sum\limits_{t=1}^{T}\left(\frac{3\ell^2}{2\mu_2} + \frac{3\ell}{2}\right)\frac{1}{K}\sum\limits_{k=1}^{K}\|\bar{\v}_{t-1} - \v_{t-1}^k\|^2 + \frac{1}{T}\sum\limits_{t=1}^{T}\left(\frac{3\ell}{2} + \frac{3\ell^2}{2\mu_2}\right)\frac{1}{K}\sum\limits_{k=1}^{K}(\bar{\alpha}_{t-1} - \alpha_{t-1}^k)^2\\
&~~~~
+\frac{1}{T} \sum\limits_{t=1}^{T}\frac{3\eta \sigma^2}{K}\bigg\}\\ 
&\leq \frac{1}{\eta T} \|\v_0 - \v\|^2 + \frac{1}{\eta T} \|\alpha_0 - \alpha\|^2 +
\left(\frac{3\ell^2}{2\mu_2} + \frac{3\ell}{2}\right)(12\eta^2 I \sigma^2 + 36 \eta^2 I^2 D^2 )\I_{I>1}  + \frac{3\eta\sigma^2}{K},
\end{split} 
\end{equation*} 
where we use Lemma \ref{lem:diverge_codaplus}, $\v_0 = \bar{\v}_0$, and $\alpha_0 =  \bar{\alpha}_0$ in the last inequality. 
\end{proof}

\subsection{Main Proof of Theorem \ref{thm:coda_plus}} 
\begin{proof}

Since $f(\v, \alpha)$ is $\ell$-smooth (thus $\ell$-weakly convex) in $\v$ for any $\alpha$, $\phi(\v) = \max\limits_{\alpha'} f(\v, \alpha')$ is also $\ell$-weakly convex. 
Taking $\gamma = 2\ell$, we have
\begin{align}  
\begin{split} 
\phi(\v_{s-1}) &\geq \phi(\v_s) + \langle \partial \phi(\v_s), \v_{s-1} - \v_s\rangle - \frac{\ell}{2} \|\v_{s-1} - \v_s\|^2 \\ 
& = \phi(\v_s) + \langle \partial \phi(\v_s) + 2 \ell (\v_s - \v_{s-1}), \v_{s-1} - \v_s\rangle + \frac{3\ell}{2} \|\v_{s-1} - \v_s\|^2 \\ 
& \overset{(a)}{=} \phi(\v_s) + \langle \partial \phi_s(\v_s), \v_{s-1} - \v_s \rangle + \frac{3\ell}{2} \|\v_{s-1} - \v_s\|^2 \\ 
& \overset{(b)}{=}  \phi(\v_s) - \frac{1}{2\ell} \langle \partial \phi_s(\v_s), \partial \phi_s(\v_s) - \partial \phi(\v_s) \rangle + \frac{3}{8\ell} \|\partial \phi_s(\v_s) - \partial \phi(\v_s)\|^2 \\ 
& = \phi(\v_s) - \frac{1}{8\ell} \|\partial \phi_s(\v_s)\|^2
- \frac{1}{4\ell} \langle \partial \phi_s(\v_s), \partial \phi(\v_s)\rangle + \frac{3}{8\ell} \|\partial \phi(\v_s)\|^2,
\end{split} 
\label{local:P_weakly} 
\end{align} 
where $(a)$ and $(b)$ hold by the definition of $\phi_s(\v)$.

Rearranging the terms in (\ref{local:P_weakly}) yields
\begin{align}
\begin{split}
\phi(\v_s) - \phi(\v_{s-1}) &\leq \frac{1}{8\ell} \|\partial \phi_s(\v_s)\|^2 + \frac{1}{4\ell}\langle \partial \phi_s(\v_s), \partial \phi(\v_s)\rangle - \frac{3}{8\ell} \|\partial \phi(\v_s)\|^2 \\
&\overset{(a)}{\leq} \frac{1}{8\ell} \|\partial \phi_s(\v_s)\|^2
+ \frac{1}{8\ell} (\|\partial \phi_s(\v_s)\|^2 + \|\partial \phi(\v_s)\|^2) - \frac{3}{8\ell} \|\phi(\v_s)\|^2 \\
& = \frac{1}{4\ell}\|\partial \phi_s(\v_s)\|^2 - \frac{1}{4\ell}\|\partial \phi(\v_s)\|^2\\
& \overset{(b)}{\leq} \frac{1}{4\ell} \|\partial \phi_s(\v_s)\|^2 - \frac{\mu}{2\ell}(\phi(\v_s) - \phi(\v_*)) 
\end{split} 
\end{align} 
where $(a)$ holds by using $\langle \mathbf{a}, \mathbf{b}\rangle \leq \frac{1}{2}(\|\mathbf{a}\|^2 +  \|\mathbf{b}\|^2)$, and $(b)$ holds by the $\mu$-PL property of $\phi(\v)$.

Thus, we have 
\begin{align}
\left(4\ell+2\mu\right) (\phi(\v_s) - \phi(\v_*)) - 4\ell (\phi(\v_{s-1}) - \phi(\v_*)) \leq \|\partial \phi_s(\v_s)\|^2. 
\label{local:nemi_thm_partial_P_s_2} 
\end{align} 

Since $\gamma = 2\ell$, $f^s(\v, \alpha)$ is $\ell$-strongly convex in $\v$ and $\mu_2=2p(1-p)$ strong concave in $\alpha$.
Apply Lemma \ref{lem:Yan1} to $f^s$, we know that 
\begin{align}  
\frac{\ell}{4} \|\hat{\v}_s(\alpha_s) - \v_0^s\|^2 + \frac{\mu_2}{4} \|\hat{\alpha}_s(\v_s) - \alpha_0^s\|^2 \leq \text{Gap}_s(\v_0^s, \alpha_0^s) + \text{Gap}_s(\v_s, \alpha_s). 
\end{align}

By the setting of $\eta_s = \eta_0  \exp\left(-(s-1)\frac{2\mu}{c+2\mu}\right)$, and 
$T_s = \frac{212}{\eta_0 \min\{\ell, \mu_2\}} \exp  \left((s-1)\frac{2\mu}{c+2\mu}\right)$, we note that $\frac{1}{\eta_s T_s} \leq  \frac{\min\{\ell,\mu_2\}}{212}$. 
Set $I_s$ such that $\left(\frac{3\ell^2}{2\mu_2} + \frac{3\ell}{2}\right) (12 \eta_s^2 I_s + 36 \eta^2 I_s^2 D^2) \leq \frac{\eta_s \sigma^2}{K}$, where the specific choice of $I_s$ will be made later.
Applying Lemma \ref{lem:one_stage_codaplus} with  $\hat{\v}_s(\alpha_s) = \arg\min\limits_{\v'} f^s(\v', \alpha_s)$ and $\hat{\alpha}_s(\v_s) = \arg\max\limits_{\alpha'} f^s(\v_s, \alpha')$, we have 
\begin{align}  
\begin{split} 
&\E[\text{Gap}_s(\v_s, \alpha_s)] 
\leq \frac{4\eta_s\sigma^2}{K} 
+ \frac{1}{53} \E\left[\frac{\ell}{4}\|\hat{\v}_s(\alpha_s) - \v_0^s\|^2 + \frac{\mu_2}{4}\|\hat{\alpha}_s(\v_s) - \alpha_0^s\|^2 \right]
\\  
& \leq \frac{4\eta_s\sigma^2}{K} + \frac{1}{53} \E\left[\text{Gap}_s(\v_0^s, \alpha_0^s) + \text{Gap}_s(\v_s, \alpha_s)\right].  
\end{split} 
\end{align} 

Since $\phi(\v)$ is $L$-smooth and $\gamma = 2\ell$, then $\phi_s (\v)$ is $\hat{L} = (L+2\ell)$-smooth. 
According to Theorem 2.1.5 of  \citep{DBLP:books/sp/Nesterov04}, we have 
\begin{align}
\begin{split} 
& \E[\|\partial \phi_s(\v_s)\|^2] \leq 2\hat{L}\E(\phi_s(\v_s) - \min\limits_{x\in \mathbb{R}^{d}} \phi_s(\v)) \leq 2\hat{L}\E[\text{Gap}_s(\v_s, \alpha_s)] \\
& = 2\hat{L}\E[4\text{Gap}_s(\v_s, \alpha_s) - 3\text{Gap}_s(\v_s, \alpha_s)] \\ 
&\leq 2\hat{L} \E \left[4\left(\frac{4\eta_s\sigma^2}{K} +  \frac{1}{53}\left(\text{Gap}_s(\v_0^s, \alpha_0^s) + \text{Gap}_s(\v_s, \alpha_s)\right)\right) - 3\text{Gap}_s(\v_s,\alpha_s)\right] \\
& = 2\hat{L} \E \left[\frac{16 \eta_s \sigma^2}{K} +   \frac{4}{53}\text{Gap}_s(\v_0^s, \alpha_0^s) - \frac{155}{53}\text{Gap}_s(\v_s, \alpha_s)\right]
\end{split}  
\label{local:nemi_thm_P_smooth_1}
\end{align}

Applying Lemma \ref{lem:Yan5} to  (\ref{local:nemi_thm_P_smooth_1}), we have
\begin{align} 
\begin{split} 
& \E[\|\partial \phi_s(\v_s)\|^2] \leq 2\hat{L} \E \bigg[\frac{16\eta_s \sigma^2}{K}  +  \frac{4}{53}\text{Gap}_s(\v_0^s, \alpha_0^s) \\  
&~~~~~~~~~~~~~~~~~~~~~~~~~~~~~~~~~~~~~~~~  
- \frac{155}{53} \left(\frac{3}{50} \text{Gap}_{s+1}(\v_0^{s+1}, \alpha_0^{s+1}) + \frac{4}{5} (\phi(\v_0^{s+1}) - \phi(\v_0^s))\right) \bigg] \\
& = 2\hat{L}\E \bigg[\frac{16 \eta_s \sigma^2}{K} +  \frac{4}{53}\text{Gap}_s(\v_0^s, \alpha_0^s) \!-\! \frac{93}{530}\text{Gap}_{s+1}(\v_0^{s+1}, \alpha_0^{s+1}) \!-\! 
\frac{124}{53} (\phi(\v_0^{s+1}) - \phi(\v_0^s)) \bigg]. 
\end{split} 
\end{align}

Combining this with  (\ref{local:nemi_thm_partial_P_s_1}), rearranging the terms, and defining a constant $c = 4\ell + \frac{248}{53}\hat{L} \in O(L+\ell)$, we get
\begin{align} 
\begin{split}
&\left(c + 2\mu\right)\E [\phi(\v_0^{s+1}) - \phi(\v_*)] + \frac{93}{265}\hat{L} \E[\text{Gap}_{s+1}(\v_0^{s+1}, \alpha_0^{s+1})] \\ 
&\leq \left(4\ell + \frac{248}{53} \hat{L}\right) \E[\phi(\v_0^s) - \phi(\v_*)] 
+ \frac{8\hat{L}}{53} \E[\text{Gap}_s(\v_0^s, \alpha_0^s)] 
+ \frac{32 \eta_s \hat{L}\sigma^2}{K}  \\ 
& \leq c \E\left[\phi(\v_0^s) - \phi(\v_*) + \frac{8\hat{L}}{53c} \text{Gap}_s(\v_0^s, \alpha_0^s)\right] + \frac{32 \eta_s \hat{L} \sigma^2}{K} 
\end{split}  
\end{align} 

Using the fact that $\hat{L} \geq \mu$,
\begin{align}
\begin{split}
(c+2\mu) \frac{8\hat{L}}{53c} = \left(4\ell + \frac{248}{53}\hat{L} + 2\mu\right)\frac{8\hat{L}}{53(4\ell + \frac{248}{53}\hat{L})} \leq \frac{8\hat{L}}{53} + \frac{16\mu \hat{L}}{248\hat{L}} \leq \frac{93}{265} \hat{L}. 
\end{split}
\end{align}

Then, we have
\begin{align}
\begin{split} 
&(c+2\mu)\E \left[\phi(\v_0^{s+1}) - \phi(\v_*) + \frac{8\hat{L}}{53c}\text{Gap}_{s+1}(\v_0^{s+1}, \alpha_0^{s+1})\right] \\ 
&\leq c \E \left[\phi(\v_0^s) - \phi(\v_*) 
+  \frac{8\hat{L}}{53c}\text{Gap}_{s}(\v_0^{s},  \alpha_0^s)\right] 
+ \frac{32 \eta_s \hat{L}\sigma^2}{K}. 
\end{split}
\end{align}

Defining $\Delta_s = \phi(\v_0^s) - \phi(\v_*) +  \frac{8\hat{L}}{53c}\text{Gap}_s(\v_0^s, \alpha_0^s)$, then
\begin{align} 
\begin{split}
&\E[\Delta_{s+1}] \leq \frac{c}{c+2\mu} \E[\Delta_s] +  \frac{32 \eta_s \hat{L}\sigma^2}{(c+2\mu)K} 
\end{split}
\label{equ:codaplus:recursiveDelta}
\end{align}

Using this inequality recursively, it yields
\begin{align} 
\begin{split}
& E[\Delta_{S+1}] \leq \left(\frac{c}{c+2\mu}\right)^S E[\Delta_1]
+ \frac{32 \hat{L} \sigma^2}{(c+2\mu)K}  \sum\limits_{s=1}^{S} \left(\eta_s \left(\frac{c}{c+2\mu}\right)^{S+1-s} \right). 
\end{split}
\end{align}

By definition,
\begin{align}
\begin{split}
\Delta_1 &= \phi(\v_0^1) - \phi(\v^*) + \frac{8\hat{L}}{53c}\widehat{Gap}_1(\v_0^1, \alpha_0^1) \\
& = \phi(\v_0) - \phi(\v^*) + \left( f(\v_0, \hat{\alpha}_1(\v_0)) +  \frac{\gamma}{2}\|\v_0 - \v_0\|^2 -  f(\hat{\v}_1(\alpha_0), \alpha_0) - \frac{\gamma}{2}\|\hat{\v}_1(\alpha_0) - \v_0\|^2 \right) \\ 
& \leq \epsilon_0 + f(\v_0, \hat{\alpha}_1(\v_0)) - f(\hat{\v}(\alpha_0), \alpha_0) \leq 2\epsilon_0.
\end{split} 
\end{align}
Using inequality $1-x \leq \exp(-x)$,
we have
\begin{align*}
\begin{split} 
&\E[\Delta_{S+1}] \leq \exp\left(\frac{-2\mu S}{c+2\mu}\right)\E[\Delta_1] + \frac{32 \eta_0 \hat{L} \sigma^2} {(c+2\mu)K} 
\sum\limits_{s=1}^{S}\exp\left(-\frac{2\mu S}{c+2\mu}\right) \\ 
&\leq 2\epsilon_0 \exp\left(\frac{-2\mu S}{c+2\mu}\right)
+ \frac{32 \eta_0 \hat{L} \sigma^2 }{(c+2\mu)K}  
S\exp\left(-\frac{2\mu S}{(c+2\mu)}\right).  
\end{split}  
\end{align*}  

To make this less than $\epsilon$, it suffices to make
\begin{align}
\begin{split} 
& 2\epsilon_0 \exp\left(\frac{-2\mu S}{c+2\mu}\right) \leq \frac{\epsilon}{2}, \\
& \frac{32 \eta_0 \hat{L} \sigma^2 }{(c+2\mu)K} 
S\exp\left(-\frac{2\mu S}{c+2\mu}\right) \leq \frac{\epsilon}{2}. 
\end{split} 
\end{align} 

Let $S$ be the smallest value such that $\exp\left(\frac{-2\mu S}{c+2\mu}\right) \leq \min \{ \frac{\epsilon}{4\epsilon_0}, 
\frac{(c+2\mu)K\epsilon }{64 \eta_0 \hat{L} S \sigma^2}\}$.  
We can set $S = \max\bigg\{\frac{c+2\mu}{2\mu}\log \frac{4\epsilon_0}{\epsilon}, 
\frac{c+2\mu}{2\mu}\log \frac{64 \eta_0 \hat{L} S \sigma^2} {(c+2\mu)K \epsilon} \bigg\}$. 

Then, the total iteration complexity is 
\begin{equation} 
\label{equ:codaplus:iter_comp}
\begin{split}
\sum\limits_{s=1}^{S}T_s &\leq O\left( \frac{424}{\eta_0 \min\{\ell,\mu_2\}}
\sum\limits_{s=1}^{S}\exp\left((s-1)\frac{2\mu}{c+2\mu}\right) \right)\\    
& \leq O\bigg(\frac{1}{\eta_0\min\{\ell,\mu_2\}}  \frac{\exp(S\frac{2\mu}{c+2\mu}) -  1}{\exp(\frac{2\mu}{c+2\mu})-1}  \bigg)\\ 
& \overset{(a)}{\leq} \widetilde{O}   \left(\frac{c}{\eta_0\mu\min\{\ell,\mu_2\}}  \max\left\{\frac{\epsilon_0}{\epsilon}, 
\frac{\eta_0 \hat{L} S \sigma^2} 
{(c+2\mu)K\epsilon} \right\}\right) \\
& \leq  \widetilde{O}\left(\max\left\{\frac{(L+\ell) \epsilon_0}{\eta_0 \mu\min\{\ell, \mu_2\} \epsilon},
\frac{(L + \ell)^2 \sigma^2} {\mu^2 \min\{\ell,\mu_2\} K \epsilon}\right\} \right) \\
&\leq  \widetilde{O}\left(\max\left\{\frac{1}{\mu_1 \mu_2^2\epsilon}, 
\frac{1} {\mu_1^2 \mu_2^3 K \epsilon}\right\} \right),
\end{split}
\end{equation} 
where $(a)$ uses the setting of $S$ and $\exp(x) - 1\geq x$, and $\widetilde{O}$ suppresses logarithmic factors. 

$\eta_s = \eta_0 \exp(-(s-1)\frac{2\mu}{c+2\mu}), T_s = \frac{212}{\eta_0 \mu_2} \exp\left((s-1)\frac{2\mu}{c+2\mu}\right)$.

\textbf{Next, we will analyze the communication cost}. We investigate both $D= 0 $ and $D>0$ cases.

\textbf{(i) Homogeneous Data (D = 0)}: To assure
$\left(\frac{3\ell^2}{2\mu_2} + \frac{3\ell}{2}\right) (12 \eta_s^2 I_s + 36 \eta^2 I_s^2 D^2) \leq \frac{\eta_s \sigma^2}{K}$ which we used in above proof, we take $I_s = \frac{1}{M K\eta_s}=\frac{\exp((s-1)\frac{2\mu}{c+2\mu})}{M K\eta_0}$,  where $M$ is a proper constant.

If $\frac{1}{MK\eta_0} > 1$, then $I_s=\max(1, \frac{\exp((s-1)\frac{2\mu}{c+2\mu})}{MK\eta_0}) = \frac{\exp((s-1)\frac{2\mu}{c+2\mu})}{MK\eta_0}$.

Otherwise, $\frac{1}{MK\eta_0} \leq 1$, then $K_s = 1$ for $s\leq S_1 := \frac{c+2\mu}{2\mu}\log(MK \eta_0) + 1$ and $K_s = \frac{\exp((s-1)\frac{2\mu}{c+2\mu})}{MK\eta_0}$ for $s>S_1$.

\begin{equation}
\begin{split}
&\sum\limits_{s=1}^{S_1} T_s = 
\sum\limits_{s=1}^{S_1} O\left(\frac{212}{\eta_0} \exp\left((s-1)\frac{2\mu}{c+2\mu} \right) \right)
\\ 
& = \widetilde{O} \left(\frac{212}{\eta_0} \frac{\exp\left(\frac{2\mu}{c+2\mu} S_1   \right)-1}{\exp\left(\exp(\frac{2\mu}{c+2\mu}) - 1\right)} \right)
= \widetilde{O}\left( \frac{K}{\mu}  \right)
\end{split}   
\end{equation}  
Thus, for both above cases, the total communication complexity can be bounded by 
\begin{equation}
\begin{split}
&\sum\limits_{s=1}^{S_1} T_s + \sum\limits_{s=S_1+1}^{S} \frac{T_s}{I_s} = \widetilde{O}\left( \frac{K}{\mu}
+ K S \right) 
\leq \widetilde{O}\left(\frac{K}{\mu}\right). 
\end{split} 
\end{equation} 

\textbf{(ii) Heterogeneous Data ($D > 0$)}:

To assure
$\left(\frac{3\ell^2}{2\mu_2} + \frac{3\ell}{2}\right) (12 \eta_s^2 I_s + 36 \eta^2 I_s^2 D^2) \leq \frac{\eta_s \sigma^2}{K}$ which we used in above proof, we take $I_s = \frac{1}{M \sqrt{K \eta_s}}$, where $M$ is proper constant. 

If $\frac{1}{M \sqrt{N\eta_0}} \leq 1$, then $I_s=1$ for $s\leq S_2 := \frac{c+2\mu}{2\mu}\log(M^2 K \eta_0) + 1$ and $I_s = \frac{\exp((s-1)\frac{2\mu}{c+2\mu})}{N\eta_0}$ for $s>S_2$.

\begin{equation}
\begin{split}
&\sum\limits_{s=1}^{S_2} T_s = 
\sum\limits_{s=1}^{S_2} O\left(\frac{212}{\eta_0} \exp\left((s-1)\frac{2\mu}{c+2\mu}\right) \right) 
 = \widetilde{O} \left( \frac{K}{\mu} \right). 
\end{split}   
\end{equation}  

Thus, the communication complexity can be bounded by
\begin{equation} 
\begin{split}
&\sum\limits_{s=1}^{S_2} T_s + \sum\limits_{s=S_2+1}^{S} \frac{T_s}{I_s} = \widetilde{O}\left(  \frac{K}{\mu}
+ \sqrt{K} \exp\left(\frac{(s-1)\frac{2\mu}{c+2\mu}}{2}\right) \right) \\
&\leq \widetilde{O}(\frac{K}{\mu} + \sqrt{K}  \frac{\exp\left(\frac{S}{2}\frac{2\mu}{c+2\mu}\right) - 1}{\exp{\frac{\mu}{c+2\mu}} - 1})  
\leq O\left(\frac{K}{\mu} +  \frac{1}{\mu^{3/2}  \epsilon^{1/2}} \right) .
\end{split} 
\end{equation} 
\end{proof}

\section{Baseline: Naive Parallel Algorithm} 
Note that if we set $I_s = 1$ for all $s$, CODA+ will be  reduced to
a naive parallel version of PPD-SG \cite{liu2019stochastic}. 
We analyze this naive parallel algorithm in the following theorem.

\begin{thm}
\label{thm:npa} 
Consider Algorithm \ref{alg:codaplus_outer} with $I_s = 1$. Set $\gamma = 2\ell$, $\hat{L}=L+2\ell$, $c  =\frac{\mu/\hat{L}}{5+\mu/\hat{L}}$.


(1) If $M<\frac{1}{K\mu\epsilon}$, 
set $\eta_s = \eta_0 \exp(-(s-1)c)\leq O(1)$ and 
$T_s = \frac{212}{\eta_0 \min(\ell, \mu_2)} \exp( (s-1)c)$, then
the communication/iteration complexity is  $\widetilde{O}\bigg(\max\left(\frac{\Delta_0}{\mu \epsilon \eta_0 K}, \frac{\hat{L}}{\mu^2 K\epsilon}\right)\bigg)$ to return $\v_S$ such that $\E[\phi(\v_S) - \phi(\v^*_{\phi})] \leq \epsilon$.
 
(2) If $M\geq \frac{1}{K\mu\epsilon}$, set $\eta_s = \min(\frac{1}{3\ell+3\ell^2/\mu_2}, \frac{1}{4\ell})$ and $T_s = \frac{212}{\eta_s \min\{\ell, \mu_2\}}$, then the communication/iteration complexity is  $\widetilde{O}\bigg(\frac{1}{\mu}\bigg)$  to return $\v_S$ such that $\E[\phi(\v_S) - \phi(\v^*_{\phi})] \leq \epsilon$.
\end{thm}

\begin{proof}
(1) If $M<\frac{1}{K\mu\epsilon}$,
note that the setting of $\eta_s$ and $T_s$ are identical to that in CODA+ (Theorem \ref{thm:coda_plus}).
However, as a batch of $M$ is used on each machine at each iteration, the variance at each iteration is reduced to $\frac{\sigma^2}{KM}$. 
Therefore, by similar analysis of Theorem \ref{alg:codaplus_outer} (specifically (\ref{equ:codaplus:iter_comp})), we see that the iteration complexity of NPA is $\widetilde{O}\left( \frac{1}{\mu \epsilon} + \frac{1}{\mu^2KM\epsilon}\right)$.
Thus, the sample complexity of each machines is $\widetilde{O}\left( \frac{M}{\mu\epsilon} + \frac{1}{\mu^2 K \epsilon}\right)$.

(2) If $M\geq\frac{1}{K\mu\epsilon}$, . 
Note $\frac{1}{\eta_s T_s} \leq \frac{\min\{\ell,\mu_2\}}{212}$, we can follow the proof of Theorem \ref{thm:coda_plus} and derive 
\begin{equation}
\begin{split}
\Delta_{s+1} & \leq \frac{c}{c+2\mu} \E[\Delta_s] + \frac{32\eta_s \hat{L} \sigma^2}{KM}  \leq  \frac{c}{c+2\mu} \E[\Delta_s] + 32\eta_s \hat{L} \sigma^2 \mu \epsilon,
\end{split} 
\end{equation} 
where the first inequality  is similar to (\ref{equ:codaplus:recursiveDelta}) and  the $\Delta$ is defined as that in Theorem \ref{thm:coda_plus}. 
Thus,
\begin{equation}
\begin{split}
\Delta_{S+1} &\leq \left(\frac{c}{c+2\mu}\right)^{S} + \mu \epsilon O\left(\sum\limits_{s=1}^S  \left(\frac{c}{c+2\mu}\right)^{s-1} \right) \\
& \leq  \left(\frac{c}{c+2\mu}\right)^{S} 
+ O(\epsilon) \leq \exp\left(\frac{-2\mu S}{c+2\mu}\right)
+ O(\epsilon).
\end{split}   
\end{equation} 
Therefore, it suffices to take $S=\widetilde{O}\left(\frac{1}{\mu}\right)$.
Hence, the total number of communication is $S\cdot T_s = \widetilde{O}\left( \frac{1}{\mu}\right)$ and the sample complexity on each machine is $\widetilde{O}\left( \frac{M}{\mu} \right)$. 

\end{proof}

\section{Proof of Lemma \ref{lem:codasca:one_stage}}   

In this section, we will prove Lemma \ref{lem:codasca:one_stage}, which is the convergence analysis of one stage in CODASCA.

First, the duality gap in stage $s$ can be bounded as
\begin{lem}
\label{lem:codasca:one_stage_split}
For any $\v, \alpha$,
\begin{equation*} 
\begin{split} 
&\frac{1}{R} \sum\limits_{r=1}^{R} 
[f^s(\v_r, \alpha)  - f^s(\v, \alpha_r) ] \\
& \leq \frac{1}{R} \sum\limits_{r=1}^{R}\bigg[ \underbrace{\langle  
\partial_{\v}f^s(\v_{r-1}, \alpha_{r-1}), \v_r - \v \rangle}_{B4} 
+ \underbrace{\langle \partial_{\alpha} f^s(\v_{r-1}, \alpha_{r-1}),  \alpha - \alpha_r \rangle}_{B5}  \\ 
&~~~~~~~~~~ 
+\frac{3\ell + 3\ell^2/\mu_{2}}{2}  \|\v_r-\v_{r-1}\|^2 
 + 2\ell (\alpha_r - \alpha_{r-1})^2 
 -\frac{\ell}{3}\|\v_{r-1}-\v\|^2   
- \frac{\mu_2}{3}(\alpha_{r-1}-\alpha)^2  \bigg] 
\end{split} 
\end{equation*}      
\end{lem}  

\begin{proof}
By $\ell$-strongly convexity of $f^s(\v, \alpha)$ in $\v$, we have
\begin{equation} 
f^s(\v_{r-1}, \alpha_{r-1}) + \langle \partial_\v f^s(\v_{r-1}, \alpha_{r-1}), \v-\v_{r-1} \rangle + \frac{\ell}{2} \|\v_{r-1}-\v\|^2 \leq f^s(\v, \alpha_{r-1}).
\label{strong_x}  
\end{equation}   

By $3\ell$-smoothness of $f^s(\v, \alpha)$ in $\v$, we have
\begin{equation} 
\begin{split}
f^s(\v_r, \alpha) &\leq f^s(\v_{r-1}, \alpha) + \langle \partial_\v f^s(\v_{r-1}, \alpha), \v_r-\v_{r-1}\rangle + \frac{3\ell}{2}\|\v_r-\v_{r-1}\|^2 \\ 
&= f^s(\v_{r-1}, \alpha) + \langle \partial_\v f^s(\v_{r-1}, \alpha_{r-1}), \v_r-\v_{r-1}\rangle + \frac{3\ell}{2}\|\v_r-\v_{r-1}\|^2 \\
&~~~+ \langle \partial_\v f^s(\v_{r-1}, \alpha)-\partial_\v f^s(\v_{r-1}, \alpha_{r-1}), \v_r-\v_{r-1}\rangle \\
&\overset{(a)}\leq f^s(\v_{r-1}, \alpha) + \langle \partial_{\v}f^s(\v_{r-1}, \alpha_{r-1}), \v_r- \v_{r-1}\rangle + \frac{3\ell}{2}\|\v_r - \v_{r-1}\|^2 \\
&~~~~ + \ell |\alpha_{r-1} - \alpha| \|\v_r - \v_{r-1}\| \\
&\overset{(b)}\leq f^s(\v_{r-1}, \alpha) + \langle \partial_{\v}f^s(\v_{r-1}, \alpha_{r-1}), \v_r - \v_{r-1}\rangle + \frac{3\ell}{2}\|\v_r - \v_{r-1}\|^2 \\
&~~~~ + \frac{\mu_2}{6}(\alpha_{r-1} - \alpha)^2 +  \frac{3\ell^2}{2\mu_2} \|\v_r - \v_{r-1}\|^2, \\ 
\end{split}  
\label{smooth_x} 
\end{equation} 
where $(a)$ holds because that we know $\partial_{\v} f^s(\v, \alpha)$ is $\ell$-Lipschitz in $\alpha$ since $f(\v, \alpha)$ is $\ell$-smooth and $(b)$ holds by Young's inequality. 

Adding (\ref{strong_x}) and (\ref{smooth_x}), by rearranging terms, we have
\begin{small} 
\begin{equation}  
\begin{split}
& f^s(\v_{r-1}, \alpha_{r-1}) + f^s(\v_r, \alpha)  \\
& \leq f^s(\v, \alpha_{r-1}) +
f^s(\v_{r-1}, \alpha) + \langle \partial_{\v}f^s(\v_{r-1}, \alpha_{r-1}), \v_r - \v \rangle \\  
&~~~ + \frac{3\ell + 3\ell^2/\mu_{2}}{2} \|\v_r-\v_{r-1}\|^2  
 - \frac{\ell}{2}\|\v_{r-1}-\v\|^2   
 + \frac{\mu_{2}}{6} (\alpha_{r-1} - \alpha)^2.
\end{split} 
\label{sum_x}  
\end{equation}
\end{small}

We know $f^s(\v, \alpha)$ is $\mu_2$-strong concave in $\alpha$ ($-f^s(\v, \alpha)$ is $\mu_{2}$-strong convexity of  in $\alpha$). 
Thus, we have 
\begin{small}
\begin{equation}
\begin{split}  
-f^s(\v_{r-1}, \alpha_{r-1}) 
- \langle\partial_{\alpha}f^s(\v_{r-1}, \alpha_{r-1}), \alpha - \alpha_{r-1}\rangle + \frac{\mu_{2}}{2}(\alpha - \alpha_{r-1})^2 \leq -f^s(\v_{r-1}, \alpha). 
\end{split} 
\label{strong_y} 
\end{equation}
\end{small}

Since $f^s(\v, \alpha)$ is $\ell$-smooth in $\alpha$, we get
\begin{small} 
\begin{equation} 
\begin{split}
& - f^s(\v, \alpha_r) \leq -f^s(\v, \alpha_{r-1}) - \langle \partial_{\alpha}f^s(\v, \alpha_{r-1}), \alpha_r -\alpha_{r-1} \rangle + \frac{\ell}{2}(\alpha_r - \alpha_{r-1})^2\\
&= -f^s(\v, \alpha_{r-1}) 
- \langle \partial_{\alpha} f^s(\v_{r-1}, \alpha_{r-1}), \alpha_r - \alpha_{r-1}\rangle 
+ \frac{\ell}{2}(\alpha_r - \alpha_{r-1})^2 \\
&~~~~~ - \langle \partial_{\alpha} (f^s(\v,\alpha_{r-1}) - f^s(\v_{r-1}, \alpha_{r-1})), \alpha_r - \alpha_{r-1}\rangle  \\
&\overset{(a)}\leq -f^s(\v, \alpha_{r-1}) - \langle \partial_{\alpha} f^s(\v_{r-1}, \alpha_{r-1}), \alpha_r - \alpha_{r-1}\rangle + \frac{\ell}{2} (\alpha_r - \alpha_{r-1})^2 
+ \ell\| \v - \v_{r-1}\| |\alpha_r - \alpha_{r-1}|\\ 
&\leq -f^s(\v, \alpha_{r-1}) - \langle \partial_{\alpha} f^s(\v_{r-1}, \alpha_{r-1}), \alpha_r - \alpha_{r-1}\rangle 
+ \frac{\ell}{2} (\alpha_r - \alpha_{r-1})^2 + \frac{\ell}{6} \|\v_{r-1} - \v\|^2  +  \frac{3\ell}{2}(\alpha_r - \alpha_{r-1})^2\\  
\end{split} 
\label{smooth_y}
\end{equation}
\end{small}
where (a) holds because that $\partial_{\alpha} f^s(\v, \alpha)$ is $\ell$-Lipschitz in $\alpha$. 

Adding (\ref{strong_y}), (\ref{smooth_y}) and arranging terms, we have
\begin{small}
\begin{equation}
\begin{split}
& -f^s(\v_{r-1}, \alpha_{r-1}) 
- f^s(\v, \alpha_r) 
\leq-f^s(\v_{r-1}, \alpha) 
-f^s(\v, \alpha_{r-1}) - \langle \partial_{\alpha} f^s(\v_{r-1}, \alpha_{r-1}), \alpha_r - \alpha\rangle \\
&~~~~~~ + 2\ell(\alpha_r - \alpha_{r-1})^2 
+ \frac{\ell}{6} \| \v_{r-1} - \v\|^2 
- \frac{\mu_{2}}{2}(\alpha - \alpha_{r-1})^2.
\end{split}   
\label{sum_y} 
\end{equation}
\end{small}

Adding (\ref{sum_x}) and (\ref{sum_y}), we get
\begin{equation}
\begin{split}
 &f^s(\v_r, \alpha)  - f^s(\v, \alpha_r) \leq   
\langle \partial_{\v}f^s(\v_{r-1}, \alpha_{r-1}), \v_r - \v \rangle - \langle \partial_{\alpha} f^s(\v_{r-1}, \alpha_{r-1}), \alpha_r - \alpha\rangle \\
&~~~ +\frac{3\ell + 3\ell^2/\mu_{2}}{2}  \|\v_r-\v_{r-1}\|^2 
 + 2\ell (\alpha_r- \alpha_{r-1})^2  -\frac{\ell}{3}\|\v_{r-1}-\v\|^2   - \frac{\mu_2}{3} (\alpha_{r-1}-\alpha)^2  
\end{split} 
\end{equation} 

Taking average over $r = 1, ..., R$, we get
\begin{small} 
\begin{equation*} 
\begin{split} 
&\frac{1}{R} \sum\limits_{r=1}^{R} 
[f^s(\v_r, \alpha)  - f^s(\v, \alpha_r) ]  \leq \frac{1}{R} \sum\limits_{r=1}^{R}\bigg[ \underbrace{\langle 
\partial_{\v}f^s(\v_{r-1}, \alpha_{r-1}), \v_r - \v\rangle}_{B_4} 
+ \underbrace{\langle \partial_{\alpha} f^s(\v_{r-1}, \alpha_{r-1}),  \alpha - \alpha_r \rangle}_{B_5}  \\
&~~~~~~~~~~ 
+\frac{3\ell + 3\ell^2/\mu_{2}}{2}  \|\v_r-\v_{r-1}\|^2 
 + 2\ell (\alpha_r - \alpha_{r-1})^2 
 -\frac{\ell}{3}\|\v_{r-1}-\v\|^2  
- \frac{\mu_2}{3}(\alpha_{r-1}-\alpha)^2  \bigg] 
\end{split} 
\end{equation*}      
\end{small}
\end{proof}

$B_4$ and $B_5$ can be bounded by the following lemma. For simplicity of notation, we define
\begin{equation}
\begin{split}
\Xi_r = \frac{1}{KI} \sum\limits_{k,t} 
\E[\|\v^k_{r,t} - \v_r\|^2 + (\alpha^k_{r,t} - \alpha_r)^2 ] ,
\end{split}
\end{equation}
which is the drift of the variables between te sequence in $r$-th round and the ending point, and 
\begin{equation} 
\begin{split}
\mathcal{E}_r = \frac{1}{KI} \sum\limits_{k,t} 
\E[\|\v^k_{r,t} - \v_{r-1}\|^2 + (\alpha^k_{r,t} - \alpha_{r-1})^2 ] ,
\end{split}
\end{equation}
which is the drift of the variables between te sequence in $r$-th round and the starting point.

$B_4$ can be bounded as 
\begin{lem} 
\label{lem:product_x}
\begin{equation*} 
\begin{split}
& \E\left\langle \nabla_{\v} f^s(\v_{r-1}, \alpha_{r-1}), \v_r - \v\right\rangle \\
& \leq \frac{3\ell}{2} \mathcal{E}_r +  \frac{\ell}{3}\E\|\bar{\v}_r - \v\|^2
+ \frac{3\teta}{2} \E\left\|\frac{1}{N K}\sum\limits_{i,t}[\nabla_{\v}f^s_k(\v_{r,t}^k, \alpha_{r,t}^k) - \nabla_{\v}F^s_k(\v_{r,t}^k, \alpha_{r,t}^k; z_{r,t}^k)]\right\|^2 \\
&~~~~~+\frac{1}{2\teta} \E (\|\v_{r-1}-\v\|^2 - \|\v_{r-1} - \v_r\|^2 - \|\v_r - \v\|^2) + \frac{1}{2\teta} \E (\|\tilde{\v}_{r-1}-\v\|^2  - \|\tilde{\v}_r-\v\|^2 ), 
\end{split} 
\end{equation*}
and
\begin{equation*}  
\begin{split}
&\E \langle \nabla_{\alpha} f^s(\v_{r-1}, \alpha_{r-1}), y - \alpha_r \rangle \leq \frac{3\ell^2}{2\mu_2} \mathcal{E}_r + \frac{\mu_2}{3} \E(\bar{\alpha}_r - \alpha)^2  \\
&~~~ +\frac{3\teta}{2} \E \left( \frac{1}{N K} \sum\limits_{i,t}[ \nabla_{\alpha} f^s_k(\v_{r,t}^k, \alpha_{r,t}^k) -\nabla_{\alpha} F^s_k (\v_{r,t}^k, \alpha_{r,t}^k; z_{r,t}^k)] \right)^2\\ 
&~~~ +\frac{1}{2\teta} \E ((\bar{\alpha}_{r-1} - \alpha)^2 - (\bar{\alpha}_{r-1} - \bar{\alpha}_r)^2 - (\bar{\alpha}_r - \alpha)^2)
+ \frac{1}{2\teta} \E((\alpha - \tilde{\alpha}_{r-1})^2 - (\alpha - \tilde{\alpha}_r)^2 ). 
\end{split} 
\end{equation*} 
\end{lem}

\begin{proof}  
\begin{equation}
\begin{split} 
&\langle \nabla_{\v}f^s(\v_{r-1}, \alpha_{r-1}), \v_r - \v\rangle =  \bigg\langle \frac{1}{K I}\sum\limits_{k,t} \nabla_{\v}f^s_k(\v_{r-1}, \alpha_{r-1}), \v_r - \v\bigg\rangle \\
&\leq \bigg\langle \frac{1}{KI}\sum\limits_{k,t} [\nabla_{\v}f^s_k(\v_{r-1}, \alpha_{r-1}) - \nabla_{\v} f^s_k(\v_{r-1}, \alpha_{r,t}^k)], \v_r - \v\bigg\rangle ~~~~~~~~~~~\textcircled{\small{1}}\\
&~~~ +\bigg\langle \frac{1}{KI}\sum\limits_{i,t} [\nabla_{\v}f^s_k(\v_{r-1}, \alpha_{r,t}^k) - \nabla_{\v}f^s_k(\v_{r,t}^k, \alpha_{r,t}^k)], \v_r - \v\bigg\rangle ~~~~~~~~~~~\textcircled{\small{2}}\\
&~~~ +\bigg\langle \frac{1}{KI}\sum\limits_{k,t}[\nabla_{\v}f^s_k(\v_{r,t}^k,\alpha_{r,t}^k) - \nabla_{\v}F^s_k(\v^k_{r,t}, \alpha_{r,t}^k; z_{r,t}^k)], \v_r - \v\bigg\rangle~~~~~~~~~~~\textcircled{\small{3}} \\
&~~~ + \bigg\langle \frac{1}{KI}\sum\limits_{k,t} \nabla_{\v}F^s_k(\v^k_{r,t}, \alpha_{r,t}^k; z_{r,t}^k), \v_r - \v\bigg\rangle~~~~~~~~~~~\textcircled{\small{4}}
\end{split} 
\label{circledx} 
\end{equation}

Then we will bound \textcircled{\small{1}}, \textcircled{\small{2}} and \textcircled{\small{3}}, respectively, 
\begin{equation}
\begin{split}
\textcircled{\small{1}} &\overset{(a)}\leq \frac{3}{2\ell} \left\| \frac{1}{KI} \sum\limits_{k,t}[\nabla_{\v}f^s_k(\v_{r-1}, \alpha_{r-1}) - \nabla_{\v}f^s_k(\v_{r-1}, \alpha_{r,t}^{k})] \right\|^2 + \frac{\ell}{6}\|\v_r - \v\|^2\\
&\overset{(b)}\leq \frac{3}{2 \ell} \frac{1}{KI}\sum\limits_{k,t} \|\nabla_{\v}f^s_k(\v_{r-1}, \alpha_{r-1}) - \nabla_{\v} f^s_k(\v_{r-1}, \alpha_{r,t}^k)\|^2 + \frac{\ell}{6}\|\v_r - \v\|^2\\
&\overset{(c)}\leq \frac{3\ell}{2}\frac{1}{KI}\sum\limits_{k,t}\|\alpha_{r-1} - \alpha_{r,t}^k\|^2 + \frac{\ell}{6}\|\v_r - \v\|^2, 
\end{split}
\label{circled1_sca}
\end{equation}
where (a) follows from Young's inequality, (b) follows from Jensen's inequality.
and (c) holds because $\nabla_{\v} f^s_k(\v, \alpha)$ is $\ell$-smooth in $\alpha$.
Using similar techniques, we have 
\begin{equation} 
\begin{split}   
\textcircled{\small{2}} &\leq \frac{3}{2\ell} \frac{1}{KI}\sum\limits_{k,t}\| \nabla_{\v} f^s_k(\v_{r-1}, \alpha_{r,t}^k) - \nabla_{\v} f^s_k(\v_{r,t}^k, \alpha_{r,t}^k)\|^2 + \frac{\ell}{6}\|\v_r -\v\|^2 \\
& \leq \frac{3 \ell}{2}\frac{1}{KI}\sum\limits_{k,t} \|\v_{r-1} - \v_{r,t}^{i}\|^2 + \frac{\ell}{6} \|\v_r - \v\|^2.
\end{split} 
\label{circled2_sca}
\end{equation}

Let $\hat{\v}_r = \arg\min\limits_{\v} \left(\frac{1}{ K I}\sum\limits_{k,t} \nabla_{\v} f_k^s(\v^k_{r,t}, y^k_{r,t})\right)^\top \v + \frac{1}{2\teta} \|\v - \v_{r-1}\|^2$,  
then we have 
\begin{equation}
\begin{split} 
\bar{\v}_r - \hat{\v}_r = \frac{\teta}{KI}\sum\limits_{k,t} \bigg(\nabla_{\v}f^s_k(\v^{k}_{r,t}, \alpha^k_{r,t}) - \nabla_{\v}f^s_k(\v^{k}_{r,t}, y^k_{r,t}; z_{r,t}^k)\bigg). 
\end{split}
\end{equation}

Hence we get 
\begin{equation}
\begin{split}
&\textcircled{\small{3}} = \left\langle \frac{1}{KI}\sum\limits_{k,t}[\nabla_{\v}f^s_k(\v_{r,t}^k, \alpha_{r,t}^k) - \nabla_{\v} F^s_k(\v_{r,t}^k, \alpha_{r,t}^k; z_{r,t}^k)], \v_r - \hat{\v}_r \right\rangle \\
&~~~~+ \left\langle \frac{1}{KI}\sum\limits_{k,t}[\nabla_{\v}f^s_k(\v_{r,t}^k, \alpha_{r,t}^k) - \nabla_{\v} F^s_k(\v_{r,t}^k, \alpha_{r,t}^k; z_{r,t}^k)], \hat{\v}_r - \v \right\rangle\\
& = {\teta} \left\|\frac{1}{KI} \sum\limits_{k,t} [\nabla_{\v}f^s_k(\v_{r,t}^k, \alpha_{r,t}^k) - \nabla_{\v} F^s_k(\v_{r,t}^k, \alpha_{r,t}^k; z_{r,t}^k)]  \right\|^2\\
&~~~~+ \left\langle \frac{1}{KI}\sum\limits_{k,t}[\nabla_{\v}f^s_k(\v_{r,t}^k, \alpha_{r,t}^k) - \nabla_{\v} F^s_k(\v_{r,t}^k, \alpha_{r,t}^k; z_{r,t}^k)], \hat{\v}_r - \v \right\rangle .
\end{split}
\label{pre_circled3_sca}
\end{equation} 

Define another auxiliary sequence as
\begin{equation}
\begin{split} 
\Tilde{\v}_r = \Tilde{\v}_{r-1} - \frac{\teta}{KI} \sum\limits_{k,t} \left(\nabla_\v F^s_k(\v^k_{r,t}, y^k_{r,t}; z_{r,t}^k) - \nabla_\v f^s_k(\v_{r,t}^k, \alpha_{r,t}^k) \right), \text{ for $r>0$; } \Tilde{\v}_0 = \v_0.  
\end{split}
\end{equation}

Denote 
\begin{equation}
\Theta_{r} (\v) = \left(\frac{1}{KI}\sum\limits_{k,t} (\nabla_\v F^s_k(\v^k_{r,t}, y^k_{r,t}; z_{r,t}^k) - \nabla_\v f^s_k(\v_{r,t}^k, \alpha_{r,t}^k) )\right)^\top \v + \frac{1}{2\teta}  \|\v-\Tilde{\v}_{r-1}\|^2. 
\end{equation} 
Hence, for the auxiliary sequence $\Tilde{\alpha}_r$, we can verify that
\begin{equation}
\Tilde{\v}_r = \arg\min\limits_\v \Theta_{r}(\v).
\end{equation} 
Since $\Theta_{r}(\v)$ is $\frac{1}{\teta}$-strongly convex, we have
\begin{small}
\begin{align} 
\begin{split}
& \frac{1}{2\teta}\|\v - \tilde{\v}_r\|^2 \leq \Theta_{r}(\v) - \Theta_{r}(\tilde{\v}_r) \\
& = \bigg(\frac{1}{KI}\sum\limits_{k,t} ( 
 \nabla_{\v} F^s_k(\v_{r,t}^{k},\alpha_{r,t}^k; z_{r,t}^k) 
-\nabla_{\v} f^s_k(\v_{r,t}^{k},\alpha_{r,t}^k))\bigg)^\top \v + \frac{1}{2\teta}\|\v - \tilde{\v}_{r-1}\|^2 \\  
&~~~ -\bigg(\frac{1}{KI}\sum\limits_{k,t}( 
 \nabla_{\v} F^s_k(\v_{r,t}^{i},\alpha_{r,t}^k; z_{r,t}^k) 
-\nabla_{\v} f^s_k(\v_{r,t}^{i},\alpha_{r,t}^k))\bigg)^\top\tilde{\v}_r 
- \frac{1}{2\teta}\|\tilde{\v}_r  - \tilde{\v}_{r-1}\|^2 \\ 
& = \bigg(\frac{1}{KI}\sum\limits_{k,t}(  
 \nabla_{\v} F^s_k(\v_{r,t}^{k},\alpha_{r,t}^k; z_{r,t}^k) 
-\nabla_{\v} f^s_k(\v_{r,t}^{k},\alpha_{r,t}^k))\bigg)^\top (\v - \tilde{\v}_{r-1}) 
+ \frac{1}{2\teta}\|\v - \tilde{\v}_{r-1}\|^2 \\ 
&~~~ -\bigg(\frac{1}{KI} \sum\limits_{k,t}(  
 \nabla_{\alpha} F^s_k(\v_{r,t}^{k},\alpha_{r,t}^k; z_{r,t}^k) 
-\nabla_{\alpha} f^s_k(\v_{r,t}^{k},\alpha_{r,t}^k))\bigg)^\top 
(\tilde{\v}_{r} - \tilde{\v}_{r-1})  
- \frac{1}{2\teta} \|\tilde{\v}_r  - \tilde{\v}_{r-1}\|^2 \\
&\leq \bigg(\frac{1}{KI}\sum\limits_{k,t}(
 \nabla_{\v} F^s_k(\v_{r,t}^{k},\alpha_{r,t}^k; z_{r,t}^k)  
-\nabla_{\v} f^s_k(\v_{r,t}^{k},\alpha_{r,t}^k))\bigg)^\top (\v - \tilde{\v}_{r-1})+ \frac{1}{2\teta}\|\v - \tilde{\v}_{r-1}\|^2 \\ 
&~~~+ \frac{\teta}{2}\bigg\|\frac{1}{KI}\sum\limits_{k,t}( 
 \nabla_{\v} F^s_k(\v_{r,t}^{k},\alpha_{r,t}^k; z_{r,t}^k) 
-\nabla_{\v} f^s_k(\v_{r,t}^{k},\alpha_{r,t}^k))\bigg\|^2.
\end{split} 
\end{align} 
\end{small} 

Adding this with (\ref{pre_circled3_sca}), we get
\begin{equation} 
\begin{split}
\textcircled{3} \leq &\frac{3\teta}{2}\bigg\|\frac{1}{KI} \sum\limits_{k,t}( 
 \nabla_{\v} F^s_k(\v_{r,t}^{k},\alpha_{r,t}^k; z_{r,t}^k) 
-\nabla_{\v} f^s_k(\v_{r,t}^{k},\alpha_{r,t}^k))\bigg\|^2 + \frac{1}{2\teta }\|\v - \tilde{\v}_{r-1}\|^2
- \frac{1}{2\teta}\|\v - \tilde{\v}_r\|^2 \\
&+ \left\langle \frac{1}{KI}\sum\limits_{k,t}[\nabla_{\v}f^s_k(\v_{r,t}^k, \alpha_{r,t}^k) - \nabla_{\v} F^s_k(\v_{r,t}^{k}, \alpha_{r,t}^k; z_{r,t}^k)], \hat{\v}_r - \Tilde{\v}_{r-1}  \right\rangle.
\end{split}
\label{circled3_sca} 
\end{equation}

\textcircled{\small{4}} can be bounded as
\begin{small}
\begin{equation}
\begin{split}
\textcircled{\small{4}} &= \frac{1}{\teta}\langle \v_r - \v_{r-1}, \v - \v_r\rangle =\frac{1}{2\teta}(\|\v_{r-1}-\v\|^2 - \|\v_{r-1} - \v_r\|^2 - \|\v_r-\v\|^2 ) 
\end{split}  
\label{circled4_sca} 
\end{equation} 
\end{small}

Plug (\ref{circled1_sca}), (\ref{circled2_sca}), (\ref{circled3_sca}) and  (\ref{circled4_sca}) into (\ref{circledx}), we get
\begin{equation*} 
\begin{split}
& \E\left\langle \nabla_{\v} f^s(\v_{r-1}, \alpha_{r-1}), \v_r - \v \right\rangle \\
& \leq \frac{3\ell}{2} \mathcal{E}_r + \frac{\ell}{3}\E\|\bar{\v}_r - \v\|^2
+ \frac{3\teta}{2} \E\left\|\frac{1}{KI} \sum\limits_{k,t}[\nabla_{\v}f^s_k(\v_{r,t}^k, \alpha_{r,t}^k) - \nabla_{\v}F^s_k(\v_{r,t}^k, \alpha_{r,t}^k; z_{r,t}^k)]\right\|^2 \\
&~~~~~+\frac{1}{2\teta} \E (\|\v_{r-1}-\v\|^2 - \|\v_{r-1} - \v_r\|^2 - \|\v_r - \v\|^2) + \frac{1}{2\teta} \E (\|\tilde{\v}_{r-1}-\v\|^2  - \|\tilde{\v}_r-\v\|^2 )
\end{split}  
\label{grad_x}
\end{equation*}

Similarly for $\alpha$, noting $f^s_k$ is $\ell$-smooth and $\mu_2$-strongly concave in $\alpha$, 
\begin{equation*}  
\begin{split} 
&\E \langle \nabla_{\alpha} f^s(\v_{r-1}, \alpha_{r-1}), y - \alpha_r \rangle \leq \frac{3\ell^2}{2\mu_2} \mathcal{E}_r + \frac{\mu_2}{3} \E(\bar{\alpha}_r - \alpha)^2\\
&~~~ +\frac{3\teta}{2} \E \left( \frac{1}{KI} \sum\limits_{k,t}[ \nabla_{\alpha} f^s_k(\v_{r,t}^k, \alpha_{r,t}^k) -\nabla_{\alpha} F^s_k (\v_{r,t}^k, \alpha_{r,t}^k; z_{r,t}^k)] \right)^2\\ 
&~~~ +\frac{1}{2\teta} \E ((\bar{\alpha}_{r-1} - \alpha)^2 - (\bar{\alpha}_{r-1} - \bar{\alpha}_r)^2 - (\bar{\alpha}_r - \alpha)^2) 
+ \frac{1}{2\teta} \E((\alpha - \tilde{\alpha}_{r-1})^2 - (\alpha - \tilde{\alpha}_r)^2 ) 
\end{split} 
\label{grad_y} 
\end{equation*} 
\end{proof}

We show the following lemmas where $\Xi$ and $\mathcal{E}$ are coupled.
\begin{lem} 
\label{lem:bound_A}
\begin{equation}
\begin{split}
\Xi_r 
&\leq  4\mathcal{E}_r + 8\teta^2 [\|\nabla_\v f(\v_r, \alpha_r)\|^2 + (\nabla_\alpha f(\v_r, \alpha_r))^2] + \frac{5\teta^2 \sigma^2}{KI} .
\end{split}
\end{equation} 
\end{lem}

\begin{proof} 
\begin{equation}
\begin{split}
&\E[\|\v_r - \v_{r-1}\|^2] = \E\left\| -\frac{\teta}{KI} \sum\limits_{k,t} (\nabla_\v f^s_k(\v^k_{r,t}, \alpha^k_{r,t}; z^k_{r,t})-c_\v^k+c_\v)\right\|^2 \\
& =   \E\left\| -\frac{\teta}{KI} \sum\limits_{k,t} \left[ \nabla_\v f^s_k(\v^k_{r,t}, \alpha^k_{r,t}; z^k_{r,t}) - \nabla_\v f^s_k(\v^k_{r,t}, \alpha^k_{r,t}) +  \nabla_\v f^s_k(\v^k_{r,t}, \alpha^k_{r,t}) \right]\right\|^2 \\ 
& \leq  \E\left\| -\frac{\teta}{KI} \sum\limits_{k,t} \left[ \nabla_\v f^s_k(\v^k_{r,t}, \alpha^k_{r,t}) \right]\right\|^2 + \frac{\teta^2  \sigma^2}{KI}\\  
& = \E\left\| -\frac{\teta}{KI} \sum\limits_{k,t} [\nabla_\v f^s_k(\v^k_{r,t}, \alpha^k_{r,t}) - \nabla_\v f^s_k(\v_{r-1},\alpha_{r-1})] + \teta \nabla_\v f^s(\v_{r-1},\alpha_{r-1}) )\right\|^2 + \frac{\teta^2   \sigma^2}{KI}\\  
& \leq  2\E\left\| -\frac{\teta}{KI} \sum\limits_{k,t} [\nabla_\v f^s_k(\v^k_{r,t}, \alpha^k_{r,t}) - \nabla_\v f^s_k(\v_{r-1},\alpha_{r-1})]\right\|^2 + 2 \teta^2 \E\left\| \nabla_\v f^s(\v_{r-1},\alpha_{r-1}) \right\|^2  + \frac{\teta^2 \sigma^2}{KI}\\ 
&\leq \frac{2\teta^2 \ell^2}{KI}\sum\limits_{k,t} \E[\|\v^k_{r,t} - \v_{r-1}\|^2 + (\alpha^k_{r,t} - \alpha_{r-1})^2]   + 2 \teta^2 \E\left\| \nabla_\v f^s(\v_{r-1},\alpha_{r-1}) \right\|^2  + \frac{\teta^2 \sigma^2}{KI} \\
&\leq 2\teta^2 \ell^2 \mathcal{E}_r + 2 \teta^2 \E\left\| \nabla_\v f^s(\v_{r-1},\alpha_{r-1}) \right\|^2  + \frac{\teta^2 \sigma^2}{KI}.
\end{split}    
\end{equation} 

Similarly, 
\begin{equation}
\begin{split}
&\E[(\alpha_r - \alpha_{r-1})^2] 
\leq 2\teta^2 \ell^2 \mathcal{E}_r + 2 \teta^2 \E\left( \nabla_\alpha f^s(\v_{r-1},\alpha_{r-1}) \right)^2  +  \frac{\teta^2 \sigma^2}{KI}. 
\end{split}     
\end{equation}  

Using the $3\ell$-smoothness of $f^s$ and combining with above results, 
\begin{equation} 
\begin{split}
&\|\nabla_\v f^s (\v_{r-1}, \alpha_{r-1})\|^2+(\nabla_\alpha f^s(\v_{r-1}, \alpha_{r-1}))^2 \\
& = \|\nabla_\v f^s(\v_{r-1}, \alpha_{r-1})-\nabla_\v f^s(\v_{r}, \alpha_{r})+\nabla_\v f^s(\v_{r}, \alpha_{r}) \|^2\\
&~~~+(\nabla_\alpha f^s(\v_{r-1}, \alpha_{r-1})-\nabla_\alpha f^s(\v_{r}, \alpha_{r})+\nabla_\alpha f^s (\v_{r}, \alpha_{r}))^2 \\ 
& \leq 2 [\|\nabla_\v f^s (\v_r, \alpha_r)\|^2+(\nabla_\alpha f^s (\v_r, \alpha_r))^2] + 18 \ell^2(\|\v_{r-1}-\v_r\|^2 + (\alpha_{r-1} - \alpha_r )^2)   \\
&\leq  2 [\|\nabla_\v f^s(\v_r, \alpha_r)\|^2+(\nabla_\alpha f^s(\v_r, \alpha_r))^2] + 60 \ell^4 \teta^2 \mathcal{E}_r +\frac{40 \teta^2 \ell^2 \sigma^2}{KI}  \\ 
&\leq  2 [\|\nabla_\v f^s(\v_r, \alpha_r)\|^2+(\nabla_\alpha f^s(\v_r, \alpha_r))^2]  + \frac{\ell^2}{24} \mathcal{E}_r  +\frac{\sigma^2}{144 KI} . 
\end{split}
\label{equ:grad_v_plus_grad_alpha_1}
\end{equation} 

Thus,
\begin{equation}
\begin{split}
\Xi_r &= \frac{1}{KI} \sum\limits_{k,t} \E[\|\v^{k}_{r,t} - \v_{r}\|^2 +  (\alpha^{k}_{r,t} - \alpha_{r})^2]\\ 
&\leq  \frac{2}{KI} \sum\limits_{k,t} \E[\|\v^{k}_{r,t} - \v_{r-1}\|^2 + \|\v_{r-1} - \v_r\|^2 + (\alpha^{k}_{r,t} - \alpha_{r-1})^2 + (\alpha_{r-1}-\alpha_r)^2 ]\\ 
&\leq 2\mathcal{E}_r + 2\E[\|\v_{r-1} - \v_r\|^2 +(\alpha_{r-1}-\alpha_r)^2] \\ 
&\leq 2\mathcal{E}_r + 8\teta^2 \ell^2 \mathcal{E}_r + 4\teta^2 \E[(\nabla_\v f^s(\v_{r-1}, \alpha_{r-1}))^2 + (\nabla_\alpha f^s(\v_{r-1}, \alpha_{r-1}))^2]
+ \frac{4\teta^2 \sigma^2}{KI} \\
&\leq 3\mathcal{E}_r + 4\teta^2 \left( 2 [\|\nabla_\v f^s(\v_r, \alpha_r)\|^2 + (\nabla_\alpha f^s(\v_r, \alpha_r))^2] + \frac{\ell^2}{24}\mathcal{E}_r + \frac{\sigma^2}{144 KI}\right) + \frac{4\teta^2\sigma^2}{KI} \\
&\leq 4\mathcal{E}_r + 8\teta^2 [\|\nabla_\v f^s(\v_r, \alpha_r)\|^2 + (\nabla_\alpha f^s(\v_r, \alpha_r))^2] + \frac{5\teta^2 \sigma^2}{KI}. 
\end{split} 
\end{equation} 
\end{proof}

\begin{lem}
\label{lem:bound_B}
\begin{equation}
\begin{split}
&\mathcal{E}_r \leq \frac{\teta\sigma^2}{2\ell K \eta_g^2} + \teta \ell \Xi_{r-1}  +\frac{48\teta^2}{\eta_g^2} [ \|\nabla_\v f^s(\v_r, \alpha_r)\|^2 + (\nabla_\alpha f^s(\v_r, \alpha_r))^2]. 
\end{split}
\end{equation}
\end{lem}

\begin{proof}    
\begin{small}
\begin{equation}  
\begin{split}     
& \E\| \v^k_{r,t} - \v_{r-1}\|^2 = \E\| \v^k_{r,t-1} - \eta_l (\nabla_\v f_k(\v^k_{r,t-1}, y^k_{r,t-1}; z^k_{r,t-1}) -c_\v^k + c_\v) - \v_{r-1} \|^2 \\ 
&\leq \E\|\v^k_{r,t-1} - \eta_l (\nabla_\v f_k(\v^k_{r,t-1}, y^k_{r,t-1}) -\E[c^k_\v] + \E[c_\v]) - \v_{r-1} \|^2 + 2\eta_l^2 \sigma^2 \\
&\leq \left(1+\frac{1}{I-1} \right)\E\|\v^k_{r,t-1} - \v_{r-1}\|^2 + I \eta_l^2 \E \|\nabla_\v f_k(\v^k_{r,t-1}, \alpha^k_{r,t-1}) -\E[c^k_\v] + \E[c_\v ] \|^2 + 2\eta_l^2 \sigma^2,
\end{split}  
\end{equation} 
\end{small} 
where $\E[c^k_\v] = \frac{1}{I} \sum\limits_{t=1}^{I} f^s(\v^k_{r, t}, \alpha^k_{r,t}) $ and $\E[c_\v] = \frac{1}{K} \sum\limits_{k=1}^{K} \frac{1}{I} \sum\limits_{t=1}^{I} f^s(\v^k_{r, t}, \alpha^k_{r,t}) $. 

Then,
\begin{small}
\begin{equation}  
\begin{split}     
&I \eta_l^2 \E \|\nabla_\v f_k^s (\v^k_{r,t-1}, \alpha^k_{r,t-1}) -\E[c^k_\v] + \E[c_\v]  \|^2  \\
&\leq I \eta_l^2 \E \|\nabla_\v f_k^s(\v^k_{r,t-1}, \alpha^k_{r,t-1}) - \nabla_\v f_k^s(\v_{r-1}, \alpha_{r-1}) + (\E[c_\v]-\nabla_\v f^s(\v_{r-1}, \alpha_{r-1})) \\ 
&~~~~~~~~~~~~~~~ +\nabla_\v f^s(\v_{r-1}, \alpha_{r-1})  - (\E[c^k_\v] - \nabla_\v f_k^s(\v_{r-1}, \alpha_{r-1})) \|^2 \\
&\leq 4I\eta_l^2 \ell^2 \bigg(\E[\|\v^k_{r,t-1} - \v_{r-1}\|^2]
+ \E[\|\alpha^k_{r,t-1} - \alpha_{r-1}\|^2]\bigg) 
+ 4I\eta_l^2 \E[\|\E[c^k_\v] - \nabla_\v f_k^s (\v_{r-1}, \alpha_{r-1})\|^2] \\
&~~~ +4I\eta_l^2 \E[\|\E[c_\v]-\nabla_\v f^s(\v_{r-1}, \alpha_{r-1} \|^2] 
+4I\eta_l^2 \E[\|\nabla_\v f^s(\v_{r-1}, \alpha_{r-1})\|^2 \\
&\leq 4I\eta_l^2 \ell^2 \bigg(\E[\|\v^k_{k-1,r} - \v_{r-1}\|^2]
+ \E[(\alpha^k_{k-1,r} - \alpha_{r-1})^2]\bigg) \\
&~~~ + 4I\eta_l^2 \ell^2 \frac{1}{I}\sum_{\tau=1}^{I} \E[\|\v^k_{r-1,\tau} -\v_{r-1}\|^2
+ (\alpha^k_{r-1, \tau}-\alpha_{r-1})^2] \\
&~~~ +4I\eta_l^2 \ell^2  \frac{1}{KI}\sum_{j=1}^{K}\sum_{t=1}^{I}\E[\|\v^j_{r-1,t} -\v_{r-1}\|^2  
+ (\alpha^j_{r-1, k}-\alpha_{r-1})^2] + 4I\eta_l^2 \E[\|\nabla_\v f^s (\v_{r-1}, \alpha_{r-1})\|^2. 
\end{split} 
\end{equation}
\end{small} 
 
For $\alpha$, we have similar results, adding them together
\begin{equation}
\begin{split}
& \E\| \v^k_{k,r} - \v_{r-1}\|^2 + \E( \alpha^k_{k,r} - \alpha_{r-1})^2 \\
& \leq 
\left(1+\frac{1}{K-1} +8 K \eta_l^2 \ell^2  \right) (\E\|\v^k_{k-1,r} - \v_{r-1}\|^2+\E(\alpha^k_{k-1,r} - \alpha_{r-1})^2) \\
 &~~~ + 2\eta_l^2\sigma^2 + 4 I \eta_l^2 \ell^2 \Xi_{r-1}  + 4I\eta_l^2 \frac{1}{I} \sum\limits_{\tau=1}^{I} 
 \E[\|\v_{r-1,\tau}^k-\v_{r-1}\|^2+ (\alpha_{r-1,\tau}^k - \alpha_{r-1} )^2]  \\  
&~~~  + 4 I \eta_l^2 \E[\|\nabla_\v f^s (\v_{r-1}, \alpha_{r-1})\|^2+ (\nabla_\alpha f^s(\v_{r-1}, \alpha_{r-1}))^2  ]  
\end{split} 
\end{equation} 

Taking average over all machines,
\begin{small}
\begin{equation}
\begin{split}
&\frac{1}{K} \sum\limits_k \E\| \v^k_{r,t} - \v_{r-1}\|^2 + \E(\alpha^k_{r,t} - \alpha_{r-1})^2 \\
&\leq \left(1+\frac{1}{I-1} +8 I \eta_l^2 \ell^2  \right)\frac{1}{K}\sum_k (\E\|\v^k_{r,t-1} - \v_{r-1}\|^2+\E(\alpha^k_{r,t-1} - \alpha_{r-1})^2) + 2\eta_l^2\sigma^2 \\
&~~~ +8 I \eta_l^2 \ell^2 \Xi_{r-1} 
+ 4I\eta_l^2 \E[\|\nabla_\v f^s(\v_{r-1}, \alpha_{r-1})\|^2
+ (\nabla_\alpha f^s(\v_{r-1}, \alpha_{r-1}))^2  ]   ]   \\ 
&\leq \left(2\eta_l^2\sigma^2 +8I\eta_l^2 \ell^2 \Xi_{r-1}  + 4I\eta_l^2 \E[\|\nabla_\v f^s(\v_{r-1}, \alpha_{r-1})\|^2+ (\nabla_\alpha f^s(\v_{r-1}, \alpha_{r-1}))^2  \right) \left(  \sum\limits_{\tau=0}^{t-1}(1+\frac{1}{I-1} +8 I \eta_l^2 \ell^2 )^\tau \right)\\ 
&\leq  \left(\frac{2\teta^2\sigma^2}{I^2 \eta_g^2} +\frac{8 \teta^2 \ell^2}{I  \eta_g^2} \Xi_{r-1}  + \frac{4\teta^2}{I\eta_g^2} \E[\|\nabla_\v f^s(\v_{r-1}, \alpha_{r-1})\|^2+ (\nabla_\alpha f^s (\v_{r-1}, \alpha_{r-1}))^2  ] \right) 3I   \\ 
&\leq \left(\frac{\teta \sigma^2}{24 \ell I^2  \eta_g^2} + \frac{\teta\ell}{3I\eta_g^2} \Xi_{r-1} + \frac{4\tilde{\eta}^2}{I \eta_g^2}  \E[\|\nabla_\v f^s(\v_{r-1}, \alpha_{r-1})\|^2+ (\nabla_\alpha f^s(\v_{r-1}, \alpha_{r-1}))^2  ] \right) 3I. 
\end{split}  
\end{equation} 
\end{small}

Taking average over $t=1,...,I$, 
\begin{equation}
\begin{split}
&\mathcal{E}_r \leq \frac{\teta \sigma^2}{8 \ell I \eta_g^2} + {\teta\ell} \Xi_{r-1} +   \frac{12\tilde{\eta}^2}{\eta_g^2}  \E[\|\nabla_\v f^s(\v_{r-1}, \alpha_{r-1})\|^2+ (\nabla_\alpha f^s (\v_{r-1}, \alpha_{r-1}))^2  ] 
\end{split} 
\end{equation}


Using (\ref{equ:grad_v_plus_grad_alpha_1}), we have 
\begin{equation}
\begin{split}
&\mathcal{E}_r \leq \frac{\teta\sigma^2}{8\ell I \eta_g^2} + \teta \ell \Xi_{r-1} 
 +  \frac{12\tilde{\eta}^2}{\eta_g^2}  \left( 4 [ \|\nabla_\v f^s(\v_r, \alpha_r)\|^2+(\nabla_\alpha f^s(\v_r, \alpha_r))^2] + \frac{\ell^2}{24} \mathcal{E}_r +\frac{\sigma^2}{144 K I} \right). 
\end{split}
\end{equation}

Rearranging terms,
\begin{equation}
\begin{split}
\mathcal{E}_r \leq \frac{\teta\sigma^2}{2\ell I \eta_g^2} + \teta \ell \Xi_{r-1}  +\frac{48\teta^2}{\eta_g^2} [ \|\nabla_\v f^s(\v_r, \alpha_r)\|^2 + (\nabla_\alpha f^s(\v_r, \alpha_r))^2] 
\end{split}
\end{equation} 
\end{proof}

\subsection{Main Proof of Lemma \ref{lem:codasca:one_stage}}
\begin{proof}
Plugging Lemma \ref{lem:product_x}  into Lemma \ref{lem:codasca:one_stage_split}, we get 

\begin{small}
\begin{equation}
\begin{split}  
& \frac{1}{R} \sum\limits_{r=1}^{R} [f^s(\v_r, \alpha) -  f^s(\v, \alpha_r)]\\
&\leq \frac{1}{R}\sum\limits_{r=1}^{R} \Bigg[ \underbrace{ \left(\frac{3\ell+3\ell^2/\mu_2}{2} - \frac{1}{2\teta}\right) \|\v_{r-1} - \v_r\|^2 +  \left(2\ell - \frac{1}{2\teta}\right)(\alpha_r -  \alpha_{r-1})^2}_{C_1}\\ 
&+\underbrace{\left(\frac{1}{2\teta} - \frac{\mu_2}{3} \right) (\alpha_{r-1} - \alpha)^2 - \left(\frac{1}{2\teta} - \frac{\mu_2}{3}\right)(\alpha_r-\alpha)^2}_{C_2}  \\ 
& +\underbrace{\left(\frac{1}{2\teta }-\frac{\ell}{3}  \right) \|\v_{r-1} - \v\|^2 - \left(\frac{1}{2\teta} -  \frac{\ell}{3}\right)\|\v_r - \v\|^2}_{C_3}\\ 
& +\underbrace{\frac{1}{2\teta} ((\alpha - \Tilde{\alpha}_{r-1})^2 -  (\alpha-\Tilde{\alpha}_r)^2)}_{C_4}  +\underbrace{ 
\left(\frac{3 \ell}{2} + \frac{3 \ell^2}{2\mu_2}\right)\mathcal{E}_r}_{C_5}\\
&+\underbrace{\frac{3\teta}{2}  \left\|\frac{1}{K I}\sum\limits_{k,i} [\nabla_{\v} f_k^s(\v_{r,t}^k, \alpha_{r,t}^k) - \nabla_{\v}F_k^s(\v_{r,t}^k, \alpha_{r,t}^k; z_{r,t}^k)]\right\|^2}_{C_6} \\
&
+  \underbrace{\frac{3\teta}{2} \left(  \frac{1}{K I} \sum\limits_{k,i}  \nabla_{\alpha}f_k^s(\v_{r,t}^k, \alpha_{r,t}^k) - \nabla_{\alpha} F_k^s(\v_{r,t}^k, \alpha_{r,t}^k; z_{r,t}^k) \right)^2}_{C_7}.
\end{split}
\label{before_summation}
\end{equation}
\end{small}

Since $\teta \leq \min(\frac{1}{3\ell + 3\ell^2/\mu_2}, \frac{1}{4\ell}, \frac{3}{2 \mu_2} )$,
thus in the RHS of (\ref{before_summation}), $C_1$ can be cancelled. 
$C_2$, $C_3$ and $C_4$ will be handled by telescoping sum.
$C_5$ can be bounded by Lemma \ref{lem:bound_B}. 

Taking expectation over $C_6$,
\begin{small}
\begin{equation}
\begin{split}
&\E\left[\frac{3\teta}{2}  \left\|\frac{1}{KI}\sum\limits_{k,i}[\nabla_{\v} f_k^s(\v_{r,t}^k, \alpha_{r,t}^k) - \nabla_{\v}F_k^s(\v_{r,t}^k, \alpha_{r,t}^k; z_{r,t}^k)]\right\|^2\right]\\
&=\E\left[\frac{3\teta}{2K^2  I^2}\sum\limits_{k,i} \|\nabla_{\v} f_k^s(\v_{r,t}^k, \alpha_{r,t}^k) - \nabla_{\v}F_k^s(\v_{r,t}^k, \alpha_{r,t}^k; z_{r,t}^k)\|^2\right]
\leq \frac{3\teta \sigma^2}{2KI}. 
\end{split}  
\label{local_variance_x}
\end{equation}
\end{small}
The equality is due to \\$\E_{r,t} \left\langle \nabla_{\v} f_k^s(\v_{r,t}^k, \alpha_{r,t}^k) 
- \nabla_{\v} F_k^s(\v_{r,t}^{i}, \alpha_{r,t}^{i}; z_{r,t}^k), 
\nabla_{\v} f_j^s(\v_{r,t}^j, \alpha_{r,t}^j) 
- \nabla_{\v} F_j^s(\v_{r,t}^{j}, \alpha_{r,t}^{j}; z_{r,t}^j) 
\right\rangle = 0$ for any $i \neq j$ as each machine draws data independently, where $\E_{r,t}$ denotes an expectation in round $r$ conditioned on events until $k$.
The last inequality holds because $\|\nabla_{\v}f_k(\v_{t-1}^k,\alpha_{t-1}^k) - \nabla_{\v}F_k (\v_{t-1}^k, \alpha_{t-1}^k;z_{t-1}^k)\|^2 \leq \sigma^2$
for any $i$.
Similarly, we take expectation over $C_7$ and have
\begin{small}
\begin{equation}
\begin{split} 
&\E\left[\frac{3\teta}{2}  \left(\frac{1}{K I}\sum\limits_{k,t} [\nabla_{\alpha} f_k(\v_{r,t}^k, \alpha_{r,t}^k) - \nabla_{\alpha}F_k(\v_{r,t}^k, \alpha_{r,t}^k; \z_{r,t}^k)]\right)^2\right]
\leq \frac{3\teta \sigma^2}{2KI}. 
\end{split}  
\label{local_variance_y}
\end{equation}
\end{small} 

Plugging (\ref{local_variance_x}) and  (\ref{local_variance_y})  into (\ref{before_summation}), and taking expectation, it yields
\begin{equation*}
\begin{split}
&\frac{1}{R} \sum\limits_r \E[f^s(\v_r, \alpha) - f^s(\v, \alpha_r)]\\
&\leq \E\bigg\{\frac{1}{R}\left(   \frac{1}{2\teta}-\frac{\ell_2}{3}\right) \|\v_0-\v\|^2 +  \frac{1}{R}\left(\frac{1}{2\teta} - \frac{\mu_2}{3} \right)(\alpha_0 - \alpha)^2
+ \frac{1}{2\teta R} \|\v_0 - \v\|^2 + \frac{1}{2\teta R} (\alpha_0-\alpha)^2  \\  
&~~~~~+ \frac{1}{R}\sum\limits_{r=1}^{R}\left(\frac{3\ell^2}{2\mu_2} + \frac{3\ell}{2}\right) \mathcal{E}_r+\frac{3\teta \sigma^2}{KI}\bigg\}\\ 
&\leq \frac{1}{\teta R} \|\v_0 - \v\|^2 + \frac{1}{\teta R} (\alpha_0 -\alpha)^2 + 
\frac{3\ell^2}{\mu_2} \frac{1}{R}\sum\limits_{r=1}^R \mathcal{E}_r
+ \frac{3\teta\sigma^2}{KI}, 
\end{split} 
\end{equation*} 
where we use  $\v_0 = \bar{\v}_0$, and $\alpha_0 = \bar{\alpha}_0$ in the last inequality. 

Using Lemma \ref{lem:bound_B},
\begin{equation*}
\begin{split}
&\frac{1}{R} \sum\limits_r \E[f^s(\v_r, \alpha) - f^s(\v, \alpha_r)]\\
&\leq \frac{1}{\teta R} \|\v_0 - \v\|^2 + \frac{1}{\teta R} (\alpha_0 -\alpha)^2 +  \frac{3\ell^2}{\mu_2} \frac{1}{R} \sum\limits_{r=1}^{R} \mathcal{E}_r  + \frac{3\teta\sigma^2}{KI} \\
& \leq \frac{1}{\teta R } \|\v_0 - \v\|^2 + \frac{1}{\teta R} (\alpha_0 -\alpha)^2 \\ 
&~~~ + \frac{3\ell^2}{\mu_2} \frac{1}{R}  \sum\limits_{r=1}^{R} \left[  \left(\frac{\teta \sigma^2}{2\ell I \eta_g^2} + \teta \ell \Xi_{r-1} +  \frac{48\teta^2}{\eta_g^2}  \E[\|\nabla_\v f^s(\v_{r}, \alpha_{r})\|^2+ (\nabla_\alpha f^s(\v_{r}, \alpha_{r}))^2  ] \right) \right]  + \frac{3\teta\sigma^2}{KI}\\  
&\leq  \frac{1}{\teta R} \|\v_0 - \v\|^2 + \frac{1}{\teta R} (\alpha_0 -\alpha)^2 + \frac{3 \teta \ell^3}{\mu_2 R \eta_g^2} \sum_r \Xi_{r-1} 
+ \frac{5 \ell}{\mu_2 I \eta_g^2} \teta  \sigma^2 + \frac{3000 \teta^2 \ell^4}{\mu_2^2  \eta_g^2} \frac{1}{R} \sum\limits_{r=1}^R Gap_{r}, 
\end{split}  
\end{equation*}  
where the last inequality holds because 
\begin{equation}
\begin{split}
\|\nabla_\v f^s(\v_r, \alpha_r)\|^2 + \|\nabla_\alpha f^s(\v_r, \alpha_r)\|^2 \leq 9\ell^2 (\|\v_r - \v^*_{f_s}\|^2+  (\alpha_r-\alpha^*_{f_s})^2) \leq  \frac{18\ell^2}{\mu_2} Gap_s(\v_r, \alpha_r),
\end{split}
\end{equation}
where $(\v^*_{f^s}, \alpha^*_{f^s})$ denotes a saddle point of $f^s$ and the second inequality uses the strong convexity and strong concavity of $f^s$. In detail, 
\begin{equation}
\begin{split}
Gap_s(\v_r, \alpha_r) &= \max\limits_{\alpha} f^s(\v_r, \alpha) - f^s(\v^*_{f^s}, \alpha^*_{f^s}) + f^s(\v^*_{f^s}, \alpha^*_{f^s}) - \min\limits_{\v} f^s(\v, \alpha_r) \\
& \geq \frac{\ell}{2} \|\v_r-\v^*_{f^s}\|^2 + \frac{\mu_2}{2}(\alpha_r -  \alpha^*_{f^s})^2. 
\end{split}
\end{equation}

Using Lemma \ref{lem:bound_A}, we have
\begin{equation}
\begin{split}
\Xi_r &\leq 4 \mathcal{E}_r + 16\teta^2 [\|\nabla_\v f^s(\v_r, \alpha_r)\|^2 + (\nabla_\alpha f^s(\v_r, \alpha_r))^2] + \frac{5\teta^2 \sigma^2}{KI} \\
&\leq 4 \left(\frac{\teta \sigma^2}{2\ell K \eta_g^2} + \teta \ell \Xi_{r-1} + \frac{48 \teta^2}{\eta_g^2}[\|\nabla_\v f^s(\v_r, \alpha_r)\|^2 + (\nabla_\alpha f^s(\v_r, \alpha_r))^2] \right) \\ 
&~~~ + 16\teta^2 [\|\nabla_\v f^s(\v_r, \alpha_r)\|^2 + (\nabla_\alpha f^s(\v_r, \alpha_r))^2] + \frac{5\teta \sigma^2}{ KI} \\
&\leq 4\teta \ell \Xi_{r-1} + 160 \teta^2 [ \|\nabla_\v f^s(\v_r, \alpha_r)\|^2 + (\nabla_\alpha f^s(\v_r, \alpha_r))^2 ]
+ \frac{5\teta\sigma^2}{KI} (1+\frac{K}{\eta_g^2}) \\
&\leq  \Xi_{r-1} + 160 \teta^2 [ \|\nabla_\v f^s(\v_r, \alpha_r)\|^2 + (\nabla_\alpha f^s(\v_r, \alpha_r))^2 ] 
+ \frac{5\teta\sigma^2}{KI} (1+\frac{K}{\eta_g^2}). 
\end{split} 
\end{equation} 
Thus,
\begin{equation}
\begin{split}
\frac{2\teta \ell^3}{\mu_2 R \eta_g^2} \sum\limits_{r=1}^{R} \Xi_{r} \leq&  \frac{2\teta \ell^3}{\mu_2 R \eta_g^2} \sum_r \Xi_{r-1} + \frac{320 \teta^3 \ell^3}{\mu_2 R \eta_g^2} \sum\limits_{r=1}^{R}  [\|\nabla_\v f^s(\v_r, \alpha_r)\|^2 + (\nabla_\alpha f^s(\v_r, \alpha_r))^2]\\
&
+ \frac{5\teta\sigma^2}{ KI}  (1+\frac{K}{\eta_g^2}) \\
\leq& \frac{2\teta \ell^3}{\mu_2 R \eta_g^2} \sum_r \Xi_{r-1} + \frac{1}{2 R} \sum\limits_r Gap_{r} + \frac{5\teta\sigma^2}{ KI}  (1+\frac{K}{\eta_g^2})  
\end{split}
\end{equation}

Taking $A_0 = 0$,
\begin{equation*}
\begin{split}
&\frac{1}{R} \sum\limits_r \E[f^s(\v_r, \alpha) - f^s(\v, \alpha_r)]\\
&\leq  \frac{1}{\teta R} \|\v_0 - \v\|^2 + \frac{1}{\teta R} (\alpha_0 -\alpha)^2 
+ \frac{1}{2 R} \sum\limits_r Gap_{r} + \frac{5\teta\sigma^2}{ KI}  (1+\frac{K}{\eta_g^2}). 
\end{split} 
\end{equation*} 

It follows that
\begin{equation*}
\begin{split}
&\frac{1}{R} \sum\limits_r \E[f^s(\v_r, \alpha) - f^s(\v, \alpha_r)] - \frac{1}{2 R} \sum\limits_r Gap_{r} \\
&\leq  \frac{1}{\teta R} \|\v_0 - \v\|^2 + \frac{1}{\teta R} (\alpha_0 -\alpha)^2 
+ \frac{5\teta\sigma^2}{ KI}  (1+\frac{K}{\eta_g^2}).
\end{split}  
\end{equation*}

Sample a $\tilde{r}$ from $1, ..., R$, we have 
\begin{equation}
\begin{split}
\E[Gap^s_{\tilde{r}}] \leq  \frac{2}{\teta R} \|\v_0 - \v\|^2 + \frac{2}{\teta R} (\alpha_0 -\alpha)^2  + \frac{10\teta\sigma^2}{KI}  \left(1+\frac{K}{\eta_g^2}\right).
\end{split} 
\end{equation}

\end{proof}

\section{Proof of Theorem \ref{thm:coda_plus}}
\begin{proof} 
Since $f(\v, \alpha)$ is $\ell$-weakly convex in $\v$ for any $\alpha$, $\phi(\v) = \max\limits_{\alpha'} f(\v, \alpha')$ is also $\ell$-weakly convex. 
Taking $\gamma = 2\ell$, we have
\begin{align}  
\begin{split} 
\phi(\v_{s-1}) &\geq \phi(\v_s) + \langle \partial \phi(\v_s), \v_{s-1} - \v_s\rangle - \frac{\ell}{2} \|\v_{s-1} - \v_s\|^2 \\ 
& = \phi(\v_s) + \langle \partial \phi(\v_s) + 2 \ell (\v_s - \v_{s-1}), \v_{s-1} - \v_s\rangle + \frac{3\ell}{2}  \|\v_{s-1} - \v_s\|^2 \\ 
& \overset{(a)}{=} \phi(\v_s) + \langle \partial \phi_s(\v_s), \v_{s-1} - \v_s \rangle + \frac{3\ell}{2} \|\v_{s-1} - \v_s\|^2 \\ 
& \overset{(b)}{=}  \phi(\v_s) - \frac{1}{2\ell} \langle \partial \phi_s(\v_s), \partial \phi_s(\v_s) - \partial \phi(\v_s) \rangle + \frac{3}{8\ell} \|\partial \phi_s(\v_s) - \partial \phi(\v_s)\|^2 \\ 
& = \phi(\v_s) - \frac{1}{8\ell} \|\partial \phi_s(\v_s)\|^2
- \frac{1}{4\ell} \langle \partial \phi_s(\v_s), \partial \phi(\v_s)\rangle + \frac{3}{8\ell} \|\partial \phi(\v_s)\|^2,
\end{split} 
\label{local:P_weakly_2} 
\end{align} 
where $(a)$ and $(b)$ hold by the definition of $\phi_s(\v)$.

Rearranging the terms in (\ref{local:P_weakly_2}) yields
\begin{align}
\begin{split}
\phi(\v_s) - \phi(\v_{s-1}) &\leq \frac{1}{8\ell} \|\partial \phi_s(\v_s)\|^2 + \frac{1}{4\ell}\langle \partial \phi_s(\v_s), \partial \phi(\v_s)\rangle - \frac{3}{8\ell} \|\partial \phi(\v_s)\|^2 \\
&\overset{(a)}{\leq} \frac{1}{8\ell} \|\partial \phi_s(\v_s)\|^2
+ \frac{1}{8\ell} (\|\partial \phi_s(\v_s)\|^2 + \|\partial \phi(\v_s)\|^2) - \frac{3}{8\ell} \|\phi(\v_s)\|^2 \\
& = \frac{1}{4\ell}\|\partial \phi_s(\v_s)\|^2 - \frac{1}{4\ell}\|\partial \phi(\v_s)\|^2\\
& \overset{(b)}{\leq} \frac{1}{4\ell} \|\partial \phi_s(\v_s)\|^2 - \frac{\mu}{2\ell}(\phi(\v_s) - \phi(\v^*_{\phi_s})) 
\end{split} 
\end{align}
where $(a)$ holds by using $\langle \mathbf{a}, \mathbf{b}\rangle \leq \frac{1}{2}(\|\mathbf{a}\|^2 +  \|\mathbf{b}\|^2)$, and $(b)$ holds by the $\mu$-PL property of $\phi(\v)$.

Thus, we have 
\begin{align}
\left(4\ell+2\mu\right) (\phi(\v_s) - \phi(\v_*)) - 4\ell (\phi(\v_{s-1}) - \phi(\v^*_{\phi_s})) \leq \|\partial \phi_s(\v_s)\|^2.  
\label{local:nemi_thm_partial_P_s_1} 
\end{align} 

Since $\gamma = 2\ell$, $f_s(\v, \alpha)$ is $\ell$-strongly convex in $\v$ and $\mu_2$ strong concave in $\alpha$.
Apply Lemma \ref{lem:Yan1} to $f_s$, we know that 
\begin{align} 
\frac{\ell}{4} \|\hat{\v}_s(\alpha_s) - \v_0^s\|^2 + \frac{\mu_2}{4} (\hat{\alpha}_s(\v_s) - \alpha_0^s)^2 \leq \text{Gap}_s(\v_0^s, \alpha_0^s) + \text{Gap}_s(\v_s, \alpha_s). 
\end{align} 

By the setting of $\teta_s$, $I_s = I_0*2^s$, and 
$R_s = \frac{1000}{\teta \min(\ell, \mu_2)}$, we note that $\frac{4}{\teta R_s} \leq \frac{\min\{\ell,\mu_2\}}{212}$. 
Applying Lemma (\ref{lem:codasca:one_stage}), we have
\begin{align}  
\begin{split} 
&\E[\text{Gap}_s(\v_s, \alpha_s)] 
\leq \frac{10 \teta\sigma^2}{K I_0 2^s}    
+ \frac{1}{53}  \E\left[\frac{\ell}{4}\|\hat{\v}_s(\alpha_s) - \v_0^s\|^2 + \frac{\mu_2}{4}(\hat{\alpha}_s(\v_s) - \alpha_0^s)^2 \right]
\\  
& \leq \frac{10 \teta\sigma^2}{K I_0 2^s}  + \frac{1}{53} \E\left[\text{Gap}_s(\v_0^s, \alpha_0^s) + \text{Gap}_s(\v_s, \alpha_s)\right]. 
\end{split} 
\end{align} 

Since $\phi(\v)$ is $L$-smooth and $\gamma = 2\ell$, then $\phi_k(\v)$ is $\hat{L} = (L+2\ell)$-smooth. 
According to Theorem 2.1.5 of \citep{DBLP:books/sp/Nesterov04}, we have 
\begin{align}
\begin{split} 
& \E[\|\partial \phi_s(\v_s)\|^2] \leq 2\hat{L}\E(\phi_s(\v_s) - \min\limits_{x\in \mathbb{R}^{d}} \phi_s(\v)) \leq 2\hat{L}\E[\text{Gap}_s(\v_s, \alpha_s)] \\
& = 2\hat{L}\E[4\text{Gap}_s(\v_s, \alpha_s) - 3\text{Gap}_s(\v_s, \alpha_s)] \\ 
&\leq 2\hat{L} \E \left[4\left(\frac{10 \teta\sigma^2}{K I_0 2^s} +  \frac{1}{53}\left(\text{Gap}_s(\v_0^s, \alpha_0^s) + \text{Gap}_s(\v_s, \alpha_s)\right)\right) - 3\text{Gap}_s(\v_s,\alpha_s)\right] \\ 
& = 2\hat{L} \E \left[40\frac{\teta\sigma^2}{K I_0 2^s} +   \frac{4}{53}\text{Gap}_s(\v_0^s, \alpha_0^s) - \frac{155}{53}\text{Gap}_s(\v_s, \alpha_s)\right]. 
\end{split}  
\label{local:nemi_thm_P_smooth}
\end{align}

Applying Lemma \ref{lem:Yan5} to  (\ref{local:nemi_thm_P_smooth}), we have
\begin{align}
\begin{split} 
& \E[\|\partial \phi_s(\v_s)\|^2] \leq 2\hat{L} \E \bigg[ \frac{40 \teta\sigma^2}{K I_0 2^s}  +  \frac{4}{53}\text{Gap}_s(\v_0^s, \alpha_0^s) \\  
&~~~~~~~~~~~~~~~~~~~~~~~~~~~~~  
- \frac{155}{53} \left(\frac{3}{50} \text{Gap}_{s+1}(\v_0^{s+1}, \alpha_0^{s+1}) + \frac{4}{5} (\phi(\v_0^{s+1}) - \phi(\v_0^s))\right) \bigg] \\
& = 2\hat{L}\E \bigg[\frac{40\teta\sigma^2}{K I_0 2^s} \!+ \! \frac{4}{53}\text{Gap}_s(\v_0^s, \alpha_0^s) \!-\! \frac{93}{530}\text{Gap}_{s+1}(\v_0^{s+1}, \alpha_0^{s+1}) \!-\! 
\frac{124}{53} (\phi(\v_0^{s+1}) - \phi(\v_0^s)) \bigg]. 
\end{split} 
\end{align}

Combining this with  (\ref{local:nemi_thm_partial_P_s_1}), rearranging the terms, and defining a constant $c = 4\ell + \frac{248}{53}\hat{L} \in O(L+\ell)$, we get
\begin{align} 
\begin{split}
&\left(c + 2\mu\right)\E [\phi(\v_0^{s+1}) - \phi(\v_*)] + \frac{93}{265}\hat{L} \E[\text{Gap}_{s+1}(\v_0^{s+1}, \alpha_0^{s+1})] \\ 
&\leq \left(4\ell + \frac{248}{53} \hat{L}\right) \E[\phi(\v_0^s) - \phi(\v^*_{\phi})] 
+ \frac{8\hat{L}}{53} \E[\text{Gap}_s(\v_0^s, \alpha_0^s)] 
+ \frac{80 \hat{L} \teta \sigma^2}{KI_0 2^s}  \\ 
& \leq c \E\left[\phi(\v_0^s) - \phi(\v_*) + \frac{8\hat{L}}{53c} \text{Gap}_s(\v_0^s, \alpha_0^s)\right] + \frac{80 \hat{L} \teta \sigma^2}{KI_0 2^s}. 
\end{split}  
\end{align} 

Using the fact that $\hat{L} \geq \mu$,
\begin{align}
\begin{split}
(c+2\mu) \frac{8\hat{L}}{53c} = \left(4\ell + \frac{248}{53}\hat{L} +  2\mu\right)\frac{8\hat{L}}{53(4\ell + \frac{248}{53}\hat{L})} \leq \frac{8\hat{L}}{53} + \frac{16\mu_1 \hat{L}}{248\hat{L}} \leq \frac{93}{265} \hat{L}. 
\end{split}
\end{align}

Then, we have
\begin{align}
\begin{split} 
&(c+2\mu_1)\E \left[\phi(\v_0^{s+1}) - \phi(\v_*) + \frac{8\hat{L}}{53c}\text{Gap}_{s+1}(\v_0^{s+1}, \alpha_0^{s+1})\right] \\ 
&\leq c \E \left[\phi(\v_0^s) - \phi(\v_*) 
+  \frac{8\hat{L}}{53c}\text{Gap}_{s}(\v_0^{s},  \alpha_0^s)\right] 
+ \frac{80 \hat{L} \teta \sigma^2}{KI_0 2^s}.  
\end{split}
\end{align}

Defining $\Delta_s = \phi(\v_0^s) - \phi(\v_*) +  \frac{8\hat{L}}{53c}\text{Gap}_s(\v_0^s, \alpha_0^s)$, then
\begin{align}
\begin{split}
&\E[\Delta_{s+1}] \leq \frac{c}{c+2\mu} \E[\Delta_s] +  \frac{80 \hat{L} }{c+2\mu} \frac{\teta \sigma^2}{KI_0 2^s}
\end{split}
\end{align}

Using this inequality recursively, it yields
\begin{align} 
\begin{split}
& E[\Delta_{S+1}] \leq \left(\frac{c}{c+2\mu}\right)^S E[\Delta_1]
+ \frac{80 \hat{L} }{c+2\mu} \frac{\teta\sigma^2}{KI_0} \sum\limits_{s=1}^{S} \left(\exp\left(- \frac{2\mu}{c+2\mu}(s-1) \right)  \left(\frac{c}{c+2\mu}\right)^{S+1-s} \right) \\
&\leq 2\epsilon_0 \exp\left(\frac{-2\mu S}{c+2\mu}\right)
+ \frac{80 \teta \hat{L} \sigma^2} {(c+2\mu)KI_0}  
S\exp\left(-\frac{2\mu S}{c+2\mu}\right) ,
\end{split}
\end{align} 
where the second inequality uses the fact $1-x \leq \exp(-x)$,
and 
\begin{align}
\begin{split}
\Delta_1 &= \phi(\v_0^1) - \phi(\v^*) + \frac{8\hat{L}}{53c} {Gap}_1(\v_0^1, \alpha_0^1) \\
& = \phi(\v_0) - \phi(\v^*) + \left( f(\v_0, \hat{\alpha}_1(\v_0)) +  \frac{\gamma}{2}\|\v_0 - \v_0\|^2 -  f(\hat{\v}_1(\alpha_0), \alpha_0) - \frac{\gamma}{2}\|\hat{\v}_1(\alpha_0) - \v_0\|^2 \right) \\
& \leq \epsilon_0 + f(\v_0, \hat{\alpha}_1(\v_0)) - f(\hat{\v}(\alpha_0), \alpha_0) \leq 2\epsilon_0.
\end{split} 
\end{align}

To make this less than $\epsilon$, it suffices to make
\begin{align}
\begin{split}
& 2\epsilon_0 \exp\left(\frac{-2\mu S}{c+2\mu}\right) \leq \frac{\epsilon}{2},\\
& \frac{80 \teta \hat{L} \sigma^2} {(c+2\mu)KI_0} 
S\exp\left(-\frac{2\mu S}{c+2\mu}\right) \leq \frac{\epsilon}{2}.
\end{split} 
\end{align}

Let $S$ be the smallest value such that $\exp\left(\frac{-2\mu S}{c+2\mu}\right) \leq \min \{ \frac{\epsilon}{4\epsilon_0}, 
\frac{(c+2\mu)\epsilon}{160 \hat{L} S} \frac{KI_0}{\teta \sigma^2}\}$.  
We can set $S$ to be the smallest value such that $S >  \max\bigg\{\frac{c+2\mu}{2\mu}\log \frac{4\epsilon_0}{\epsilon}, 
\frac{c+2\mu}{2\mu}\log \frac{160 \hat{L} S } {(c+2\mu)\epsilon} \frac{\teta\sigma^2}{K I_0} \bigg\}$. 

Then, the total communication complexity is 
\begin{align*} 
\sum\limits_{s=1}^{S} R_s &\leq O\left( \frac{1000}{\teta \mu_2}
S \right)
\leq \widetilde{O}\bigg(\frac{1}{\teta\mu_2}  \frac{c}{\mu}  \bigg)
\leq \widetilde{O}\left(\frac{1}{\mu} \right).  
\end{align*} 

Total iteration complexity is 
\begin{equation}
\begin{split}
&\sum\limits_{s=1}^S T_s = \sum\limits_{s=1}^S  R_s I_s \\
& = \sum\limits_{s=1}^S  R_s I_0 \exp(\frac{2\mu}{c+2\mu}(s-1))  = O\left( I_0 \sum_s \exp(\frac{2\mu}{c+2\mu}(s-1)) \right)\\
& = \widetilde{O}\left( I_0 \frac{\exp(\frac{2\mu}{c+2\mu} S )}{\exp(\frac{2\mu_1}{c+2\mu})} \right)
= \widetilde{O} \left( \frac{c}{\mu_2^2 \mu} \left(  \frac{ \epsilon_0}{\epsilon}, \frac{S \teta\sigma^2}{I_0 K\epsilon}\right) \right)\\
& = 
\widetilde{O}\left( \max( \frac{1}{\mu \epsilon}, \frac{c^2}{\mu^2} \frac{\teta \sigma^2}{K})  \right) = \widetilde{O} \left(\max(\frac{1}{\mu \epsilon}, \frac{1}{K \mu^2 \epsilon}) \right) ,
\end{split}
\end{equation} 
which is also the sample complexity on each single machine. 

\end{proof}

\section{More Results}
In this section, we report more experiment results for imratio=30\% with DenseNet121 on ImageNet-IH, and CIFAR100-IH in Figure \ref{fig:cifar_0.1},\ref{fig:imagenet_0.3} and \ref{fig:cifar_0.3}. 

\begin{figure}[h]
    \centering
    {\includegraphics[scale=0.11]{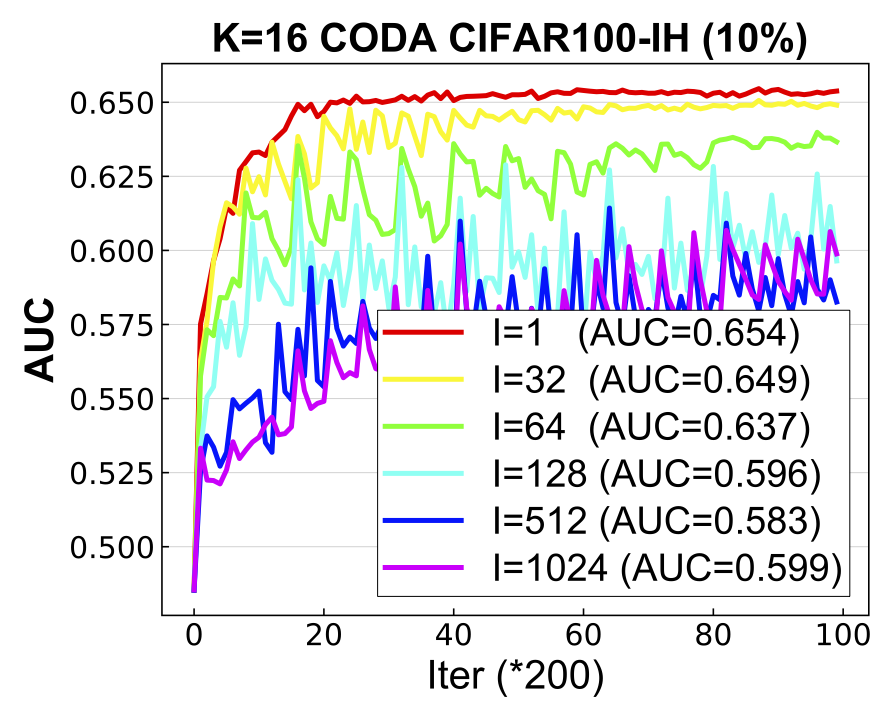}
    \includegraphics[scale=0.11]{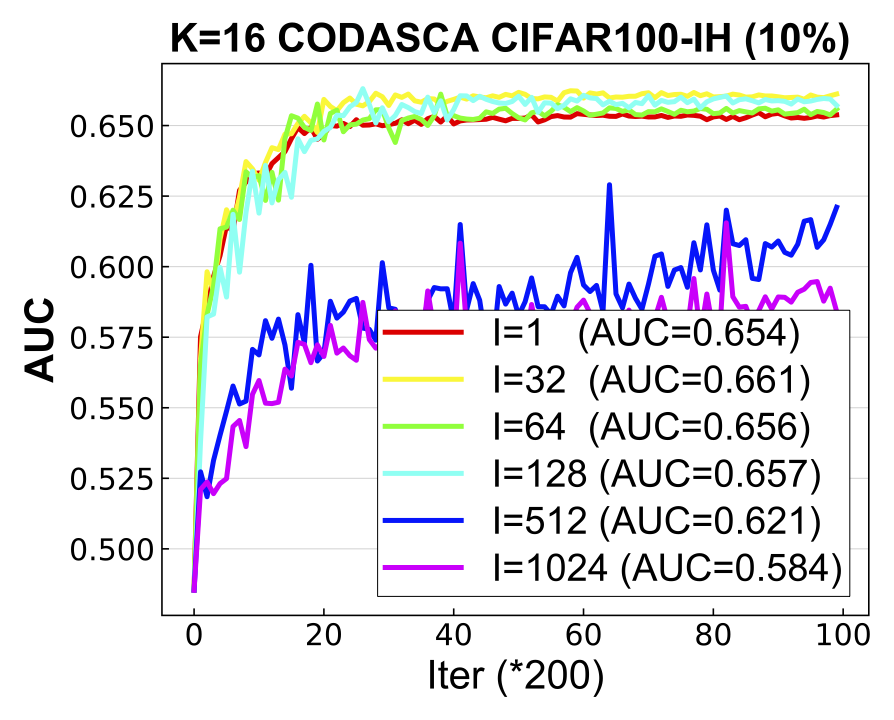}
    }
    {
    \includegraphics[scale=0.11]{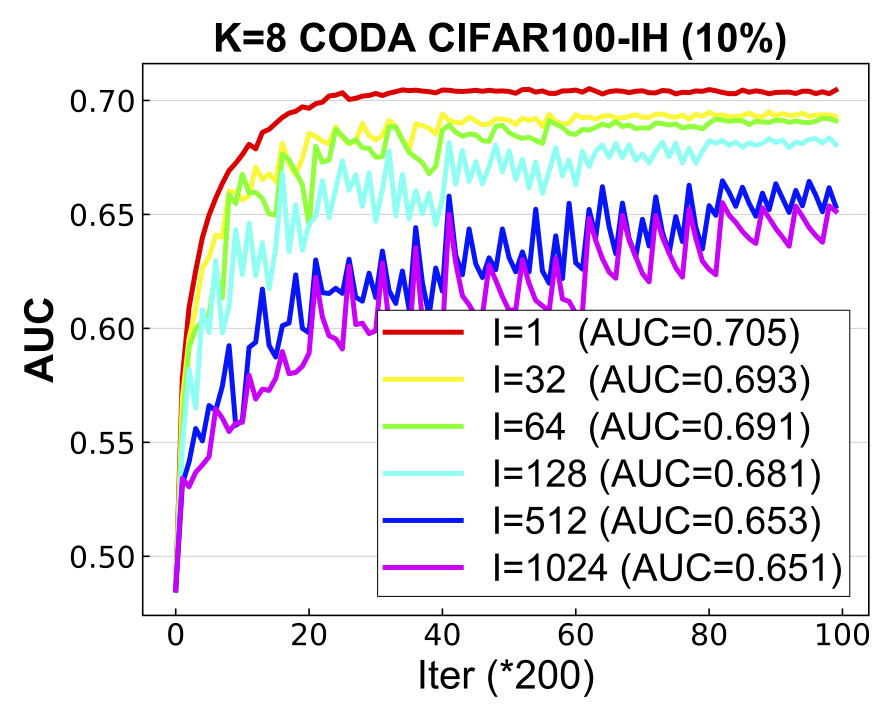}
    \includegraphics[scale=0.11]{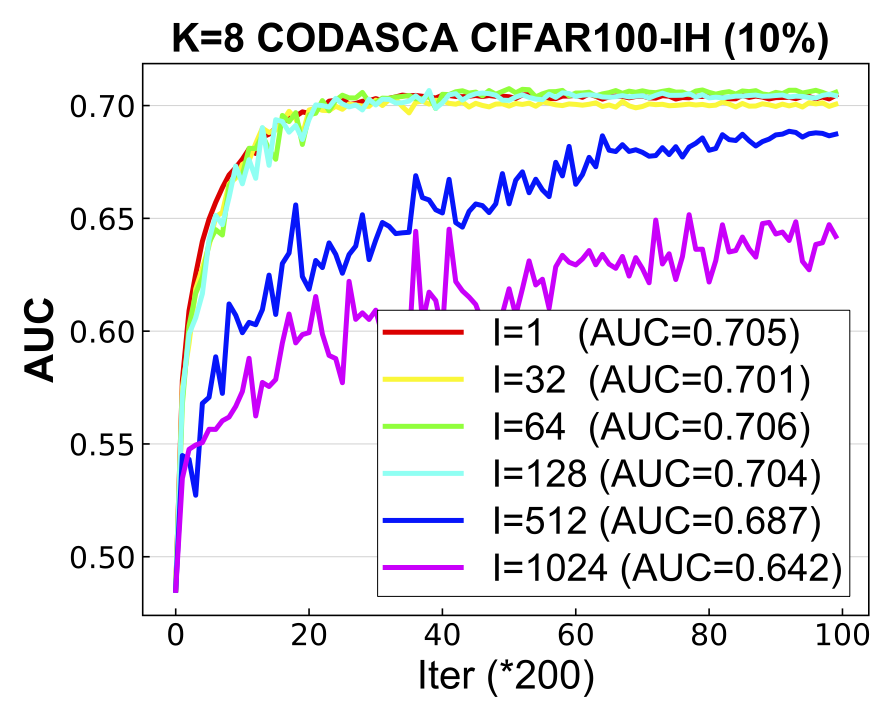}
    }   
    \caption{Imbalanced Heterogeneous CIFAR100 with imratio = 10\% and K=16,8 on Densenet121.}
    \label{fig:cifar_0.1}
\end{figure}

\begin{figure}[h]
    \centering
    {\includegraphics[scale=0.11]{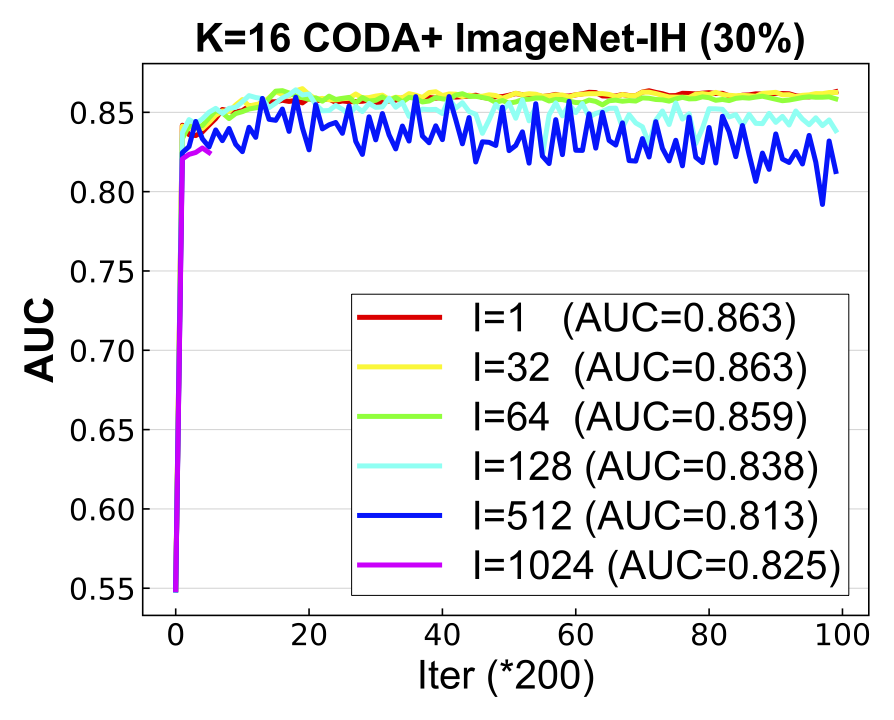}
    \includegraphics[scale=0.11]{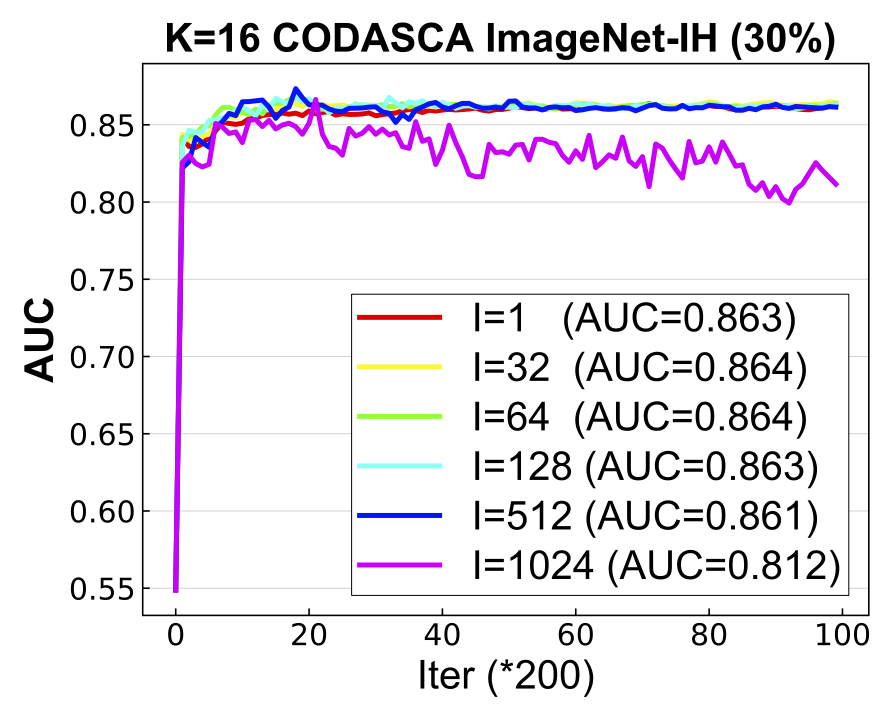}
    }
    {
    \includegraphics[scale=0.11]{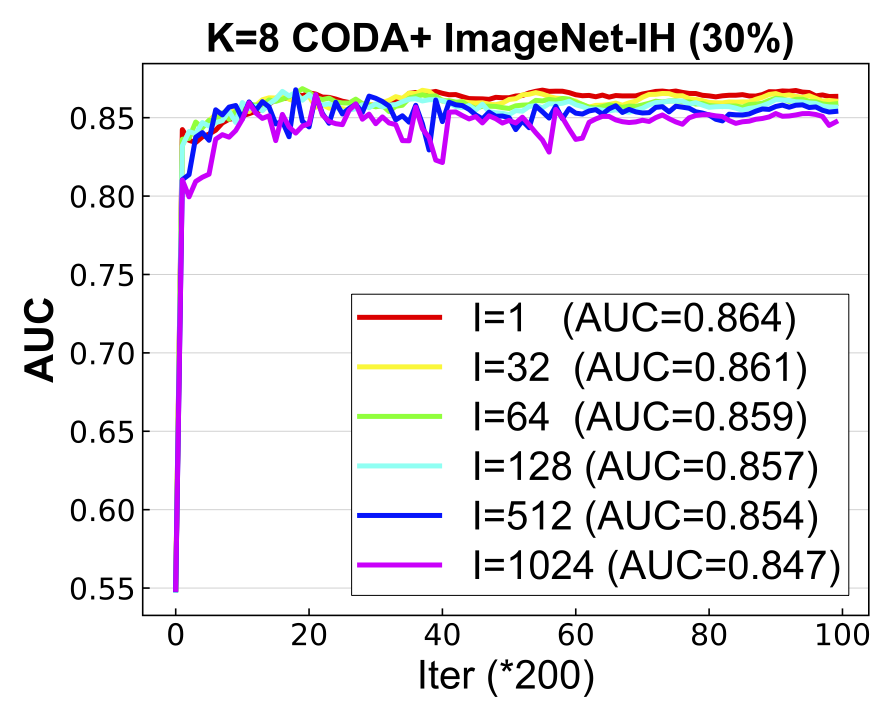}
    \includegraphics[scale=0.11]{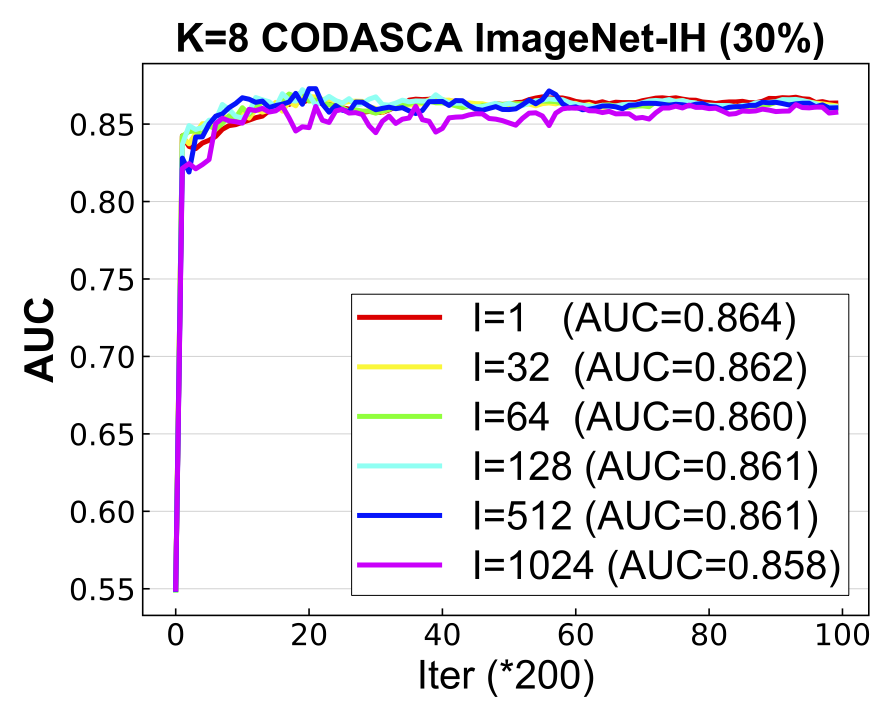}
    }  
    \caption{Imbalanced Heterogeneous ImageNet with imratio = 30\% and K=16,8 on Densenet121.}
    \label{fig:imagenet_0.3}
\end{figure}

\begin{figure}[h]
    \centering
    {\includegraphics[scale=0.11]{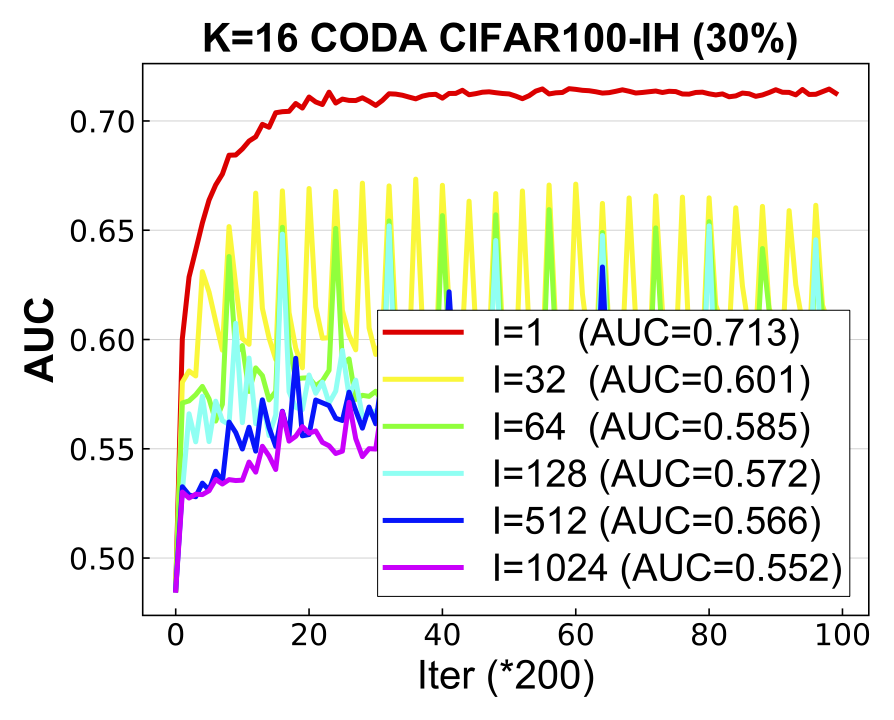}
    \includegraphics[scale=0.11]{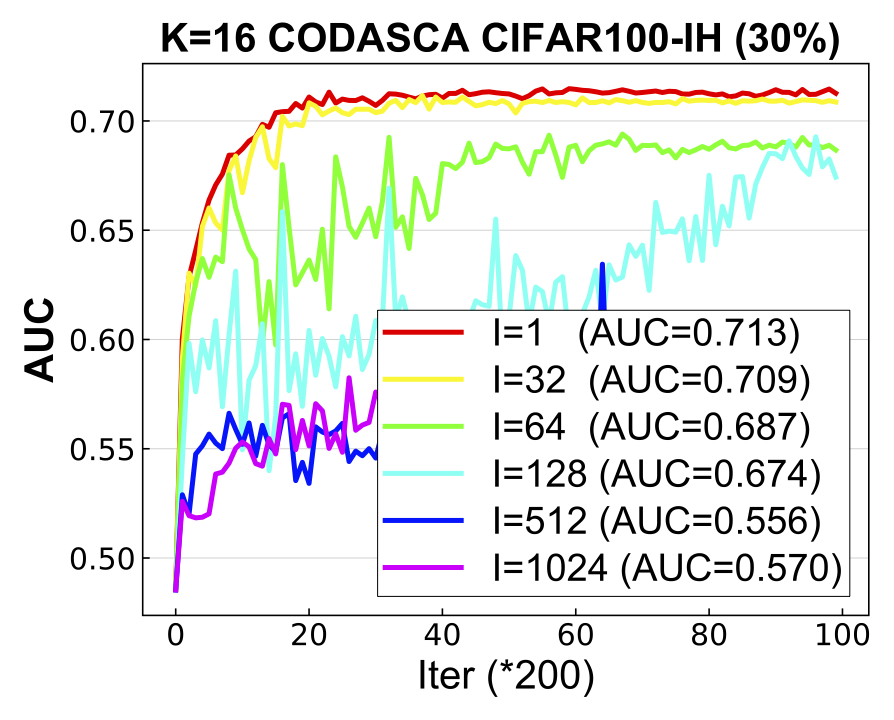}
    }
    {
    \includegraphics[scale=0.11]{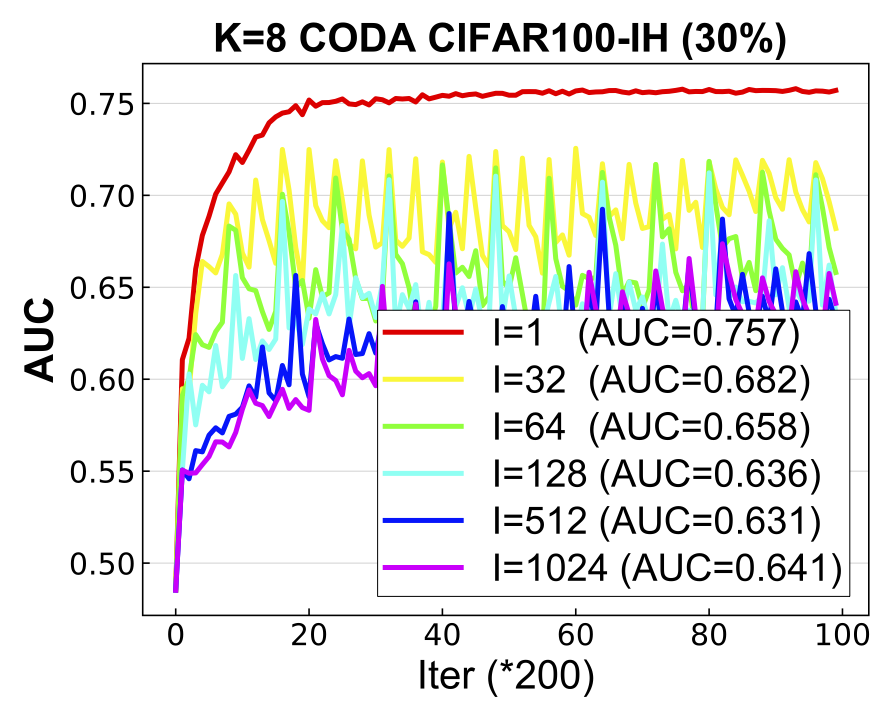}
    \includegraphics[scale=0.11]{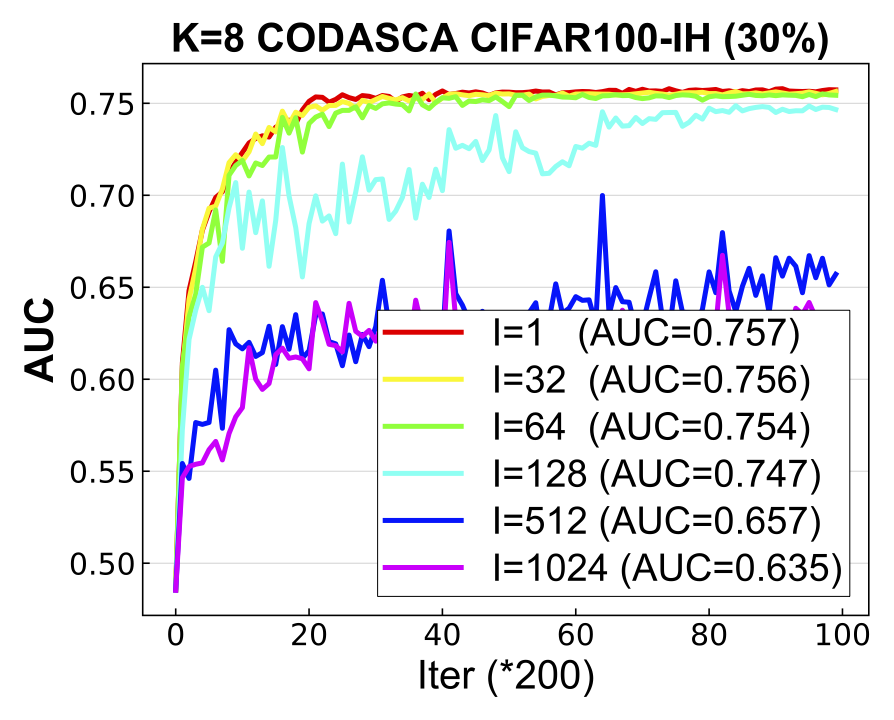}
    }  
    \caption{Imbalanced Heterogeneous CIFAR100 with imratio = 30\% and K=16,8 on Densenet121.}
    \label{fig:cifar_0.3}
\end{figure}

\newpage

\section{Descriptions of Datasets}
\vspace{-0.2in}
\begin{table}[h!]
\centering
\caption{Statistics of Medical Chest X-ray Datasets. The numbers for each disease denote the imbalance ratio (imratio).}
\scalebox{0.7}{
\begin{tabular}{cccccccc}
\hline
\textbf{Dataset} & \textbf{Source} & \textbf{Samples} & \textbf{Cardiomegaly} & \textbf{Edema} & \textbf{Consolidation} & \textbf{Atelectasis} & \textbf{Effusion} \\ \hline
CheXpert      & Stanford Hospital (US)   & 224,316           & 0.211                 & 0.342          & 0.120                  & 0.310                & 0.414             \\ 
ChestXray8  & NIH Clinical Center (US)    & 112,120           & 0.025                 & 0.021          & 0.042                  & 0.103                & 0.119             \\
PadChest  & Hospital San Juan (Spain)       & 110,641           & 0.089                 & 0.012          & 0.015                  & 0.056                & 0.064             \\ 
MIMIC-CXR   & BIDMC (US)     & 377,110           & 0.196                 & 0.179          & 0.047                  & 0.246                & 0.237             \\
ChestXrayAD   & H108 and HMUH (Vietnam)   & 15,000            & 0.153                 & 0.000          & 0.024                  & 0.012                & 0.069             \\ \hline
\end{tabular}}
\label{table:xray_stats}
\end{table}
\end{document}